\documentclass{article}

\usepackage{microtype}
\usepackage{graphicx}
\usepackage{booktabs} 

\usepackage{hyperref}

\usepackage[accepted]{icml2024}

\usepackage{amsmath,amsfonts,bm}

\def\eqref#1{equation~\ref{#1}}

\def\1{\bm{1}}

\DeclareMathAlphabet{\mathsfit}{\encodingdefault}{\sfdefault}{m}{sl}
\SetMathAlphabet{\mathsfit}{bold}{\encodingdefault}{\sfdefault}{bx}{n}

\DeclareMathOperator*{\argmax}{arg\,max}
\DeclareMathOperator*{\argmin}{arg\,min}

\usepackage[algo2e,ruled,vlined,linesnumbered]{algorithm2e}
\SetKwInput{KwInput}{Input}

\usepackage[utf8]{inputenc} 
\usepackage[T1]{fontenc}    
\usepackage{wrapfig}
\usepackage{amssymb}
\usepackage{color}
\usepackage{amsthm}
\usepackage{stmaryrd,pict2e,picture}
\usepackage{nicefrac}
\usepackage{subcaption}
\usepackage{bbm}
\usepackage{soul}
\usepackage{multirow}
\usepackage{listings}
\usepackage{array}
\usepackage{tikz}
\usepackage{hyperref}
\usepackage{booktabs}
\usepackage{graphicx}
\usepackage{microtype}
\usepackage{url}
\usepackage{amsfonts}

\makeatletter
\newtheorem*{rep@theorem}{\rep@title}
\newcommand{\newreptheorem}[2]{%
\newenvironment{rep#1}[1]{%
 \def\rep@title{#2 \ref{##1}}%
 \begin{rep@theorem}}%
 {\end{rep@theorem}}}
\makeatother

\theoremstyle{plain}
\newtheorem{theorem}{Theorem}
\newreptheorem{theorem}{Theorem}
\newtheorem{definition}{Definition}

\newtheorem{lemma}{Lemma}

\newtheorem{assumption}{Assumption}

\renewcommand{\vec}[1]{\mathbf{#1}}
\newcommand{\x}{\mathbf{x}}
\newcommand{\W}{\mathbf{W}}
\newcommand{\bias}{\mathbf{b}}
\newcommand{\solnet}{u_\theta}
\newcommand{\resnet}{f_\theta}
\newcommand{\preout}{y}
\newcommand{\out}{z}
\newcommand{\xspatial}{\hat{\mathbf{x}}}
\newcommand{\partialcrown}{\mbox{$\partial$-CROWN}\xspace}

\newcommand{\schrodinger}{Schr\"{o}dinger}
\newcommand{\sech}{\text{sech}}
\newcommand{\ie}{\textit{i.e.}}
\newcommand{\lin}[4]{\mathbb{L}_{\mathbf{#1}, \mathbf{#2}}^{(#3)#4}}

\newcolumntype{P}[1]{>{\centering\arraybackslash}p{#1}}

\newcommand*\circled[1]{\protect\tikz[baseline=(char.base)]{
            \node[shape=circle,draw,inner sep=1pt] (char) {#1};}}

\icmltitlerunning{Efficient Error Certification for Physics-Informed Neural Networks}

\begin{document}

\twocolumn[
\icmltitle{Efficient Error Certification for Physics-Informed Neural Networks}

\begin{icmlauthorlist}
\icmlauthor{Francisco Eiras}{oxford}
\icmlauthor{Adel Bibi}{oxford}
\icmlauthor{Rudy Bunel}{dm}
\icmlauthor{Krishnamurthy Dj Dvijotham}{dm}
\icmlauthor{Philip H.S. Torr}{oxford}
\icmlauthor{M. Pawan Kumar}{dm}
\end{icmlauthorlist}

\icmlaffiliation{oxford}{University of Oxford}
\icmlaffiliation{dm}{Google DeepMind}

\icmlcorrespondingauthor{FE}{eiras@robots.ox.ac.uk}

\icmlkeywords{Machine Learning, ICML}

\vskip 0.3in
]

\printAffiliationsAndNotice{}

\begin{abstract}
Recent work provides promising evidence that Physics-Informed Neural Networks (PINN) can efficiently solve partial differential equations (PDE). However, previous works have failed to provide guarantees on the \textit{worst-case} residual error of a PINN across the spatio-temporal domain -- a measure akin to the tolerance of numerical solvers -- focusing instead on point-wise comparisons between their solution and the ones obtained by a solver on a set of inputs. In real-world applications, one cannot consider tests on a finite set of points to be sufficient grounds for deployment, as the performance could be substantially worse on a different set.
To alleviate this issue, we establish guaranteed error-based conditions for PINNs over their \textit{continuous} applicability domain. To verify the extent to which they hold, we introduce \partialcrown: a general, efficient and scalable post-training framework to bound PINN residual errors. We demonstrate its effectiveness in obtaining tight certificates by applying it to two classically studied PINNs -- Burgers' and \schrodinger's equations --, and two more challenging ones with real-world applications -- the Allan-Cahn and Diffusion-Sorption equations.
\end{abstract}

\section{Introduction}
Accurately predicting the evolution of complex systems through simulation is a difficult, yet necessary, process in the physical sciences. Many of these systems are represented by partial differential equations (PDE) the solutions of which, while well understood, pose a major computational challenge to solve at an appropriate spatio-temporal resolution~\citep{raissi2019physics,kochkovlearning}. Inspired by the success of machine learning in other domains, recent work has attempted to overcome the aforementioned challenge through \textit{physics-informed neural networks} (PINN)~\citep{raissi2019physics,sun2020surrogate,pang2019fpinns}. For example, the Diffusion-Sorption equation -- which has real-world applications in the modeling of groundwater contaminant transport -- takes 59.83s to solve per inference point using a classical PDE solver, while inference in its PINN version from~\citet{takamoto2022pdebench} takes only $2.7\times 10^{-3}$s, a speed-up of more than $10^4$ times.

The parameters of a PINN are estimated by minimizing the residual of the given PDE, together with its initial and boundary conditions, over a set of spatio-temporal training inputs. Its accuracy is then empirically estimated by measuring the solution estimate over a set of discrete input points, and (typically) comparing them to numerical PDE solvers. In other words, most current work on PINNs provides no certified error bounds applicable for \textit{every} input within the domain of the PDE. 

While testing on a finite set of points provides a good initial signal on the accuracy of the PINN, such an approach cannot be relied upon in practice, and error certification is needed to understand the quality of the PINN trained~\citep{hillebrecht2022certified}. For example, by estimating the maximum residual error of the Diffusion-Sorption PINN from~\citet{takamoto2022pdebench} using $10^4$ Monte Carlo samples across the domain we obtain an estimate of $1.1\times 10^{-3}$, whereas the estimate using $10^6$ samples is $21.09$ -- indicating the PINN has failed to learn a continuous function that correctly maps to the solution of the underlying PDE. This empirical difference shows the need for computing certified error bounds to avoid deploying poorly trained PINNs.

We introduce formal, error-based \textit{correctness} conditions for PINNs which require that the residual error is \textit{globally} upper bounded by a tolerance parameter, that is, that the continuous function learned approximates the underlying PDE solution across the domain.
To compute this bound and verify the correctness conditions, we build on recent progress in neural network verification. Specifically, we efficiently extend the CROWN framework~\citep{zhang2018efficient} by deriving linear upper and lower bounds for the various nonlinear terms that appear in PINNs, and devise a novel bound propagation strategy for the task at hand. 

Our contributions are threefold. \textbf{(i)} We formally define correctness conditions for general PINNs that approximate continuous solutions of PDEs. \textbf{(ii)} We introduce a general, efficient, and scalable post-training \textit{error certification framework} (\mbox{\partialcrown}) to theoretically verify PINNs over their entire spatio-temporal domains. \textbf{(iii)} We demonstrate our post-training framework on two widely studied PDEs in the context of PINNs, Burgers' and \schrodinger's equations \citep{raissi2019physics}, and two more challenging ones with real-world applications, the Allan-Cahn equation~\citep{monaco2023training} and the Diffusion-Sorption equation~\citep{takamoto2022pdebench}.

\section{Related work}

Since our certification framework for PINNs is based on the verification literature of image classifiers, 
in this section we explore: related work for PINNs, and previous work on neural network robustness verification.

\paragraph{Physics-informed Neural Networks.} 
\citet{raissi2019physics} introduced PINNs, which leverage automatic differentiation to obtain approximate solutions to the underlying PDE. Since then, a variety of different PINNs have emerged in a range of applications, from fluid dynamics \citep{raissi2019deep,raissi2020hidden,sun2020surrogate,jin2021nsfnets}, to meta material design \citep{liu2019multi,fang2019deep,chen2020physics} for different classes of PDEs \citep{pang2019fpinns, fang2019physics, zhang2020learning}. A few works analyze the convergence of the training process of PINNs under specific conditions \citep{shin2020convergence,wang2022and}. \citet{mishra2022estimates} approximated the generalization error of various PINNs under specific stability and training process assumptions, and others introduced approximation bounds under regularity assumptions \citep{ryck2022generic, hillebrecht2022certified}. Our verification framework is applicable to any PINN where the solution is modeled by a fully connected network.

\paragraph{Robustness Verification of Neural Networks.} The presence of adversarial examples, \ie, small local input perturbations that lead to large output changes, was established
by \citet{szegedy2013intriguing} in image classifiers.
As robust classifiers emerged~\citep{madry2017towards}, so did attempts to certify them formally. Those methods can be divided into \textit{exact}, \ie, complete \citep{katz2017reluplex,ehlers2017formal,huang2017safety,lomuscio2017approach,bunel2018unified,de2021improved,ferrari2022complete}, or \textit{conservative}, \ie, sound but incomplete \citep{gowal2018effectiveness,mirman2018differentiable,wang2018mixtrain,wong2018provable,ayers2020parot}.
A promising set of conservative methods poses the problem as a convex relaxation of the original nonlinear network architecture, and solves it using a linear programming solver \citep{salman2019convex,zhang2022general} or by obtaining closed-form bounds \citep{zhang2018efficient,wang2021beta}. The latter are especially appealing due to their efficiency. Examples include CROWN \citep{zhang2018efficient} and $\alpha$-CROWN \citep{xu2020fast}. \citet{xu2020automatic} extended the linear relaxation framework from \citet{zhang2018efficient} to general computation graphs, but the purely backward propagation nature makes it potentially less efficient than custom bounds/hybrid approaches~\citep{shi2020robustness}. Our work adapts techniques from verification to certify the \textit{full} applicability domain of PINNs, in a similar fashion to the \textit{global} specification from \citet{muller2023abstract}.

\section{Preliminaries}
\label{sec:preliminaries}

\subsection{Notation}
\label{sec:notation}

Given vector $\vec{a} \in \mathbb{R}^d$, $\vec{a}_i$ refers to its \mbox{$i$-th} component. We use $\partial_{\x_i^j} f$ and $\frac{\partial^j f}{(\partial \x_i)^j}$ interchangeably to refer to the $j$-th partial derivative of a function $f: \mathbb{R}^{n} \to \mathbb{R}$ with respect to the $i$-component of its input, $\x_i$. Where it is clear, we use $f(\x)$ and $f$ interchangeably. 
We take $\lin{W}{b}{i}{}(\x) = \mathbf{W}^{(i)}\x + \mathbf{b}^{(i)}$ to be a function of $\x$ parameterized by weights $\mathbf{W}^{(i)}$ and bias $\mathbf{b}^{(i)}$. We define an $L$-layer \textit{fully connected neural network} $g: \mathbb{R}^{d_0} \to \mathbb{R}^{d_L}$ for an input $\x$ as $g(\x) = \preout^{(L)}(\x)$ where 
$\preout^{(k)}(\x) = \lin{W}{b}{k}{}(\out^{(k-1)}(\x))$, $\out^{(k-1)}(\x) = \sigma(\preout^{(k-1)}(\x))$, $\out^{(0)}(\x) = \x$, in which $\W^{(k)} \in \mathbb{R}^{d_k\times d_{k-1}}$ and $\bias^{(k)} \in \mathbb{R}^{d_k}$ are the weight and bias of layer $k$, $\sigma$ is the nonlinear activation, and $k\in\{1,\dots,L\}$. 

\subsection{Physics-informed neural networks (PINNs)}
\label{sec:prelim_pinns}

We consider general nonlinear PDEs of the form:
\begin{equation}
\label{eq:pinn-definition}
    f(t, \xspatial) = \partial_t u(t, \xspatial) + \mathcal{N}[u](t, \xspatial) = 0,\; \xspatial\in \mathcal{D}, t\in[0, T],
\end{equation}
where $f$ is the residual of the PDE, $t$ is the temporal and $\xspatial$ is the spatial
components of the input, $u: [0, T] \times \mathcal{D} \to \mathbb{R}$ is the solution, $\mathcal{N}$ is a nonlinear differential operator on $u$, $T\in \mathbb{R}^+$, and $\mathcal{D}\subset \mathbb{R}^D$. Where possible, for conciseness we will use $\x = (t, \xspatial)$, for $\x \in \mathcal{C} = [0, T] \times \mathcal{D}$, with $\x_0 = t$.

We assume $f$ is the residual of an $R^{th}$ order PDE where the differential operators of $\mathcal{N}$ applied to $u$ yield the partial derivatives for order $\{0, ..., R\}$ as: $u \in \mathcal{N}^{(0)}$, $\partial_{\x_i} u \in \mathcal{N}^{(1)}$, $\partial_{\x_i^2} u \in \mathcal{N}^{(2)}$, $\dots$, $\partial_{\x_i^R} u \in \mathcal{N}^{(R)}$ for $i\in \{0,\dots,D\}$\footnote{For simplicity, we assume $\mathcal{N}$ does not contain any cross-derivative operators, yet an extension would be trivial to derive.}. With these, we can re-write $f = \mathcal{P}(u, \partial_{\x_0} u, \dots, \partial_{\x_{D}} u, \dots, \partial_{\x_{D}^R} u)$, where $\mathcal{P}$ is a nonlinear function of those terms. Furthermore, the PDE is defined under (1) initial conditions, \ie, $u(0, \xspatial) = u_0(\xspatial)$, for $\xspatial \in \mathcal{D}$, and (2) general Robin boundary conditions, \ie, $a u(t, \xspatial) + b \partial_{\vec{n}} u(t,\xspatial) = u_b(t, \xspatial)$ for $a, b \in \mathbb{R}$, $\xspatial \in \delta \mathcal{D}$ and $t \in [0, T]$, and $\partial_{\vec{n}} u$ is the normal derivative at the border with respect to some components of $\xspatial$.

Continuous-time PINNs \citep{raissi2019physics} result from approximating the solution, $u(\x)$, using a neural network parameterized by $\theta$, $\solnet(\x)$. We refer to this network as the \textit{approximate solution}. In that context, the \textit{physics-informed neural network} (or residual) is $\resnet(\x) = \partial_t \solnet(\x) + \mathcal{N}[\solnet](\x)$. For example, the one-dimensional Burgers' equation (explored in detail in Section \ref{sec:experiments}) is defined as:
\begin{equation}
\label{eq:burgers_example}
\begin{aligned}
f_\theta(\x) = \partial_t \solnet(\x) &+ \solnet(\x) \partial_x \solnet(\x) - (0.01/\pi)\partial_{x^2}\solnet(\x).
\end{aligned}
\end{equation}
Note $\resnet$ has the same order as $f$, and can be described similarly as a nonlinear function with the partial derivatives applied to $\solnet$ instead of $u$. Burgers' equation (from above) has one $0^{th}$ order term ($\solnet$), two $1^{st}$ order ones ($\partial_t \solnet$ and $\partial_x \solnet$), and a $2^{nd}$ order partial derivative ($\partial_{x^2} \solnet$), while $\solnet(\x) \partial_x \solnet(\x)$ is a nonlinear term of the $\resnet$ polynomial.

\subsection{Bounding neural network outputs using CROWN \texorpdfstring{\citep{zhang2018efficient}}{(Zheng et al., 2018}}
\label{sec:prelim_verification}

The computation of upper/lower bounds on the output of neural networks over a domain has been widely studied within verification of image classifiers~\citep{katz2017reluplex, mirman2018differentiable, zhang2018efficient}. For the sake of computational efficiency, we consider the bounds obtained using CROWN \citep{zhang2018efficient}/$\alpha$-CROWN \citep{xu2020fast} as the base for our framework.

Take $g$ to be the fully connected neural network (as defined in Section~\ref{sec:notation}) we're interested in bounding. The goal is to compute $\max/\min_{x \in \mathcal{C}} g(\x)$, where $\mathcal{C}$ is the applicability domain. Typically within verification of image classifiers, $\mathcal{C} = \mathbb{B}^p_{\x, \epsilon} = \{\x': \|\x' - \x\|_p \leq \epsilon\}$, \ie, it is a \textit{local} $\ell_p$-ball of radius $\epsilon$ around an input $\x$ from the test set.

CROWN solves the optimization problem by \textit{back-propagating} linear bounds of $g(\x)$ through each hidden layer of the network until the input is reached. To do so, assuming constant bounds on $\preout^{(k)}(\x)$ are known for $\x \in \mathcal{C}$, \ie, $\forall \x \in \mathcal{C}: \preout^{(k), L} \leq \preout^{(k)}(\x) \leq \preout^{(k), U}$, CROWN relaxes the nonlinearities of each $\out^{(k)}$ using a linear lower and upper bound approximation that contains the full possible range of $\sigma(\preout^{(k)}(\x))$. By relaxing the activations of each layer and back-propagating it through $z^{(k)}$, CROWN obtains a bound on each $y^{(k)}$ as a function of $y^{(k-1)}$. Back-substituting from the output $y^{(L)} = g(\x)$ until the input $\x$ results in:
\begingroup\makeatletter\def\f@size{10}\check@mathfonts
$$
\min_{\x \in \mathcal{C}}\, g(\x) \geq \min_{\x \in \mathcal{C}} \mathbf{A}^{L} \x + \mathbf{a}^L,\; \max_{\x \in \mathcal{C}}\, g(\x) \leq \max_{\x \in \mathcal{C}} \mathbf{A}^{U} \x + \mathbf{a}^U,
$$
\endgroup
where $\mathbf{A}^L$, $\mathbf{a}^L$, $\mathbf{A}^U$ and $\mathbf{a}^U$ are computed in polynomial time from $\W^{(k)}, \bias^{(k)}$, and the linear relaxation parameters. The solution to the optimization problems above given simple constraints $\mathcal{C}$ can be obtained in closed-form. $\alpha$-CROWN \citep{xu2020fast} improves these bounds by optimizing the linear relaxations of $\sigma(\preout^{(k)}(\x))$ for tightness.
 
\section{\texorpdfstring{$\partial$}{∂}-CROWN: Error Certification for General Physics-Informed Neural Networks}
\label{sec:methods}

Take $\solnet$ to be the learned approximate continuous solution of the PDE $f$ through the PINN $\resnet$. Previous works deriving from \citet{raissi2019physics} have measured the \textit{correctness} of $\solnet$ empirically by computing the solution error at a set of discrete point compared to that obtained via numerical solvers for $f$ \citep{takamoto2022pdebench,monaco2023training} -- a compromise arising from the fact we cannot bound $\|\solnet - u\|$ for general PDEs across their continuous domain.

To mitigate this issue for continuous-time PINNs, we approach the problem of error bounding by imposing correctness conditions on the \textit{residual} instead of the solution error. By definition, $\solnet$ is a correct solution to the PINN $\resnet$ if 3 conditions are met: \normalfont{\circled{1}} the norm of the solution error with respect to the initial condition is upper bounded by an acceptable tolerance, \normalfont{\circled{2}} the norm of the solution error with respect to the boundary conditions is bounded by an acceptable tolerance, and \normalfont{\circled{3}} the norm of the residual is bounded by an acceptable tolerance. We define these as PINN \textit{correctness conditions}, and formalize it in Definition~\ref{def:correctness}. Note these conditions are general and, at this point, no assumptions are made about $u_{\theta}$ or the PDE.

\begin{definition}[Correctness Conditions for PINNs]
\label{def:correctness}
$\solnet: [0, T]\times \mathcal{D} \to \mathbb{R}$ is a $\delta_0,\delta_b,\varepsilon$-globally correct approximation of the exact solution $u: [0, T]\times \mathcal{D} \to \mathbb{R}$ if:
\begin{equation*}
\begin{split}
\normalfont{\circled{1}} & \quad\max_{\xspatial\in \mathcal{D}} |\solnet(0, \xspatial) - u_0(\xspatial)|^2 \leq \delta_0, \\
\normalfont{\circled{2}} & \quad\max_{t\in[0,T], \xspatial\in\delta\mathcal{D}} |a \solnet(\x) + b \partial_{\vec{n}} \solnet(\x) - u_b(\x)|^2 \leq \delta_b, \\
\normalfont{\circled{3}} & \quad \max_{\x \in \mathcal{C}} |\resnet(\x)|^2 \leq \varepsilon.
\end{split}
\end{equation*}
\end{definition}
In practice, $\delta_0$, $\delta_b$, and $\varepsilon$ correspond to tolerances similar to the ones given by numerical solvers for $f$. While the residual error upper bound is similar in nature to the empirical errors used to monitor convergence in iterative solvers (e.g., in Krylov subspace methods for linear systems), the bound proposed here corresponds to the error of the continuous approximate solution $\solnet$ instead of the discretized version provided in those solvers. In Section \ref{sec:experiments}, we empirically analyze the connection between residual and solution errors using a numerical solver.

\begin{figure*}[t]
    \begin{center}
        \includegraphics[width=\textwidth]{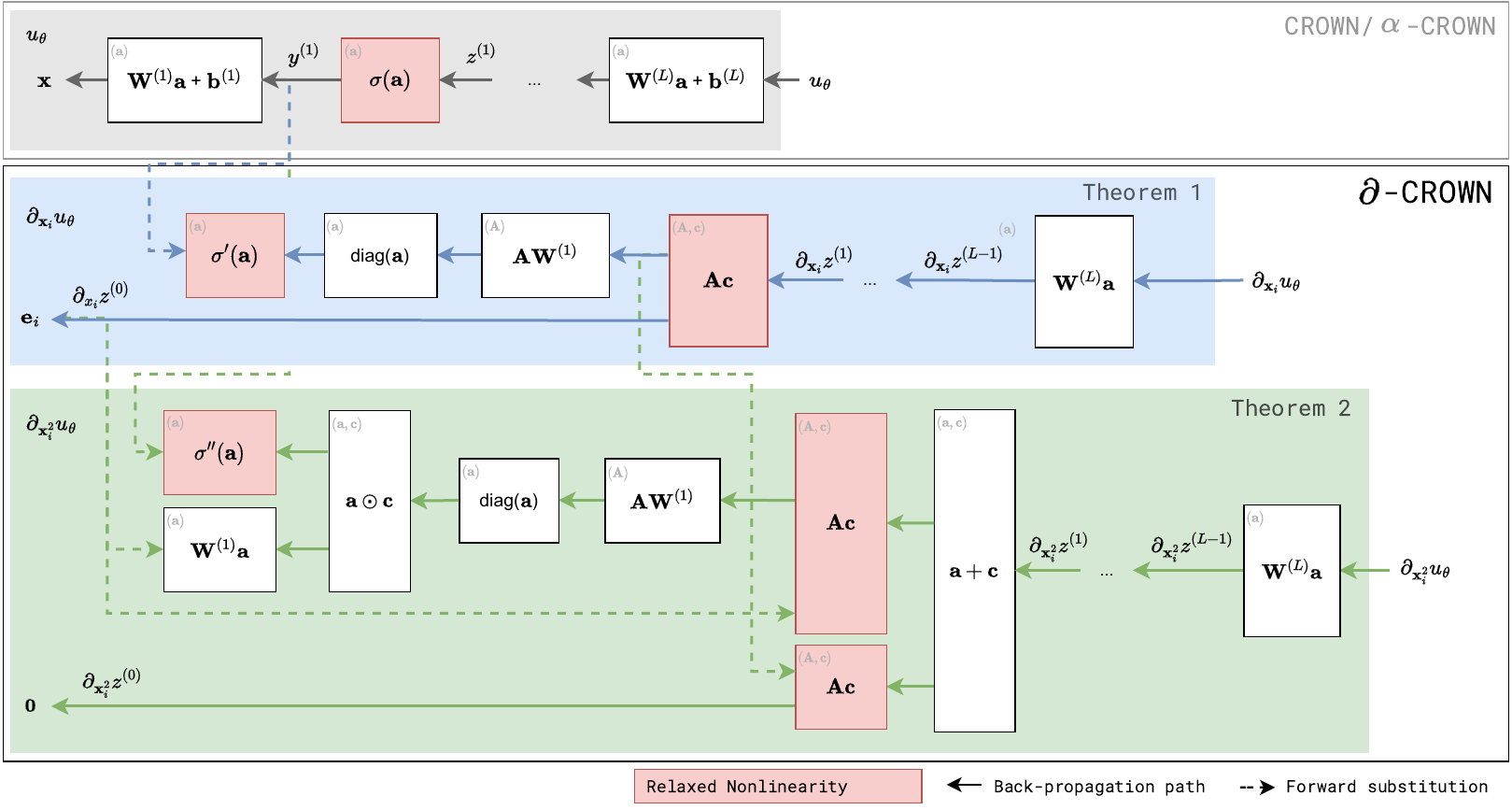}
    \end{center}
    \vspace{-1em}
    \caption{\textbf{Bounding Partial Derivatives with \partialcrown}: our hybrid scheme for bounding $\partial_{\mathbf{x}_i} \solnet$ and $\partial_{\mathbf{x}_i^2} \solnet$ uses back-propagation and forward substitution (inspired by~\citet{shi2020robustness}) to compute bounds in $\mathcal{O}(L)$ instead of the $\mathcal{O}(L^2)$ complexity of full back-propagation as in~\citet{xu2020automatic}.}
    \label{fig:partial-crown-graph}
\end{figure*}

The verification of the conditions from Definition \ref{def:correctness} requires bounding: a linear function of $\solnet$ for \normalfont{\circled{1}}, a linear function of $\solnet$ and $\partial_{\vec{n}} \solnet$ for \normalfont{\circled{2}}, and the PINN output, $\resnet$, in \normalfont{\circled{3}}. To achieve \normalfont{\circled{1}}, assuming $\solnet$ is a standard fully connected neural network as in \citet{raissi2019physics}, we can directly use CROWN/$\alpha$-CROWN \citep{zhang2018efficient,xu2020fast}. However, bounding \normalfont{\circled{2}} and \normalfont{\circled{3}} with a linear function in $\mathbf{x}$ efficiently requires a method to bound linear and nonlinear functions of the partial derivatives of $\solnet$. 

We propose \partialcrown, an efficient framework to: (i) compute closed-form bounds on the partial derivatives of an arbitrary fully-connected network $\solnet$ (Section~\ref{sec:linear_operators}), and (ii) bound a nonlinear function of those partial derivative terms, \ie, $\resnet$ (Section~\ref{sec:multiplicative_terms}). Throughout this section, we assume $\solnet(\x) = g(\x)$ as defined in Section~\ref{sec:notation}, with $d_0 = D+1$. Formal statements and proofs for the lemmas and theorems presented in this section are in Appendix \ref{app:proofs}.

\subsection{Bounding Partial Derivatives of \texorpdfstring{$\solnet$}{u\_θ}}
\label{sec:linear_operators}

The computation of the bounds for the $0^{th}$ order derivative, \ie, $\solnet$, and intermediate pre-activations can be computed using CROWN/$\alpha$-CROWN~\citep{zhang2018efficient,xu2020fast}. As such, for what follows, we assume that for $\x \in \mathcal{C}$, both the bounds on $\solnet$ and $y^{(k)},~\forall k$ are given.

\begin{assumption}
The pre-activation layer outputs of $\solnet$, $\preout^{(k)} = \lin{W}{b}{k}{}(\out^{(k-1)})$, are lower and upper bounded by linear functions $\lin{A}{a}{k}{, L}(\x) \leq \preout^{(k)} \leq \lin{A}{a}{k}{, U}(\x)$. Moreover, for $\x\in \mathcal{C}$, we have $\preout^{(k), L} \leq \preout^{(k)} \leq \preout^{(k), U}$.
\end{assumption}

Note that using CROWN/$\alpha$-CROWN, $\mathbf{A}^{(k), L}$, $\mathbf{a}^{(k), L}$, $\mathbf{A}^{(k), U}$, $\mathbf{a}^{(k), U}$ are functions of all the previous layers' parameters. For $1^{st}$ order derivatives, we start by explicitly obtaining the expression of $\partial_{\x_i} \solnet$.

\begin{lemma}[Expression for $\partial_{\x_i} \solnet$]
\label{lem:first_derivative}
For $i\in\{1,\dots,d_0\}$, the partial derivative of $\solnet$ with respect to $\x_i$ can be computed recursively as $\partial_{\x_i} \solnet = \W^{(L)} \partial_{\x_i} \out^{(L-1)}$ for:
$$
\partial_{\x_i} \out^{(k)} = \partial_{\out^{(k-1)}} \out^{(k)} \partial_{\x_i} \out^{(k-1)},\quad \partial_{\x_i} \out^{(0)} = \vec{e}_i,
$$
for $k \in \{1,\dots,L-1\}$, and where $\partial_{\out^{(k-1)}} \out^{(k)} = \text{diag}\left[\sigma'\left(y^{(k)}\right)\right]\W{^{(k)}}$.
\end{lemma}

Using this result, we can efficiently linearly lower and upper bound $\partial_{\x_i} \solnet$.

\begin{theorem}[Informal, \partialcrown Linear Bounding $\partial_{\x_i} \solnet$]
\label{thm:first_derivative_bounds}
There exist two linear functions $\partial_{\x_i} u^{U}_{\theta}$ and $\partial_{\x_i} u^{L}_{\theta}$ such that, $\forall \x\in \mathcal{C}$ it holds that $\partial_{\x_i} u^{L}_{\theta} \leq \partial_{\x_i} u_{\theta} \leq \partial_{\x_i} u^{U}_{\theta}$, where the linear coefficients can be computed recursively in closed-form in $\mathcal{O}(L)$ time as a function of $\mathbf{W}^{(k)}$, $\mathbf{A}^{(k), L}$, $\mathbf{a}^{(k), L}$, $\mathbf{A}^{(k), U}$, $\mathbf{a}^{(k), U}$, $\mathbf{y}^{(k), L}$, and $\mathbf{y}^{(k), U}$.
\end{theorem}

The formal statement of Theorem~\ref{thm:first_derivative_bounds} and expressions for $\partial_{\x_i} u^{L}_{\theta}$ and $\partial_{\x_i} u^{U}_{\theta}$ are provided in Appendix \ref{app:proof_bounding_first_derivative}. Note that this bound is not computed using fully backward propagation as in~\citet{xu2020automatic}. Instead we use a \textit{hybrid} scheme in the spirit of~\citet{shi2020robustness} for the sake of efficiency. We perform backward propagation to compute $\partial_{\out^{(k-1)}} \out^{(k)}$ as a function of $y^{(k)}$, and forward-substitute the pre-computed CROWN bounds $\lin{A}{a}{k}{, L}(\x) \leq \preout^{(k)} \leq \lin{A}{a}{k}{, U}(\x)$ at that point instead of fully backward propagating which would have $\mathcal{O}(L^2)$ complexity. This induces a significant speed-up while achieving tight enough bounds.  Figure~\ref{fig:partial-crown-graph} showcases the back-propagation and forward substitution paths for bounding $\partial_{\x_i} \solnet$ in blue. Similarly to CROWN with the activation $\sigma$, this bound requires relaxing $\sigma'(y^{(k)})$.

Similarly, we can linearly bound $\partial_{\x_i^2} \solnet$, a requirement to bound $\resnet$ in $2^{nd}$ order PINNs. 

\begin{lemma}[Expression for $\partial_{\x_i^2} \solnet(\x)$]
\label{lem:second_derivative}

For $i\in\{1,\dots,d_0\}$, the second partial derivative of $\solnet$ with respect to $\x_i$ can be computed recursively as $\partial_{\x_i^2} \solnet = \W^{(L)} \partial_{\x_i^2} \out^{(L-1)}$ where:
$$
\partial_{\x_i^2} \out^{(k)} = \partial_{x_i \out^{(k-1)}} \out^{(k)} \partial_{\x_i} \out^{(k-1)} + \partial_{\out^{(k-1)}} \out^{(k)} \partial_{\x_i^2} \out^{(k-1)},
$$
and $\partial_{\x_i^2} \out^{(0)} = \vec{0}$, for $k \in \{1,\dots,L-1\}$, with $\partial_{\x_i} \out^{(k-1)}$ and $\partial_{\out^{(k-1)}} \out^{(k)}$ as per in Lemma \ref{lem:first_derivative}, and $\partial_{x_i \out^{(k-1)}} \out^{(k)} = \text{diag}\left[\sigma''\left(y^{(k)}\right) \left(\W^{(k)} \partial_{\x_i} \out^{(k-1)}\right)\right] \W^{(k)}$.
\end{lemma}

\begin{theorem}[Informal, \partialcrown Linear Bounding $\partial_{\x_i^2} \solnet$]
\label{thm:second_derivative_bounds}
Assume that through a previous bounding of $\partial_{\x_i} \solnet$, we have linear lower and upper bounds on $\partial_{\x_i} \out^{(k-1)}$ and $\partial_{\out^{(k-1)}} \out^{(k)}$.
There exist two linear functions $\partial_{\x_i^2} u^{U}_{\theta}$ and $\partial_{\x_i^2} u^{L}_{\theta}$ such that, $\forall \x\in \mathcal{C}$ it holds that $\partial_{\x_i^2} u^{L}_{\theta} \leq \partial_{\x_i^2} u_{\theta} \leq \partial_{\x_i^2} u^{U}_{\theta}$, where the linear coefficients can be computed recursively in closed-form in $\mathcal{O}(L)$ time as a function of $\mathbf{W}^{(k)}$, $\mathbf{A}^{(k), L}$, $\mathbf{a}^{(k), L}$, $\mathbf{A}^{(k), U}$, $\mathbf{a}^{(k), U}$, $\mathbf{y}^{(k), L}$, $\mathbf{y}^{(k), U}$, and the parameters of the linear lower and upper bounds on $\partial_{\x_i} \out^{(k-1)}$ and $\partial_{\out^{(k-1)}} \out^{(k)}$.
\end{theorem}

The formal statement of Theorem~\ref{thm:second_derivative_bounds} and expressions for $\partial_{\x_i^2} u^{L}_{\theta}$ and $\partial_{\x_i^2} u^{U}_{\theta}$ are in Appendix \ref{app:proof_bounding_second_derivative}. As with the first derivative, this bound requires a relaxation of $\sigma''(y^{(k)})$. Note that this also follows a hybrid computation scheme, with the back-propagation and forward substitution paths for bounding $\partial_{\x_i^2} u_{\theta}$ computations shown in green in Figure~\ref{fig:partial-crown-graph}.

Assuming $\mathcal{C} = \{\x \in \mathbb{R}^{d_0}: \x^L \leq \x \leq \x^U\}$, we can obtain closed-form expressions for constant global bounds on the linear functions $\partial_{\x_i} \solnet^U$, $\partial_{\x_i} \solnet^L$, $\partial_{\x_i^2} \solnet^U$, $\partial_{\x_i^2} \solnet^L$, which we formulate and prove in Appendix~\ref{app:proof_lemma_closed_form_bounds}. While here we only compute the expression for the second derivative with respect to the same input, it would be trivial to extend it to cross derivatives (\ie, $\partial_{\x_i \x_j} \solnet$ for $i\neq j$).

\subsection{Bounding \texorpdfstring{$\resnet$}{f\_θ}}
\label{sec:multiplicative_terms}

With the partial derivative terms bounded, to bound $\resnet$, we use McCormick envelopes~\citep{mccormick1976computability} to obtain linear lower and upper bound functions $f_{\theta}^L \leq \resnet \leq f_{\theta}^U$: $f_{\theta}^U = \mu^U_{0} + \mu^U_{1} \solnet + \sum_{j=1}^r \sum_{\partial_{\x_i^j}\in \mathcal{N}^{(j)}} \mu^U_{j,i} \partial_{\x_i^j} \solnet$, and $f_{\theta}^L = \mu^L_{0} + \mu^L_{1} \solnet + \sum_{j=1}^r \sum_{\partial_{\x_i^j}\in \mathcal{N}^{(j)}} \mu^L_{j,i} \partial_{\x_i^j} \solnet,$
where $\mu^U_{0}$, $\mu^U_{1}$, and $\mu^U_{i, j}$ are functions of the global lower and upper bounds of $\solnet$ and $\partial_{\x_i^j} \solnet$. In the example of Burgers' equation (Equation~\ref{eq:burgers_example}), $f_\theta^U = \mu^U_{0} + \mu^U_{1} \solnet + \mu^U_{1,0} \partial_{\x_0} \solnet + \mu^U_{1,1} \partial_{\x_1} \solnet + \mu^U_{2,1} \partial_{\x_1^2} \solnet$ (and similarly for $f_\theta^L$ with $\mu^L$).

\begin{algorithm}[t]
\KwInput{function $h$, input domain $\mathcal{C}$, \# splits $N_b$, \# empirical samples $N_s$, \# branches per split $N_d$}
\KwResult{lower bound $h_{lb}$, upper bound $h_{ub}$}
$\mathcal{B}, \mathcal{B}_\Delta$ = $\emptyset, \emptyset$

$\hat{h}_{lb}, \hat{h}_{ub} $ = $\min\text{\textbackslash} \max$ $h$(\textsc{Sample}($\mathcal{C}$, $N_s$))

$h_{lb}, h_{ub}$ = \partialcrown$(h, \mathcal{C})$

$\mathcal{B}[\mathcal{C}]$ = $(h_{lb}, h_{ub})$

$\mathcal{B}_\Delta[\mathcal{C}]$ = $\max(\hat{h}_{lb} - h_{lb}, h_{ub} - \hat{h}_{ub})$

\For{$i\in \{1, \dots, N_b\}$}{
    $\mathcal{C}_i$ = $\mathcal{B}$.\textsc{Pop}$(\argmax_{\mathcal{C}'} \mathcal{B}_{\Delta}[\mathcal{C}'])$
    
    \ForEach{$\mathcal{C}' \in \,$\textsc{DomainSplit}($\mathcal{C}_i$, $N_d$)}{
        $h_{lb}', h_{ub}'$ = \partialcrown$(h, \mathcal{C}')$
        
        $\mathcal{B}[\mathcal{C}']$ = $(h_{lb}', h_{ub}')$
        
        $\mathcal{B}_\Delta[\mathcal{C}']$ = $\max(\hat{h}_{lb} - h_{lb}', h_{ub}' - \hat{h}_{ub})$
    }
}
$h_{lb}, h_{ub} = \min_{\mathcal{C'}} \mathcal{B}_0[\mathcal{C}'], \max_{\mathcal{C'}} \mathcal{B}_1[\mathcal{C}']$

\Return $h_{lb}$, $h_{ub}$
\caption{Greedy Input Branching}
\label{alg:greedy_input_branching}
\end{algorithm}

To get $\resnet^U$ and $\resnet^L$ as linear functions of $\x$, we replace $\solnet$ and $\partial_{\x_i^j} \solnet$ with the lower and upper bound linear expressions from Section~\ref{sec:linear_operators}, depending on the sign of the coefficients $\mu^U$ and $\mu^L$. As in Section~\ref{sec:linear_operators}, since $\mathcal{C} = \{\x \in \mathbb{R}^{d_0}: \x^L \leq \x \leq \x^U\}$ we can then solve $\max_{\x \in \mathcal{C}} \resnet^U$ and $\min_{\x \in \mathcal{C}} \resnet^L$ in closed-form (see Appendix~\ref{app:proof_lemma_closed_form_bounds}), obtaining constant bounds for $\resnet$ in $\mathcal{C}$. We explore the overall complexity of running \partialcrown to bound $\resnet$ in Appendix \ref{app:complexity}, and define it generally as $\mathcal{M}$ for the sake of further complexity analysis.

\begin{table*}[t]
    \caption{\textbf{Certifying with \partialcrown}: empirical lower bounds ($l_b$) computed using Monte Carlo (MC) samples ($10^4$ and $10^6$ points), and certified upper bounds ($u_b$) using \partialcrown with greedy input branching for \protect\circled{1} initial conditions, \protect\circled{2} boundary conditions, and \protect\circled{3} residual condition for (a) Burgers~\citep{raissi2019physics}, (b) \schrodinger~\citep{raissi2019physics}, (c) Allen-Cahn~\citep{monaco2023training}, and (d) Diffusion-Sorption~\citep{takamoto2022pdebench} equations.}
    \label{tab:certifying_crown}
    \center
    {
    \vspace{-1em}
    \footnotesize
    \begin{tabular}{p{0.14\textwidth}p{0.01\textwidth}p{0.23\textwidth}P{0.13\textwidth}P{0.13\textwidth}P{0.22\textwidth}}
        & & & \multicolumn{2}{c}{Empirical $l_b$} & Certified $u_b$\\\cmidrule[1pt]{4-6}
        & & & MC $\max$ ($10^4$) & MC $\max$ ($10^6$) & \partialcrown $u_b$ (time [s])\\ \cmidrule[1pt]{1-6}
        
        \multirow{4}{0.14\textwidth}{\textbf{(a) Burgers}} & \circled{1} & $|\solnet(0, x) - u_0(x)|^2$ & $1.59\times10^{-6}$ & $1.59\times10^{-6}$ & $2.63\times10^{-6}$ ($116.5$) \\\cline{2-6}
        & \multirow{2}{*}{\circled{2}} & $|\solnet(t, -1)|^2$ & $8.08\times10^{-8}$ & $8.08\times10^{-8}$ & $6.63\times10^{-7}$ ($86.7$) \\
        & & $|\solnet(t, 1)|^2$ & $6.54\times10^{-8}$ & $6.54\times10^{-8}$ & $9.39\times10^{-7}$ ($89.8$) \\\cline{2-6}
        &\circled{3} & $|\resnet(\x)|^2$ & $1.23\times10^{-3}$ & $1.80\times10^{-2}$ & $1.03\times10^{-1}$ ($2.8\times 10^5$) \\ \specialrule{0.75pt}{1pt}{1pt}

        \multirow{4}{0.14\textwidth}{\textbf{(b) \schrodinger}} & \circled{1} & $|\solnet(0, x) - u_0(x)|^2$ & $7.06\times10^{-5}$ & $7.06\times10^{-5}$ & $8.35\times10^{-5}$ ($305.2$) \\\cline{2-6}
        & \multirow{2}{*}{\circled{2}} & $|u_\theta(t, 5) - u_\theta(t, -5)|^2$ & $7.38\times10^{-7}$ & $7.38\times10^{-7}$ & $5.73\times10^{-6}$ ($545.4$) \\
        & & $|\partial_x u_\theta(t, 5) - \partial_x u_\theta(t, -5)|^2$ & $1.14\times10^{-5}$ & $1.14\times10^{-5}$ & $5.31\times10^{-5}$ ($2.4\times 10^3$) \\\cline{2-6}
        & \circled{3} & $|\resnet(\x)|^2$ & $7.28\times10^{-4}$ & $7.67\times10^{-4}$ & $5.55\times10^{-3}$ ($1.2\times 10^6$) \\ \specialrule{0.75pt}{1pt}{1pt}

        \multirow{3}{0.14\textwidth}{\textbf{(c) Allen-Cahn}} & \circled{1} & $|\solnet(0, x) - u_0(x)|^2$ & $1.60\times10^{-3}$ & $1.60\times10^{-3}$ & $1.61\times10^{-3}$ ($52.7$) \\ \cline{2-6}
        & \circled{2} & $|\solnet(t, -1) - \solnet(t, 1)|^2$ & $5.66\times10^{-6}$ & $5.66\times10^{-6}$ & $5.66\times10^{-6}$ ($95.4$) \\\cline{2-6}
        & \circled{3} & $|\resnet(\x)|^2$ & $10.74$ & $10.76$ & $10.84$ ($6.7\times 10^{5}$) \\ \specialrule{0.75pt}{1pt}{1pt}

        \multirow{4}{0.14\textwidth}{\textbf{(d) Diffusion-Sorption}} & \circled{1} & $|\solnet(0, x)|^2$ & $0.0$ & $0.0$ & $0.0$ ($0.2$) \\ \cline{2-6}
        & \multirow{2}{*}{\circled{2}} & $|u_\theta(t, 0) - 1|^2$ & $4.22\times 10^{-4}$ & $4.39\times 10^{-4}$ & $1.09\times 10^{-3}$ ($72.5$) \\
        & & $|u_\theta(t, 1) - D \partial_x u_\theta(t, 1)|^2$ & $2.30\times 10^{-5}$ & $2.34\times 10^{-5}$ & $2.37\times 10^{-5}$ ($226.4$) \\ \cline{2-6}
        & \circled{3} & $|\resnet(\x)|^2$ & $1.10\times 10^{-3}$ & $21.09$ & $21.34$ ($2.4\times 10^{6}$) \\ \specialrule{1pt}{1pt}{1pt}
    \end{tabular}
    }
\end{table*}

\subsection{Tighter Bounds via Greedy Input Branching}
\label{sec:input_branching}

Using \partialcrown we can compute a bound on a nonlinear function of the derivatives of $\solnet$, which we will generally refer to as $h$, for $\x \in \mathcal{C}$. However, given the approximations introduced by the relaxations, it is likely these bounds will be too loose compared to the true values of $h$ to be useful.

To improve them, we introduce \textit{greedy input branching} (Algorithm \ref{alg:greedy_input_branching}). 
We start by computing empirical estimates of the min/max value of $h$ across the domain (L2), and the \partialcrown bounds over the full domain (L3), storing the latter in the certified bounds list, $\mathcal{B}$, (L4) and the max difference between empirical and certified in the list $\mathcal{B}_\Delta$ (L5). For $N_b$ iterations (\textit{branchings}), we take $\mathcal{C}_i$ as the interval with the highest difference between empirical and certified values (L7). We then split it into $N_d$ pieces using \textsc{DomainSplit}, compute the new certificates for those smaller sub-domains $\mathcal{C}'$ (L9), and add those certified bounds and their error w.r.t. the empirical estimate of the bounds to $\mathcal{B}$ (L10) and $\mathcal{B}_\Delta$ (L12), respectively. Finally, the tighter lower and upper bounds are then the minimum lower bound and the maximum upper bound in $\mathcal{B}$, respectively (L12). A more detailed step-by-step description of the algorithm is given in Appendix \ref{app:greedy-branching-details}. 

As the number of splits, $N_b$, increases, so does the tightness of our global bounds. For small dimensional spaces, it suffices to split each branch $\mathcal{C}_i$ into $N_d = 2^{d_0}$ equal branches. Note that in higher dimensional spaces, a non-equal splitting function, \textsc{DomainSplit}, can lead to improved convergence to the tighter bounds. The time complexity of greedy input branching is $\mathcal{O}(N_b N_d \mathcal{M})$, where $\mathcal{M}$ is the complexity of bounding each branch. 

\section{Experiments}
\label{sec:experiments}

The aim of this experimental section is to (i) showcase that the Definition~\ref{def:correctness} certificates obtained with \partialcrown are tight compared to empirical errors computed with a large number of samples (Section~\ref{sec:exp_certification}), (ii) highlight the relationship of our residual-based certificates and the commonly reported solution errors (Section~\ref{sec:residual-solution-connection}), (iii) compare the efficiency of our method to an alternative bound propagation one (Section~\ref{sec:efficiency-experiments}), and (iv) qualitatively analyze the importance of greedy input branching in the success of our method (Section~\ref{sec:importance-branching}). On top of these results, in Appendix~\ref{app:adversarial-training-and-analysis} we study how the training method from~\citet{shekarpaz2022piat} can lead to a reduction in empirical and certified errors, and in Appendix \ref{app:failure-identification} we showcase how \partialcrown can be used to identify failures in PINN training. 

\subsection{Certifying with \texorpdfstring{$\partial$}{∂}-CROWN}
\label{sec:exp_certification}

To achieve (i), we apply our post-training certification framework \partialcrown to two widely studied PINNs from \citet{raissi2019physics}, Burgers' and \schrodinger's equations, as well as to the more complex Allen-Cahn's equation from \citet{monaco2023training}, and the Diffusion-Sorption equation from \citet{takamoto2022pdebench}. These PINNs were chosen for the experimental section as they are well established from previous literature in the field, and either code or trained models were available from that previous work. While we considered other suitable higher dimensional PINNs, such as several of the Navier-Stokes equations from \citet{jin2021nsfnets}, or the Gray-Scott system from \citet{giampaolo2022physics}, neither training code nor the pre-trained models were released that allow us to apply \partialcrown.

Since $\solnet$ for these PINNs use $\sigma=\tanh$ activations, we need to be able to linearly relax $\sigma'$ and $\sigma''$ given pre-activation bounds. We propose a practical relaxation in Appendix~\ref{app:tanh_relaxations}, highlighting its effeciency compared to a simple baseline in Appendix \ref{app:tanh_relaxations_ablation}. All timing results were obtained on a MacBook Pro with a 10-core M1 Max CPU. Visualizations of a fine-grained discretization of the solution and residual error landscapes is provided in Figure~\ref{fig:all_residuals} in the Appendix.

\paragraph{Burgers' Equation} This one-dimensional PDE is used in several areas of mathematics, fluid dynamics, nonlinear acoustics, gas dynamics and traffic flow, and is derived from the Navier-Stokes equations for the velocity field by dropping the pressure gradient~\citep{raissi2019physics}. It is defined on a temporal domain $t \in [0,1]$ and spatial domain $x\in [-1, 1]$ as:
\begin{equation}
\label{eq:burgers}
    \partial_t u(t, x) + u(t,x) \partial_x u(t, x) - (0.01/\pi)\partial_{x^2}u(t, x) = 0,
\end{equation}
for $u(0, x) = -\sin(\pi x)$, $u(t, -1) = u(t, 1) = 0$. The solution $\solnet: \mathbb{R}^2 \to \mathbb{R}$ is modeled by an 8-hidden layer, 20 neurons per layer network \citep{raissi2019physics}. The training process took $\sim 13.35$ minutes, and resulted in a mean $\ell_2$ solution error of $6.1\cdot 10^{-4}$.

\paragraph{\schrodinger's Equation} \schrodinger's equation is a classical field equation used to study quantum mechanical systems. In~\citet{raissi2019physics}, \schrodinger's equation is defined with the temporal domain $t \in [0, \pi/2]$ and spatial domain $x \in [-5, 5]$ as:
\begin{equation}
    i\, \partial_t u(t, x) + 0.5\,\partial_{xx}u(t, x) + |u(t,x)|^2 u(t,x) = 0,
\end{equation}
where $u: [0, \pi/2] \times \mathcal{D} \to \mathbb{C}$ is a complex-valued solution, for initial conditions $u(0, x) = 2\,\sech(x)$, and periodic boundary conditions $u(t, -5) = u(t, 5)$ and $\partial_x u(t, -5) = \partial_x u(t, 5)$. As in~\citet{raissi2019physics}, $u_\theta: \mathbb{R}^{2} \to \mathbb{R}^2$ is a 5-hidden layer, 100 neurons per layer network. The training took $\sim 23.67$ minutes, and resulted in a mean $\ell_2$ solution error of $1.74\cdot 10^{-3}$.

\paragraph{Allan-Cahn Equation} The Allan-Cahn equation is a form of reaction-diffusion equation, describing the phase separation in multi-component alloy systems \citep{monaco2023training}. In 1D, it is defined on a temporal domain $t \in [0,1]$ and spatial domain $x\in [-1, 1]$ as:
\begin{equation}
\label{eq:allan_cahn}
    \partial_t u(t, x) + \rho u(t,x) (u^2(t,x) - 1) -\nu \partial_{x^2}u(t, x) = 0,
\end{equation}
for $\rho=5$, $\nu=10^{-4}$, and $u(0, x) = x^2\cos(\pi x)$, $u(t, -1) = u(t, 1)$. The solution $\solnet: \mathbb{R}^2 \to \mathbb{R}$ is modeled by an 6-hidden layer, 40 neurons per layer network, and due to its complexity, it is trained using the Causal training scheme from \citet{monaco2023training}. The training process took $\sim 18.56$ minutes, and resulted in a mean $\ell_2$ solution error of $7.9\cdot 10^{-3}$.

\paragraph{Diffusion-Sorption} The diffusion-sorption equation models a diffusion system which is retarded by a sorption process, with one of the most prominent applications being groundwater contaminant transport \citep{takamoto2022pdebench}. In \citep{takamoto2022pdebench}, the equation is defined on a temporal domain $t \in (0, 500]$ and spatial domain $x\in (0, 1)$ as:
\begin{equation}
\label{eq:diff_sorption}
    \partial_t u(t, x) - D/R(u(t,x)) \partial_{x^2}u(t, x) = 0,
\end{equation}
where $D=5\times 10^{-4}$ is the effective diffusion coefficient, and $R(u(t,x))$ is the retardation factor representing the sorption that hinders the diffusion process \citep{takamoto2022pdebench}. In particular, we consider $R(u(t,x)) = 1 + \nicefrac{(1 - \phi)}{(\phi)} \rho_s k n_f u^{n_f-1}(t,x)$, where $\phi=0.29$ is the porosity of the porus medium, $\rho_s=2880$ is the bulk density, $k=3.5\times10^{-4}$ is the Freundlich's parameter, and $n_f=0.874$ is the Freundlich's exponent. The initial and boundary conditions are defined as $u(0, x) = 0$, $u(t, 0) = 0$ and $u(t, 1) = D \partial_x u(t, 1)$. The solution $\solnet: \mathbb{R}^2 \to \mathbb{R}$ is modeled by a 7-hidden layer, 40 neurons per layer network, and we obtain the trained parameters from~\citet{takamoto2022pdebench}. The mean $\ell_2$ solution error is $9.9\cdot 10^{-2}$.

\paragraph{\partialcrown Error Certification}
We obtain certified bounds on the PINN errors for the conditions of Definition \ref{def:correctness} using \partialcrown. We report in Table~\ref{tab:certifying_crown} our verification of the initial conditions \circled{1} using $N_b=5k$ splits, boundary conditions \circled{2} using $N_b=5k$ splits, and the certified bounds on the residual condition \circled{3} using $N_b=2M$ splits. We observe that \partialcrown upper bounds approach the empirical error lower bounds obtained through high-density sampling -- showcasing tightness -- while providing a guarantee on the continuous solution.

\subsection{Empirical relation of \texorpdfstring{$|\resnet|$}{|f\_θ|} and \texorpdfstring{$|\solnet - u|$}{|u\_θ - u|}}
\label{sec:residual-solution-connection}

\begin{figure}[t]
    \centering
    \includegraphics[width=0.9\linewidth]{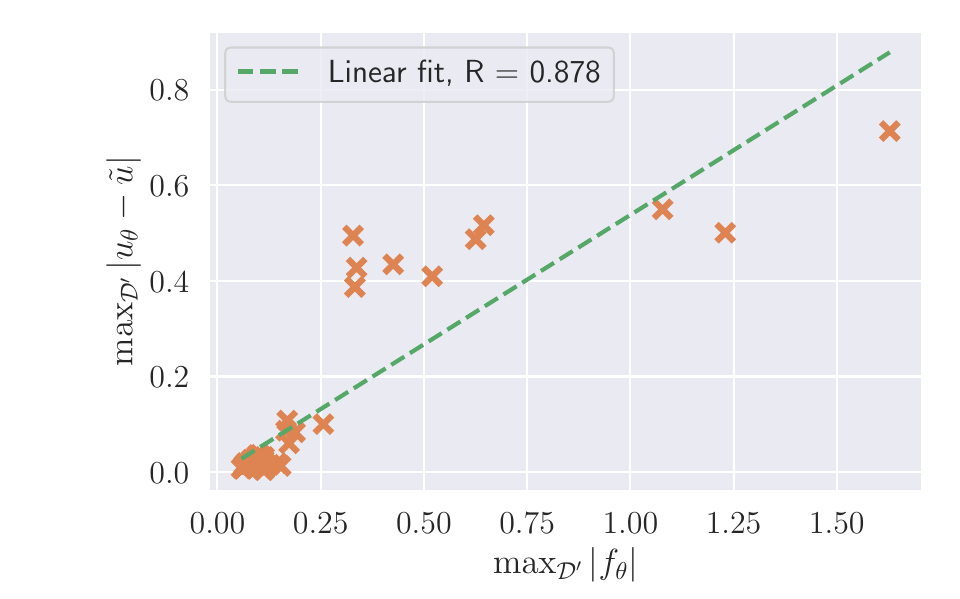}
    \vspace{-0.8em}
    \caption{\textbf{Residual and solution errors}: connection of the maximum residual error ($\max_{\mathcal{S}'} |\resnet|$) and the maximum solution error, $\max_{\mathcal{S}'} |\solnet - \tilde{u}|$, for networks at different epochs of the training process (in orange).}
    \label{fig:solution_residual_error}
\end{figure}

One question that might arise from our certification procedure is the relationship between the PINN residual error, $|\resnet|$, and the solution error with respect to true solution $u$, $|\solnet - u|$, across the domain. By definition, achieving a low $|\resnet|$ implies $\solnet$ is a valid solution for the PDE (assuming boundary and initial conditions also hold), but there is no formal guarantee related to $|\solnet - u|$ within our framework. 

Obtaining a bound on $|\solnet - u|$ is typically a non-trivial task given $u$ might not be unique, and does not necessarily exhibit an analytical solution. And while some recent works perform this analysis for specific PDEs by exploiting their structure and/or smoothness properties ~\citep{mishra2022estimates, ryck2022generic, wang20222}, these methods typically suffer from scalability and bound tightness issues.
As such, we perform an empirical analysis on Burgers' equation using a numerical, finite-difference solver to obtain $\tilde{u}(\x)$ for sampled points $\x$. We randomly sample $10^6$ domain points ($\mathcal{S}'$), and compute the maximum residual error, $\max_{\x \in \mathcal{S}'} |\resnet(\x)|$, and the empirical maximum solution error, $\max_{\x \in \mathcal{S}'} |\solnet(\x) - \tilde{u}(\x)|$, for networks obtained at different epochs of the training process. We report the results in Figure~\ref{fig:solution_residual_error}, with each point corresponding to an instance of a network. As expected, there is a correlation between these errors obtained using a numerical solver, suggesting a similar correlation holds for $|\solnet - u|$.

\begin{table}[t]
    \caption{\textbf{Efficiency of \partialcrown}: comparison of \partialcrown (Ours), Interval Bound Propagation (IBP) and LiRPA upper bounds obtained with greedy input branching (for $N_b$ branches) in Burgers' equation for fixed runtime limits ($150$s, $100$s, or $10^4$s). Lower is better.}
    \label{tab:crown_vs_ibp}
    \center
    \vspace{-1em}
    {
    \footnotesize
    \begin{tabular}{m{0.09\textwidth} >{\centering\arraybackslash} m{0.10\textwidth} >{\centering\arraybackslash} m{0.10\textwidth} >{\centering\arraybackslash} m{0.10\textwidth}}
        & Ours ($N_b$) & IBP ($N_b$) & LiRPA ($N_b$) \\ \specialrule{1pt}{1pt}{1pt}
        
        $|u_\theta(0, x)|^2$ \par{}($150$s) & $2.63\times 10^{-6}$ ($10^4$) & $4.12\times 10^{-3}$ \par{}($10^5$) & $\mathbf{2.23\times 10^{-6}}$ \par{}($10^4$) \\
        $|u_\theta(t, -1)|^2$ \par{}($100$s) & $6.63\times 10^{-7}$ ($10^4$) & $1.23\times 10^{-5}$ \par{}($10^5$) & $\mathbf{6.34\times 10^{-7}}$ \par{}($10^4$) \\
        $|u_\theta(t, 1)|^2$ \par{}($100$s) & $9.39\times 10^{-7}$ ($10^4$) & $5.69\times 10^{-5}$ \par{}($10^5$) & $\mathbf{9.12\times 10^{-7}}$ \par{}($10^4$) \\
        $|f_\theta(x, t)|^2$ \par{}($10^4$s) & $\mathbf{1.30 \times 10^1}$ ($1.3 \times 10^{5}$) & $2.78\times 10^3$ \par{}($5\times 10^6$) & $1.78 \times 10^2$ ($1.9\times 10^4$) \\ \specialrule{0.75pt}{1pt}{1pt}
    \end{tabular}
    }
\end{table}

\subsection{On the efficiency of \texorpdfstring{$\partial$}{∂}-CROWN}
\label{sec:efficiency-experiments}

To the best of our knowledge, \partialcrown is the first framework designed to bound the errors of general PINNs. To highlight its efficiency, we compare is bounding performance to that of Interval Bound Propagation (IBP) \citep{gowal2018effectiveness,mirman2018differentiable} and LiRPA \citep{xu2020automatic} for fixed runtime limits in Burgers' equation. IBP is fast yet yields loose bounds, whereas LiRPA's full back-propagation mechanism makes it slower despite having potentially tighter bounds. The results are presented in Table \ref{tab:crown_vs_ibp}, clearly showcasing how \partialcrown achieves a balance between speed (branching more than LiRPA yet less than IBP) and tightness (outperforming both methods in the tightness of the residual bounds). Note that both \partialcrown and LiRPA are reduced to CROWN in the initial and boundary conditions, and as such the minor differences in bounds in those cases can be attributed to implementation.

\subsection{On the importance of greedy input branching}
\label{sec:importance-branching}

A key factor in the success of \partialcrown in achieving tight bounds of the residual is the greedy input branching procedure from Algorithm~\ref{alg:greedy_input_branching}. To illustrate the fact that a uniform sampling strategy would be significantly more computationally expensive, we plot in Figure~\ref{fig:branching_viz} the relative density of branches (\ie, the percentage of branches per unit of input domain) in the case of Burgers' and \schrodinger's equations. As can be observed, there are clear imbalances at the level of the branching distribution -- with areas away from relative optima of $\solnet$ being relatively under sampled yet achieving tight bounds -- showcasing the efficiency of our strategy.

\section{Discussion and Limitations}

\begin{figure}[t]
    \centering
    \begin{subfigure}[b]{\linewidth}
        \centering
        \includegraphics[width=0.95\textwidth]{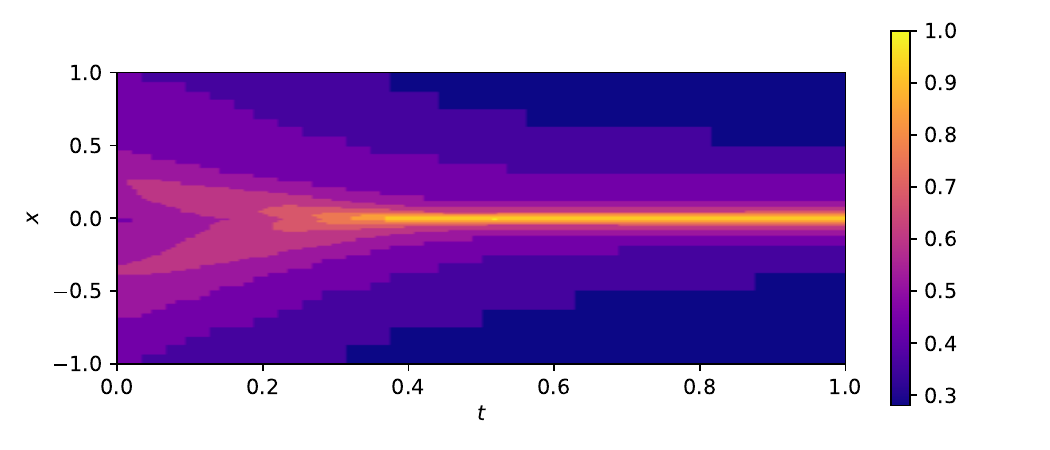}
    \end{subfigure}\\
    \begin{subfigure}[b]{\linewidth}
        \centering
        \vspace{-1em}
        \includegraphics[width=0.95\textwidth]{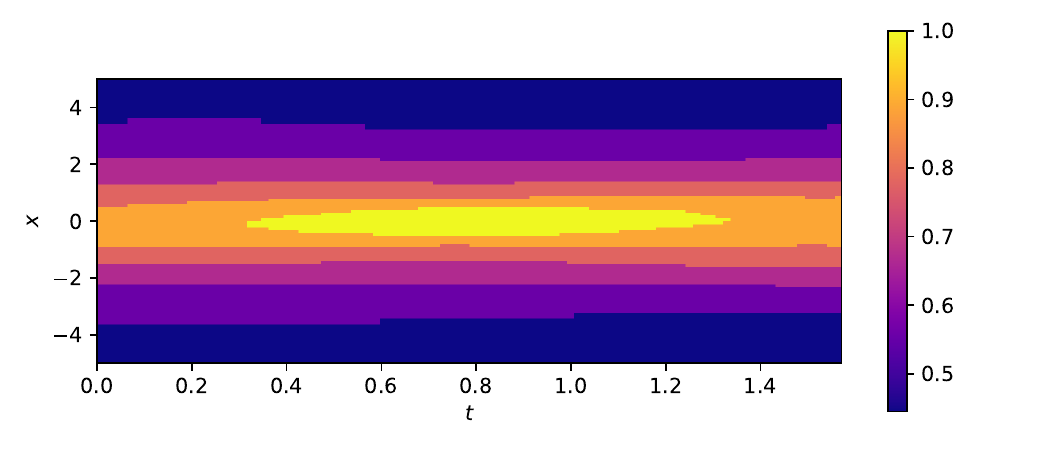}
    \end{subfigure}
    \vspace{-2.3em}
    \caption{\textbf{Branching densities}: relative density of the input branching distribution obtained via Algorithm~\ref{alg:greedy_input_branching} applied to Burgers' (top) and \schrodinger's (bottom) equations.}
    \label{fig:branching_viz}
\end{figure}

We show that \partialcrown is able to obtain tight upper bounds on the correctness conditions established in Definition~\ref{def:correctness}. We highlight in the case of the Diffusion-Sorption equation that relying on empirical lower bound estimates can be misleading -- using $10^4$ MC samples puts the maximum residual error at $1.10\times 10^{-3}$, while $10^6$ samples give an estimate of $21.09$ --, motivating the need for \partialcrown to obtain guarantees across the continuous domain. Note that the absolute values of the residual errors can be seen as a function of the PDE itself, and thus cannot be directly compared across different PINNs. However, in Appendix \ref{app:failure-identification} we effectively show how \partialcrown bounds can be used to detect failure cases in PINN training, highlighting another potential use of our framework on top of certifying well-trained ones.

One of the limitations of our method is unquestionably the running time, particularly for residual verification. This mostly comes down to the high number of branchings required as a result of the relative looseness of the \partialcrown bounds on each individual subdomain. The looseness of the bounds is likely worsened for higher-order PDEs with similar solution networks, since the PINN residual can be viewed in that case as a depth-wise extension of the original network (following Figure \ref{fig:partial-crown-graph}) which, as widely observed in the network verification community, degrades the tightness of the bounds for incomplete verifiers \citep{wang2018mixtrain} (see Appendix \ref{app:higher-order}). A similar argument can be made for higher dimensionality PINNs that \textit{require larger solution networks} (unlike those, e.g., in \citet{jin2021nsfnets, giampaolo2022physics}, which we omit from this work for the reasons in Section \ref{sec:exp_certification}). In these cases it is likely that one will need (i) tighter relaxations of the nonlinearities of the networks, and (ii) more efficient branching methods that allow us to compensate for the tightness loss in deeper networks.

For future work, it would be interesting to further study the connection between PINN \textit{correctness} errors as per Definition \ref{def:correctness} and solution errors, potentially connecting them for specific classes of PDEs by expanding the work of \citet{ryck2022generic}.

\section*{Impact Statement}
This paper presents work whose goal is to advance the field of Machine Learning. There are many potential societal consequences of our work, none which we feel must be specifically highlighted here.

\section*{Acknowledgments}
FE is supported by EPSRC Centre for Doctoral Training in Autonomous Intelligent Machines and Systems [EP/S024050/1] and Five AI Limited. PHST is supported by UKRI grant: Turing AI Fellowship EP/W002981/1, and by the Royal Academy of Engineering under the Research Chair and Senior Research Fellowships scheme.

\bibliography{references}
\bibliographystyle{icml2024}

\appendix
\onecolumn

\begin{figure*}[t]
    \centering
    \begin{subfigure}[b]{0.24\textwidth}
        \centering
        \includegraphics[width=\textwidth]{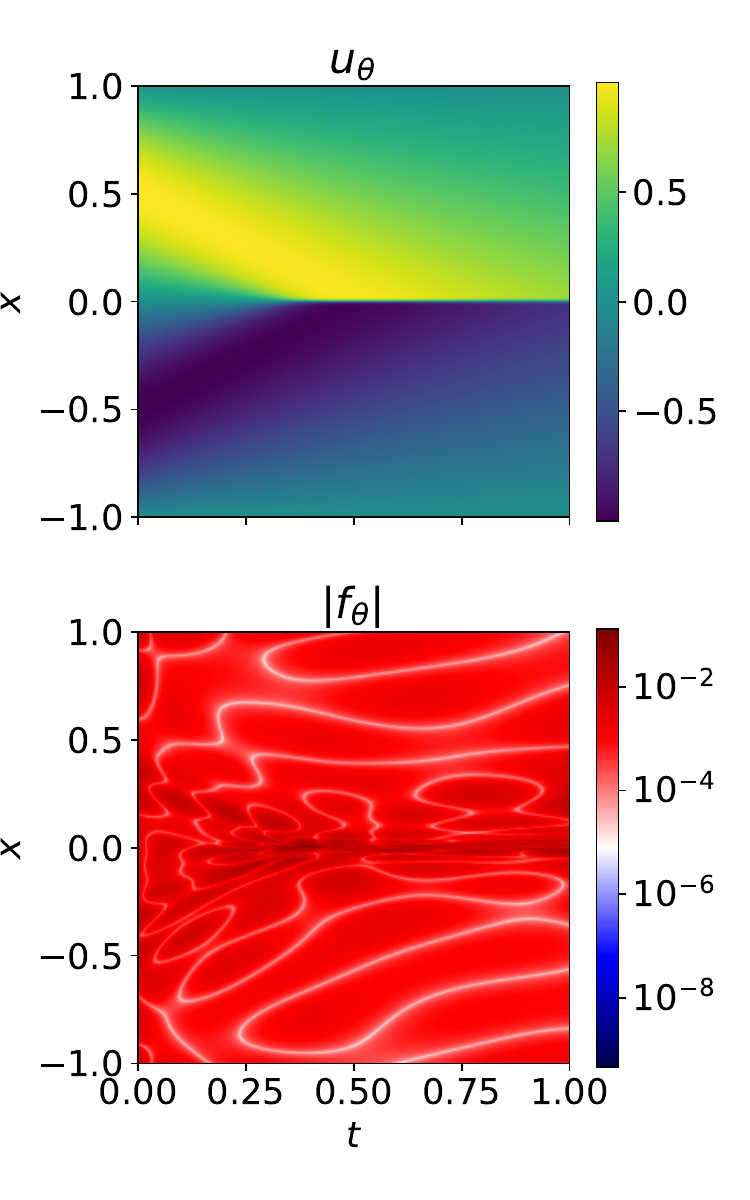}
        \vspace{-2em}
        \caption{}
        \label{fig:surfaces_burgers}
    \end{subfigure}
    \begin{subfigure}[b]{0.24\textwidth}
        \centering
        \includegraphics[width=\textwidth]{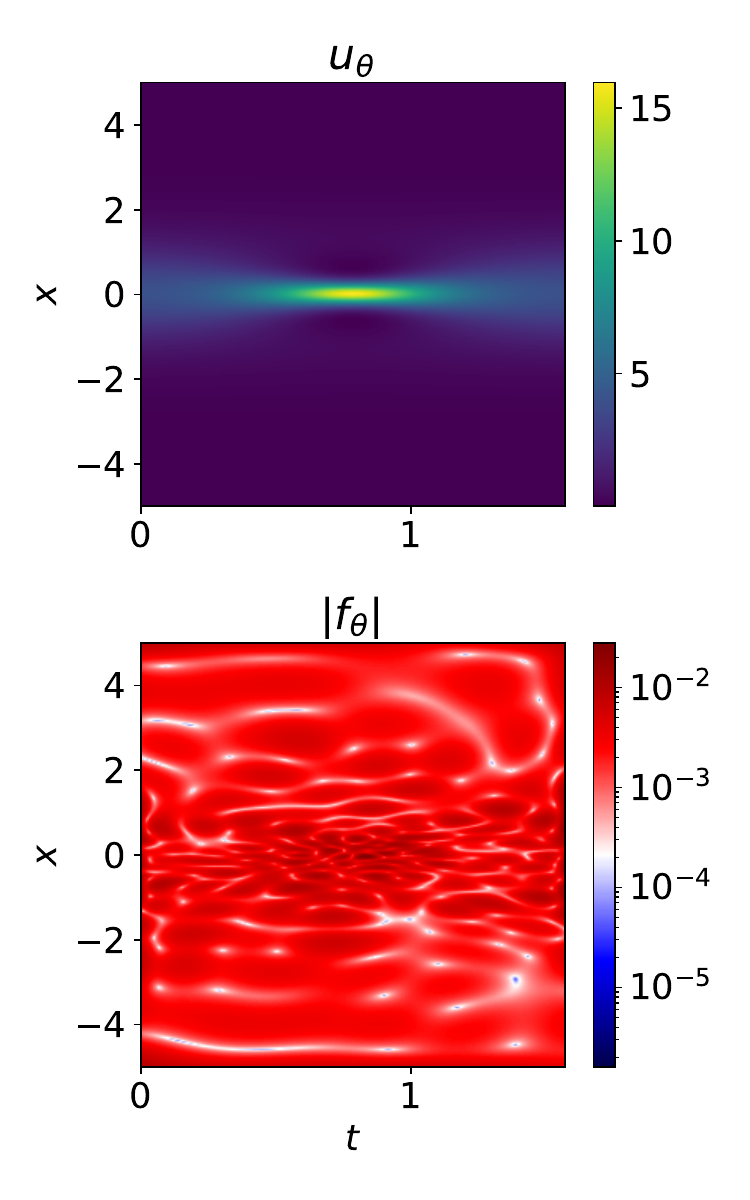}
        \vspace{-2em}
        \caption{}
        \label{fig:surfaces_schrodingers}
    \end{subfigure}
    \begin{subfigure}[b]{0.24\textwidth}
        \centering
        \includegraphics[width=\textwidth]{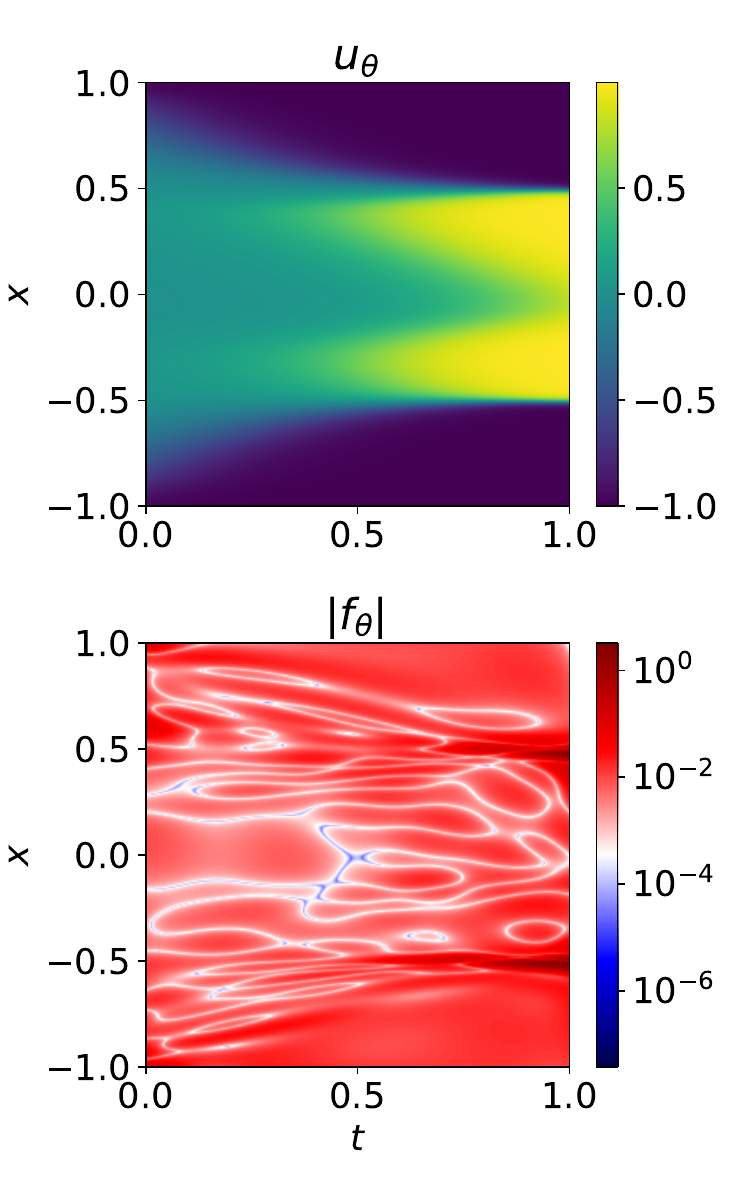}
        \vspace{-2em}
        \caption{}
        \label{fig:surfaces_allan_cahn}
    \end{subfigure}
    \begin{subfigure}[b]{0.24\textwidth}
        \centering
        \includegraphics[width=\textwidth]{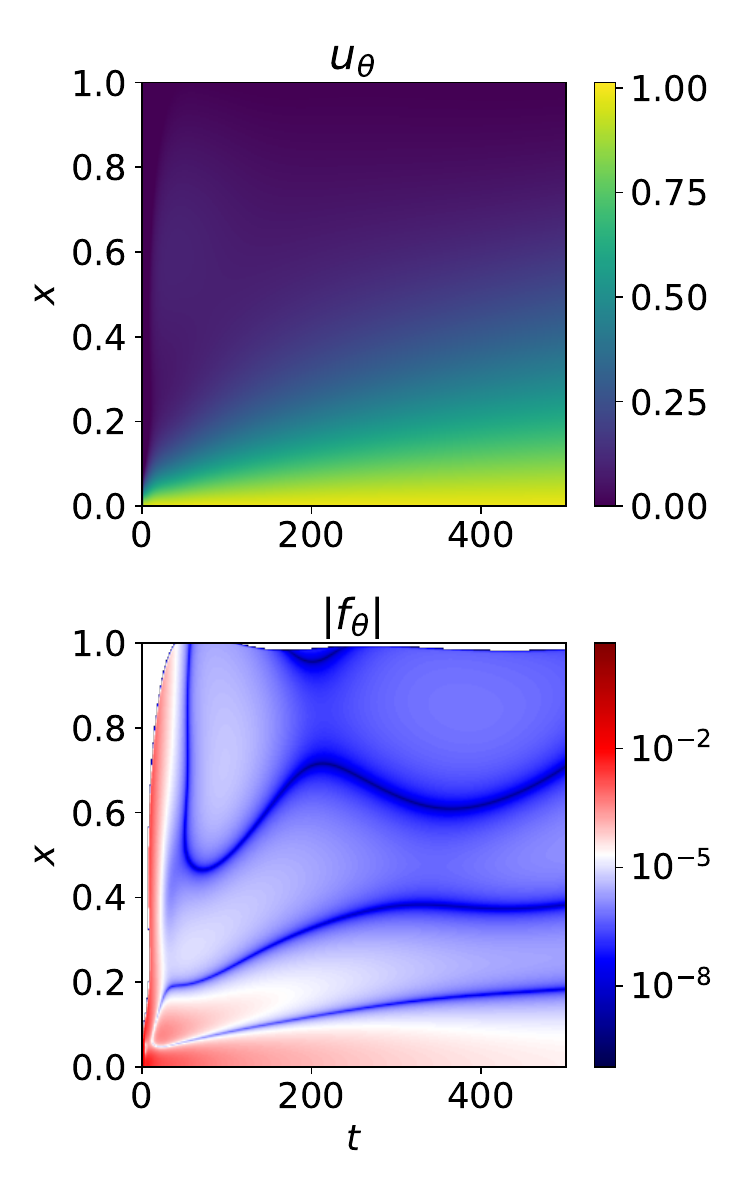}
        \vspace{-2em}
        \caption{}
        \label{fig:surfaces_diff_sorp}
    \end{subfigure}
    \vspace{-0.6em}
    \caption{\textbf{Certifying with \partialcrown}: visualization of the time evolution of $\solnet$, and the residual errors as a function of the spatial temporal domain (log-scale), $|\resnet|$, for \textbf{(a)} Burgers' equation~\citep{raissi2019physics}, \textbf{(b)} \schrodinger's equation~\citep{raissi2019physics}, \textbf{(c)} Allan-Cahn's equation~\citep{monaco2023training}, and \textbf{(d)} the Diffusion-Sorption equation~\citep{takamoto2022pdebench}.}
    \label{fig:all_residuals}
    \vspace{-1em}
\end{figure*}

\section{Reducing empirical and certified errors through Physics-Informed Adversarial Training}
\label{app:adversarial-training-and-analysis}

The goal of reducing the solution errors obtained by PINNs has been the research focus of several previous works \cite{kim2021dpm, krishnapriyan2021characterizing, shekarpaz2022piat}. To observe the effects of one of these different training schemes on the verified correctness certification of PINNs, we consider Physics-informed Adversarial Training (PIAT)~\citep{shekarpaz2022piat}.
The procedure consists in replacing the residual loss term from~\citet{raissi2019physics} with an adversarial version inspired by~\citet{madry2017towards}. While this procedure leads to improvements in the example PINNs from \citet{shekarpaz2022piat} and using our own implementation in Burgers' equation, we were unable to stably train \schrodinger's equation using PIAT. Since \schrodinger's equation is not considered in \citet{shekarpaz2022piat}, we only show PIAT results for Burgers' equation.

We solve the inner optimization problem using 5 PGD steps~\citep{madry2017towards}, and for $\epsilon=0.05$ and a step size of $1.25\epsilon$. To improve convergence, we warm start PIAT training using a standard training solution after 6,000 L-BFGS iterations. The results in Table~\ref{tab:burgers_tanh_piat} show that as expected PIAT improves both empirical and certified residual bounds.

\begin{table}[h]
    \caption{\textbf{PIAT on Burgers' equation}: Monte Carlo sampled maximum values ($10^6$ samples in 0.21s) and upper bounds computed using \partialcrown with $N_b$ branchings for \protect\circled{1} initial conditions ($t = 0$, $x \in \mathcal{D}$, $N_b=5k$), \protect\circled{2} boundary conditions ($t \in [0, T]$, $x = -1 \vee x = 1$, $N_b=5k$), and \protect\circled{3} residual norm ($t\in [0, T]$, $x \in \mathcal{D}$, $N_b=125k$), for a PINN trained using PIAT from~\citet{shekarpaz2022piat}.}
    \label{tab:burgers_tanh_piat}
    \center
    {
    \footnotesize
    \begin{tabular}{p{0.22\textwidth}p{0.01\textwidth}p{0.24\textwidth}P{0.11\textwidth}P{0.26\textwidth}}
        & & & MC - $\max$ & \partialcrown - $u_b$ (time [s])\\ \specialrule{1pt}{1pt}{1pt}

        \multirow{4}{*}{\shortstack[l]{\textbf{PIAT Burgers} \\ \citep{shekarpaz2022piat}}} & \circled{1} & $|\solnet(0, x) - u_0(x)|^2$ & $7.40\cdot10^{-6}$ & $8.18\cdot10^{-6}$ ($90.9$) \\\cline{2-5}
        & \multirow{2}{*}{\circled{2}} & $|\solnet(t, -1)|^2$ & $2.31\cdot10^{-7}$ & $3.32\cdot10^{-7}$ ($49.4$) \\
        & & $|\solnet(t, 1)|^2$ & $8.41\cdot10^{-8}$ & $1.39\cdot10^{-7}$ ($48.5$) \\\cline{2-5}
        & \circled{3} & $|\resnet(\x)|^2$ & $3.60\cdot10^{-3}$ & $2.39\cdot10^{-2}$ ($2.8\times 10^5$) \\ \specialrule{0.75pt}{1pt}{1pt}
    \end{tabular}
    }
\end{table}

\begin{wrapfigure}[15]{R}{0.5\textwidth}
    \vspace{-2em}
    \begin{minipage}{\linewidth}
        \begin{figure}[H]
            \centering
            \includegraphics[width=\linewidth]{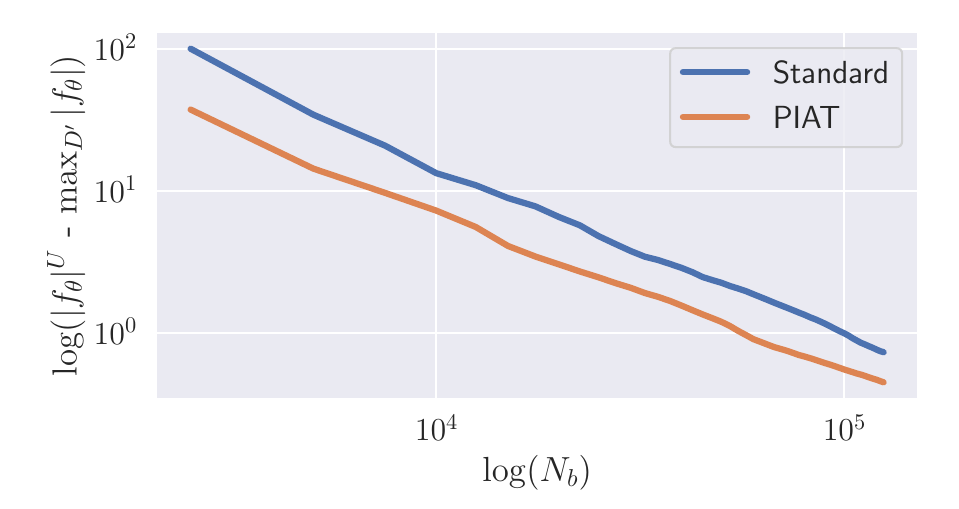}
            \caption{\textbf{Certification Convergence}: log-log plot of the relative convergence of \partialcrown certification for a standard trained PINN (in blue) and PIAT (in orange).}
            \label{fig:relative_convergence}
        \end{figure}
    \end{minipage}
\end{wrapfigure}

\paragraph{Certification convergence in PIAT vs. standard training} The regularization provided by adversarial training often leads to verification algorithms converging faster to tighter lower and upper bounds. We investigate whether this is the case with \partialcrown's greedy branching strategy by comparing the \textit{relative convergence} (\ie, the deviation between the upper bound and the empirical maximum, $|\resnet|^U - \max_{D'} |\resnet|$) for the first $125k$ splits of PINNs trained in the standard and PIAT cases. The results presented in Figure~\ref{fig:relative_convergence} show that adversarial training leads to quicker convergence, requiring a lower number of branches to reach the same error when compared to standard. This suggests that our method, while already efficient, would benefit from smarter training strategies that lead to lower residual errors.

\section{\texorpdfstring{$\partial$}{∂}-CROWN for Failure Identification}
\label{app:failure-identification}

In Section \ref{sec:residual-solution-connection} we establish the empirical correlation between residual and solution errors for PINNs at different training stages (Figure \ref{fig:solution_residual_error}). While comparing PINN errors for different PDEs is not easy due to residual scaling factors, note from Table \ref{tab:certifying_crown} that the errors obtained for Burgers’ and Schrödinger’s equations are orders of magnitude lower than the ones for the Allen-Cahn and Diffusion-Sorption equations. Even with different residual tolerances, this would suggest the maximum solution error of the latter, harder to train PINNs should be higher. 

Table \ref{tab:failure_identification} presents the residual bounds obtained using \partialcrown as well as the maximum solution error with respect to a numerical solver for each of the four PINNs studied, which empirically reinforces that correlation. E.g., Burgers' equation has a maximum solution error of $3.78\times 10^{-3}$, which is significantly lower than the trained Allen-Cahn PINN at $0.86$, as expected from the residual bounds of $1.80\times 10^{-2}$ and $10.76$, respectively. This contextualizes the results of Table \ref{tab:certifying_crown} and showcases our framework can identify weaker models. 

\begin{table}[t]
    \caption{\textbf{Failure identification using residual bounds}: empirical analysis of the connection between the residual bounds obtained by \partialcrown and the maximum solution error computed with respect to a numerical solver, $u$, over a sampled dataset $\mathcal{D}'$. The range of the solution values over the samples in $\mathcal{D}'$ are included for ease of comparison.}
    \label{tab:failure_identification}
    \center
    {
    \footnotesize
    \begin{tabular}{p{0.2\textwidth}P{0.16\textwidth}P{0.25\textwidth}P{0.25\textwidth}}
        & Residual \partialcrown $u_b$ & Max solution error ($\max_{\mathcal{D}'} |u_\theta - u|$) & Solution range ($\min/\max_{\mathcal{D}'} u_\theta$) \\ \specialrule{1pt}{1pt}{1pt}
        
        Burgers & $1.80\times 10^{-2}$ & $3.78\times 10^{-3}$ & $[-1, 1]$\\
        \schrodinger & $7.67\times 10^{-4}$ & $7.05\times 10^{-5}$ & $[1.82 \times 10^{-4}, 15.98]$ \\
        Allen-Cahn & $10.76$ & $0.86$ & $[-1, 1]$ \\
        Diffusion-Sorption & $21.09$ & $0.99$ & $[0, 1]$ \\ \specialrule{0.75pt}{1pt}{1pt}
    \end{tabular}
    }
\end{table}

\section{Ablation on \texorpdfstring{$N_b$}{N\_b}}
\label{app:ablation_Nb}

We use $N_b = 2M$ for all the PINNs evaluated in this paper. A high number of branchings is required to obtain the tight bounds presented in Table \ref{tab:certifying_crown}. To justify that need, we have added plots of the variation of the obtained residual bound for Burgers’ and \schrodinger’ equations in Figure \ref{fig:ablation_n_b}. Generally for both these PINNs we only get closer than one order of magnitude from the empirical estimates (considering the empirical MC sampled errors from Table \ref{tab:certifying_crown}) by using around $2M$ branches.

\begin{figure}[t]
    \centering
    \begin{subfigure}[b]{0.49\textwidth}
        \centering
        \includegraphics[width=\textwidth]{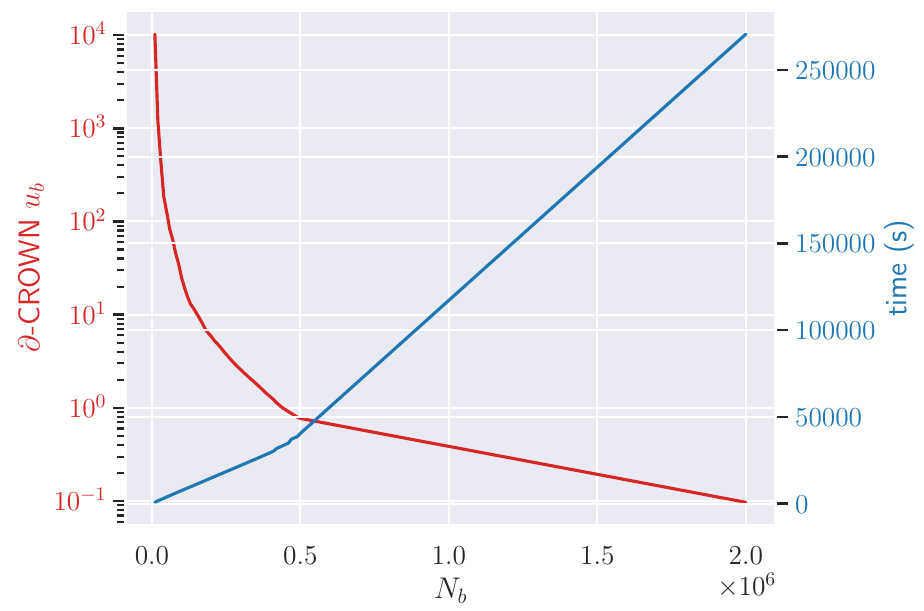}
        \caption{}
        \label{fig:ablation_n_b_burgers}
    \end{subfigure}
    \begin{subfigure}[b]{0.45\textwidth}
        \centering
        \includegraphics[width=\textwidth]{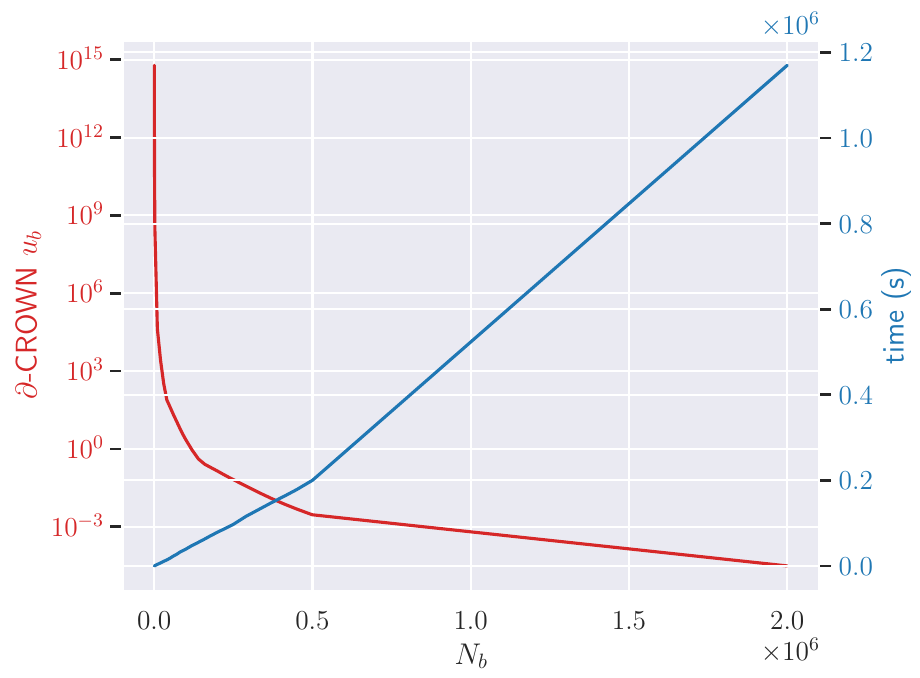}
        \caption{}
        \label{fig:ablation_n_b_schrodingers}
    \end{subfigure}
    \caption{\textbf{Ablation on $N_b$}: comparison of the residual error bounds ($|f_\theta|^2$) and runtime performance of our framework, \partialcrown on (a) Burgers' equation and (b) \schrodinger's equation.}
    \label{fig:ablation_n_b}
\end{figure}

\section{Proofs of partial derivative computations}
\label{app:proofs}

\subsection{Proof of Lemma \texorpdfstring{\ref{lem:first_derivative}}{1}: computing \texorpdfstring{$\partial_{\x_i} \solnet$}{∂\_\{x\_i\} u\_θ}}
\label{app:proof_first_derivative}

Let us now derive $\partial_{\x_i} \solnet(\x)$ for a given $i \in \{1,...,n_0\}$. Starting backwards from the last layer and applying the chain rule we obtain:
$$
    \partial_{\x_i} \solnet(\x) = \frac{\partial \preout^{(L)}}{\partial \out^{(L-1)}} \cdot \frac{\partial \out^{(L-1)}}{\partial \out^{(L-2)}} \cdot ... \cdot \frac{\partial \out^{(1)}}{\partial \x}\cdot \frac{\partial \x}{\partial x_i}
$$

Given that $\partial_{\x_i} x = \mathbf{e}_i$ and $\frac{\partial \preout^{(L)}}{\partial \out^{(L-1)}} = \W^{(L)}$, all that's left to compute to obtain the full expression is $\frac{\partial \out^{(k)}}{\partial \out^{(k-1)}}$, $k\in\{L-1,...,1\}$. Note that, for simplicity of the expressions, $\out^{(0)} = \mathbf{x}$. For every element $j\in\{1,...,d_k\}$ of $\out^{(k)}$ denoted by $\out^{(k)}_{j}$, we have:
$$
    \frac{\partial \out^{(k)}_{j}}{\partial \out^{(k-1)}} = \sigma'\left(\W^{(k)}_{[j, :]} \out^{(k-1)} + \bias^{(k)}_{j}\right) \W^{(k)}_{[j, :]}
$$
where $\W^{(k)}_{j,:}$ denotes the $j$-th row of $\W^{(k)}$, and $\bias^{(k)}_{j}$ the $j$-th element of $\bias$. Thus, the final expression can be obtained by stacking the columns of the previous expression to obtain the full Jacobian:
$$
    \frac{\partial \out^{(k)}}{\partial \out^{(k-1)}} = \text{diag}\left[\sigma'\left(\W^{(k)}\out^{(k-1)} + \bias^{(k)}\right)\right]\cdot \W{^{(k)}}
$$

This concludes the proof.

\subsection{Proof of Lemma \texorpdfstring{\ref{lem:second_derivative}}{2}: computing  \texorpdfstring{$\partial_{\x_i^2} \solnet$}{∂\_\{x\^2\_i\} u\_θ}}
\label{app:proof_derivation_second_derivative}

Given the result obtained in Appendix \ref{app:proof_first_derivative}, let us now derive $\partial_{\x_i^2} \solnet(\x)$ for a given $i \in \{1,...,d_0\}$.
Starting backwards from the last layer of $\partial_{\x_i} \solnet$ and applying the chain rule we obtain:
\begin{align*}
    \partial_{\x_i^2} \solnet & = \frac{\partial}{\partial x_i}\left(\frac{\partial \preout^{(L)}}{\partial \out^{(L-1)}} \cdot \frac{\partial \out^{(L-1)}}{\partial \out^{(L-2)}} \cdot ... \cdot \frac{\partial \out^{(1)}}{\partial \x}\cdot \frac{\partial \x}{\partial x_i}\right) = \W^{(L)} \partial_{\x_i^2} \out^{(L-1)}
\end{align*}

Now the same can be applied to $\partial_{\x_i^2} \out^{(L-1)}$, and in general to $\partial_{\x_i^2} \out^{(k)}$ to obtain:
\begin{align*}
    \partial_{\x_i^2} \out^{(k)} & = \frac{\partial}{\partial x_i}\left(\frac{\partial \out^{(k)}}{\partial \out^{(k-1)}} \partial_{\x_i} \out^{(k-1)}\right) = \frac{\partial^2 \out^{(k)}}{\partial x_i\partial \out^{(k-1)}} \partial_{\x_i} \out^{(k-1)} + \frac{\partial \out^{(k)}}{\partial \out^{(k-1)}} \partial_{\x_i^2} \out^{(k-1)},
\end{align*}
forming a recursion which can be taken until the first layer of $\partial_{\x_i} \solnet$, \ie,:
\begin{align*}
    \partial_{\x_i^2} \out^{(1)} & = \frac{\partial}{\partial x_i}\left(\frac{\partial \out^{(1)}}{\partial \x} \cdot \mathbf{e}_i\right) = \frac{\partial^2 \out^{(1)}}{\partial x_i \partial \x}\cdot \mathbf{e}_i.
\end{align*}
With the computation of $\partial_{\x_i} \solnet$, both $\partial_{\x_i} \out^{(k-1)}$ and $\frac{\partial \out^{(k)}}{\partial \out^{(k-1)}}$ are known. As such, the only missing pieces in the general recursion is the computation of $\frac{\partial^2 \out^{(k)}}{\partial x_i\partial \out^{(k-1)}}$. Recall from the previous section that $\frac{\partial \out^{(k)}}{\partial \out^{(k-1)}} = \text{diag}\left[\sigma'\left(\W^{(k)}\out^{(k-1)} + \bias^{(k)}\right)\right] \W{^{(k)}}$. As such:
\begin{align*}
    \frac{\partial^2 \out^{(k)}}{\partial x_i\partial \out^{(k-1)}} & = \frac{\partial}{\partial x_i}\left(\text{diag}\left[\sigma'\left(\W^{(k)}\out^{(k-1)} + \bias^{(k)}\right)\right] \W{^{(k)}}\right).
\end{align*}
Following the element-wise reasoning from above, we have that:
\begin{align*}
    \frac{\partial^2 \out^{(k)}_{j}}{\partial x_i\partial \out^{(k-1)}} & = \sigma''\left(\W^{(k)}_{j,:}\out^{(k-1)} + \bias^{(k)}_{j}\right) \frac{\partial}{\partial x_i}\left(\W^{(k)}_{j,:} \out^{(k-1)} + \bias^{(k)}_{j}\right) \W^{(k)}_{j,:} \\
    & = \sigma''\left(\W^{(k)}_{j,:}\out^{(k-1)} + \bias^{(k)}_{j}\right) \left(\W^{(k)}_{j,:} \frac{\partial \out^{(k-1)}}{\partial x_i}\right) \W^{(k)}_{j,:} 
\end{align*}

Stacking as in the previous case, we obtain:
\begin{align*}
    \frac{\partial^2 \out^{(k)}}{\partial x_i\partial \out^{(k-1)}} & = \text{diag}\left[\sigma''\left(\W^{(k)}\out^{(k-1)} + \bias^{(k)}\right) \left(\W^{(k)} \partial_{\x_i} \out^{(k-1)}\right)\right] \W^{(k)},
\end{align*}
completing the derivation of $\partial_{\x_i^2} \solnet(\x)$.

\subsection{Theorem \ref{thm:first_derivative_bounds}: Formal Statement and Proof}
\label{app:proof_bounding_first_derivative}

\begin{reptheorem}{thm:first_derivative_bounds}[\partialcrown: linear lower and upper bounding $\partial_{\x_i} \solnet$]
For every $j\in\{1,\dots,d_L\}$ there exist two functions $\partial_{\x_i} u^{U}_{\theta, j}$ and $\partial_{\x_i} u^{L}_{\theta, j}$ such that, $\forall \x\in \mathcal{C}$ it holds that $\partial_{\x_i} u^{L}_{\theta, j} \leq \partial_{\x_i} u_{\theta, j} \leq \partial_{\x_i} u^{U}_{\theta, j}$, with:
\begin{align*}
    \partial_{\x_i} u^{U}_{\theta, j} = \phi^{(1), U}_{0,j,i} + \sum_{r=1}^{d_0} \phi^{(1), U}_{1, j,r} \x + \phi^{(1), U}_{2,j,r}\\
    \partial_{\x_i} u^{L}_{\theta, j} = \phi^{(1), L}_{0,j,i} + \sum_{r=1}^{d_0} \phi^{(1), L}_{1, j,r} \x + \phi^{(1), L}_{2,j,r}
\end{align*}
where for $p\in\{0,1,2\}$, $\phi^{(1), U}_{p, j, r}$ and $\phi^{(1), L}_{p, j, r}$ are functions of $\W^{(k)}$, $\preout^{(k), L}$, $\preout^{(k), U}$, $\vec{A}^{(k), L}$, $\vec{A}^{(k), U}$ $\vec{a}^{(k), L}$, and $\vec{a}^{(k), U}$, and can be computed using a recursive closed-form expression in $\mathcal{O}(L)$ time.
\end{reptheorem}

\textit{Proof:} Assume that through the computation of the previous bounds on $\solnet$, the pre-activation layer outputs of $\solnet$, $\preout^{(k)}$, are lower and upper bounded by linear functions defined as $\vec{A}^{(k), L} \x + \vec{a}^{(k), L} \leq \preout^{(k)} \leq \vec{A}^{(k), U} \x + \vec{a}^{(k), U}$ and $\preout^{(k), L} \leq \preout^{(k)} \leq \preout^{(k), U}$ for $x\in \mathcal{C}$.

Take the upper and lower bound functions for $\partial_{\x_i} \solnet$ as $\partial_{\x_i} \solnet^U$ and $\partial_{\x_i} \solnet^L$, respectively, and the upper and lower bound functions for $\partial_{\x_i} \out^{(k)}$ as $\partial_{\x_i} \out^{(k), U}$ and $\partial_{\x_i} \out^{(k), L}$, respectively. For the sake of simplicity of notation, we define $\vec{B}^{(k), +} = \mathbb{I}\left(\vec{B}^{(k)} \geq 0\right) \odot \vec{B}^{(k)}$ and $\vec{B}^{(k), -} = \mathbb{I}\left(\vec{B}^{(k)} < 0\right) \odot \vec{B}^{(k)}$.

Working backwards from $\partial_{\x_i} \solnet$, we apply the same idea from CROWN \citep{zhang2018efficient}:
\begin{align}
\label{eq:partial_xi_solnet_not_substituted}
\begin{split}
\partial_{\x_i} \solnet^U & = \W^{(L),+} \partial_{\x_i} \out^{(L-1), U} + \W^{(L),-} \partial_{\x_i} \out^{(L-1), L}\\
\partial_{\x_i} \solnet^L & = \W^{(L),+} \partial_{\x_i} \out^{(L-1), L} + \W^{(L),-} \partial_{\x_i} \out^{(L-1), U}
\end{split}
\end{align}
We continue to apply this backwards propagation to $\partial_{\x_i} \out^{(L-1)}$ to obtain $\partial_{\x_i} \out^{(L-1), U}$ and $\partial_{\x_i} \out^{(L-1), L}$. Recall that $\partial_{\x_i} \out^{(k)} = \partial_{\out^{(k-1)}} \out^{(k)} \partial_{\x_i} \out^{(k-1)}$, that is, for $j \in \{1,\dots,d_k\}$ we have $\partial_{\x_i} \out^{(k)}_{j} = \partial_{\out^{(k-1)}} \out^{(k)}_{j,:} \partial_{\x_i} \out^{(k-1)} = \sum_{n=1}^{d_{k-1}} \partial_{\out^{(k-1)}} \out^{(k)}_{j,n} \partial_{\x_i} \out^{(k-1)}_{n}$.

We resolve the bilinear dependencies of each $\partial_{\x_i} \out^{(k)}_{j}$ by relaxing it using a convex combination of the upper and lower bounds obtained by the McCormick envelopes of the product. Assuming that $\partial_{\out^{(k-1)}} \out^{(k), L}_{j,n} \leq  \partial_{\out^{(k-1)}} \out^{(k)}_{j,n} \leq \partial_{\out^{(k-1)}} \out^{(k), U}_{j,n}$ and $\partial_{\x_i} \out^{(k-1), L}_{n} \leq \partial_{\x_i} \out^{(k-1)}_{n} \leq \partial_{\x_i} \out^{(k-1), U}_{n}$, we have that:
\begin{align}
\label{eq:partial_xi_hat_z_non_substituted}
\begin{split}
    \partial_{\x_i} \out^{(k)}_{j} \leq \partial_{\x_i} \out^{(k),U}_{j} =\, \sum_{n=1}^{d_{k-1}} \alpha^{(k)}_{0, j,n} \partial_{\x_i} \out^{(k-1)}_{n} + \alpha^{(k)}_{1, j,n} \partial_{\out^{(k-1)}} \out^{(k)}_{j,n} + \alpha^{(k)}_{2,j,n}\\
    \partial_{\x_i} \out^{(k)}_{j} \geq \partial_{\x_i} \out^{(k),L}_{j} =\, \sum_{n=1}^{d_{k-1}} \beta^{(k)}_{0, j,n} \partial_{\x_i} \out^{(k-1)}_{n} + \beta^{(k)}_{1, j,n} \partial_{\out^{(k-1)}} \out^{(k)}_{j,n} + \beta^{(k)}_{2, j,n},
\end{split}
\end{align}
for:
\begin{align*}
    \alpha^{(k)}_{0, j,n} & = \eta^{(k)}_{j,n} \partial_{\out^{(k-1)}} \out^{(k), U}_{j,n} + \left(1-\eta^{(k)}_{j,n}\right) \partial_{\out^{(k-1)}} \out^{(k), L}_{j,n}\\
    \alpha^{(k)}_{1, j,n} & = \eta^{(k)}_{j,n} \partial_{\x_i} \out^{(k-1), L}_{n} + \left(1-\eta^{(k)}_{j,n}\right) \partial_{\x_i} \out^{(k-1), U}_{n}\\
    \alpha^{(k)}_{2, j,n} & = -\eta^{(k)}_{j,n} \partial_{\out^{(k-1)}} \out^{(k), U}_{j,n} \partial_{\x_i} \out^{(k-1), L}_{n} - \left(1-\eta^{(k)}_{j,n}\right) \partial_{\out^{(k-1)}} \out^{(k), L}_{j,n} \partial_{\x_i} \out^{(k-1), U}_{n}\\
    \beta^{(k)}_{0, j,n} & = \zeta^{(k)}_{j,n} \partial_{\out^{(k-1)}} \out^{(k), L}_{j,n} + \left(1-\zeta^{(k)}_{j,n}\right) \partial_{\out^{(k-1)}} \out^{(k), U}_{j,n}\\
    \beta^{(k)}_{1, j,n} & = \zeta^{(k)}_{j,n} \partial_{\x_i} \out^{(k-1), L}_{n} + \left(1-\zeta^{(k)}_{j,n}\right) \partial_{\x_i} \out^{(k-1), U}_{n}\\
    \beta^{(k)}_{2, j,n} & = -\zeta^{(k)}_{j,n} \partial_{\out^{(k-1)}} \out^{(k), L}_{j,n} \partial_{\x_i} \out^{(k-1), L}_{n}  - \left(1-\zeta^{(k)}_{j,n}\right) \partial_{\out^{(k-1)}} \out^{(k), U}_{j,n} \partial_{\x_i} \out^{(k-1), U}_{n},
\end{align*}
where $\eta^{(k)}_{j,n}$ and $\zeta^{(k)}_{j,n}$ are convex coefficients that can be set as hyperparameters, or optimized for as in $\alpha$-CROWN \citep{xu2020fast}.

To continue the backward propagation, we now need to bound the components of $\partial_{\out^{(k-1)}} \out^{(k)}$. Recall from Lemma \ref{lem:first_derivative} that $\partial_{\out^{(k-1)}} \out^{(k)} = \text{diag}\left[\sigma'\left(\preout^{(k-1)}\right)\right] \W{^{(k)}}$, and $\partial_{\out^{(k-1)}} \out^{(k)}_{j,:} = \sigma'\left(\preout^{(k-1)}_{j}\right) \W^{(k)}_{j,:}$ for $j\in\{1,\dots,d_k\}$.

Since $\preout^{(k), L}_{j} \leq \preout^{(k)}_{j} \leq \preout^{(k), U}_{j}$, we can obtain a linear upper and lower bound relaxation for $\sigma'\left(\preout^{(k)}_{j}\right)$, such that $\gamma^{(k),L}_{j}\left(\preout^{(k)}_{j} + \delta^{(k),L}_{j}\right) \leq \sigma'\left(\preout^{(k)}_{j}\right)\leq \gamma^{(k),U}_{j}\left(\preout^{(k)}_{j} + \delta^{(k),U}_{j}\right)$. With this, we can proceed to bound $\partial_{\out^{(k-1)}} \out^{(k)}_{j,:}$ as:
\begin{align}
\label{eq:partial_hat_z_k_1_hat_z_non_substituted}
\begin{split}
\partial_{\out^{(k-1)}} \out^{(k)}_{j,:} & \leq \underbrace{\left(\gamma^{(k),U}_{j} \W^{(k),+}_{j,:} + \gamma^{(k),L}_{j} \W^{(k),-}_{j,:}\right)}_{\iota^{(k)}_{0, j,:}}  \preout^{(k)}_{j} + \underbrace{\left(\gamma^{(k),U}_{j} \delta^{(k),U}_{j} \W^{(k),+}_{j,:} + \gamma^{(k),L}_{j} \delta^{(k),L}_{j} \W^{(k),-}_{j,:}\right)}_{\iota^{(k)}_{1, j,:}}\\
\partial_{\out^{(k-1)}} \out^{(k)}_{j,:} & \geq \underbrace{\left(\gamma^{(k),L}_{j} \W^{(k),+}_{j,:} + \gamma^{(k),U}_{j} \W^{(k),-}_{j,:}\right)}_{\lambda^{(k)}_{0, j,:}} \preout^{(k)}_{j} + \underbrace{\left(\gamma^{(k),L}_{j} \delta^{(k),L}_{j} \W^{(k),+}_{j,:} + \gamma^{(k),U}_{j} \delta^{(k),U}_{j} \W^{(k),-}_{j,:}\right)}_{\lambda^{(k)}_{1, j,:}}
\end{split}
\end{align}
At this point, one could continue the back-substitution process using the bounds from CROWN \citep{zhang2018efficient}. However, for the sake of efficiency, we use instead the pre-computed inequalities from propagating bounds through $\solnet$: $\vec{A}^{(k), U} \x + \vec{a}^{(k), U} \leq \preout^{(k)} \leq \vec{A}^{(k), L} \x + \vec{a}^{(k), L}$. Substituting this in Equation \ref{eq:partial_hat_z_k_1_hat_z_non_substituted}, we obtain:
\begin{align}
\label{eq:partial_hat_z_k_1_hat_z_substituted}
\begin{split}
\partial_{\out^{(k-1)}} \out^{(k), U}_{j,:} & = \underbrace{\left(\iota^{(k),+}_{0, j,:} \vec{A}^{(k), U}_{j,:} + \iota^{(k),-}_{0, j,:} \vec{A}^{(k), L}_{j,:}\right)}_{\iota^{(k)}_{2, j,:}} \x + \underbrace{\iota^{(k),+}_{0, j,:} \vec{a}^{(k), U}_{j} + \iota^{(k),-}_{0, j,:} \vec{a}^{(k), L}_{j} + \iota^{(k)}_{1, j,:}}_{\iota^{(k)}_{3, j,:}}\\
\partial_{\out^{(k-1)}} \out^{(k), L}_{j,:} & = \underbrace{\left(\lambda^{(k),+}_{0, j,:} \vec{A}^{(k), L}_{j,:} + \lambda^{(k),-}_{0, j,:} \vec{A}^{(k), U}_{j,:}\right)}_{\lambda^{(k)}_{2, j,:}} \x + \underbrace{\lambda^{(k),+}_{0, j,:} \vec{a}^{(k), L}_{j} + \lambda^{(k),-}_{0, j,:} \vec{a}^{(k), U}_{j} + \lambda^{(k)}_{1, j,:}}_{\lambda^{(k)}_{3, j,:}}
\end{split}
\end{align}
In practice, we can use Equation \ref{eq:partial_hat_z_k_1_hat_z_substituted} to compute the required $\partial_{\out^{(k-1)}} \out^{(k), L}_{j,n}$ and $\partial_{\out^{(k-1)}} \out^{(k), U}_{j,n}$ for the McCormick relaxation that leads to Equation \ref{eq:partial_xi_hat_z_non_substituted}. By back-substituting the result of Equation \ref{eq:partial_hat_z_k_1_hat_z_substituted} in Equation \ref{eq:partial_xi_hat_z_non_substituted}, we obtain an expression for the upper and lower bounds on $\partial_{\x_i} \out^{(k)}_{j}$ that only depends on $\partial_{\x_i} \out^{(k-1)}$ and $\x$:
\begin{align}
\label{eq:partial_xi_hat_z_substituted}
\begin{split}
    \partial_{\x_i} \out^{(k),U}_{j} =\, \sum_{n=1}^{d_{k-1}} \alpha^{(k)}_{0, j,n} \partial_{\x_i} \out^{(k-1)}_{n} + \alpha^{(k)}_{3, j,n} \x + \alpha^{(k)}_{4,j,n}\\
    \partial_{\x_i} \out^{(k),L}_{j} =\, \sum_{n=1}^{d_{k-1}} \beta^{(k)}_{0, j,n} \partial_{\x_i} \out^{(k-1)}_{n} + \beta^{(k)}_{3, j,n} \x + \beta^{(k)}_{4, j,n},
\end{split}
\end{align}
where:
\begin{align*}
    \alpha^{(k)}_{3, j,n} = \alpha^{(k),+}_{1, j,n} \iota^{(k)}_{2, j,n} + \alpha^{(k),-}_{1, j,n} \lambda^{(k)}_{2, j,n},\qquad
    \alpha^{(k)}_{4, j,n} = \alpha^{(k),+}_{1, j,n} \iota^{(k)}_{3, j,n} + \alpha^{(k),-}_{1, j,n} \lambda^{(k)}_{3, j,n} + \alpha^{(k)}_{2, j,n}\\
    \beta^{(k)}_{3, j,n} = \beta^{(k),+}_{1, j,n} \lambda^{(k)}_{2, j,n} + \beta^{(k),-}_{1, j,n} \iota^{(k)}_{2, j,n},\qquad
    \beta^{(k)}_{4, j,n} = \beta^{(k),+}_{1, j,n} \lambda^{(k)}_{3, j,n} + \beta^{(k),-}_{1, j,n} \iota^{(k)}_{3, j,n} + \alpha^{(k)}_{2, j,n}
\end{align*}
Given Equation \ref{eq:partial_xi_hat_z_substituted}, we now have a recursive expression for each of the blocks that compose the computation of $\partial_{\x_i} \solnet$, which allows us to obtain a closed form expression for $\partial_{\x_i} \solnet^U$ and $\partial_{\x_i} \solnet^L$ by applying recursive back-substitution starting with Equation \ref{eq:partial_xi_solnet_not_substituted}. Let us begin by performing back-substitution to the result in Equation \ref{eq:partial_xi_hat_z_substituted} for layer $L-1$:
\begin{align}
    \partial_{\x_i} \out^{(L-1),U}_{j} = & \sum_{n=1}^{d_{L-2}} \alpha^{(L-1)}_{0, j,n} \partial_{\x_i} \out^{(L-2)}_{n} + \alpha^{(L-1)}_{3, j,n} \x + \alpha^{(L-1)}_{4,j,n} \label{eq:first_ub} \\
    = & \sum_{n=1}^{d_{L-2}} \alpha^{(L-1)}_{0, j,n} \left(\sum_{r=1}^{d_{L-3}} \mu^{(L-2)}_{0, n,r} \partial_{\x_i} \out^{(L-3)}_{r} + \mu^{(L-2)}_{3, n,r} \x + \mu^{(L-2)}_{4,n,r}\right) + \alpha^{(L-1)}_{3, j,n} \x + \alpha^{(L-1)}_{4,j,n}\\
    = & \sum_{n=1}^{d_{L-2}} \alpha^{(L-1)}_{0, j,n} \left(\sum_{r=1}^{d_{L-3}} \mu^{(L-2)}_{0, n,r} \partial_{\x_i} \out^{(L-3)}_{r}\right) + \alpha^{(L-1)}_{0, j,n} \left(\sum_{r=1}^{d_{L-3}} \mu^{(L-2)}_{3, n,r} \x + \mu^{(L-2)}_{4,n,r}\right) + \alpha^{(L-1)}_{3, j,n} \x + \alpha^{(L-1)}_{4,j,n}\\
    = & \sum_{r=1}^{d_{L-3}} \left(\sum_{n=1}^{d_{L-2}} \alpha^{(L-1)}_{0, j,n} \mu^{(L-2)}_{0, n,r}\right) \partial_{\x_i} \out^{(L-3)}_{r} + \\
    & + \left(\sum_{n=1}^{d_{L-2}} \alpha^{(L-1)}_{0, j,n} \left(\mu^{(L-2)}_{3, n,r} \x + \mu^{(L-2)}_{4,n,r}\right) + \frac{1}{d_{L-3}} \left(\alpha^{(L-1)}_{3, j,n} \x + \alpha^{(L-1)}_{4,j,n}\right)\right)\\
    = & \sum_{r=1}^{d_{L-3}} \rho^{(L-2)}_{0,j,r} \partial_{\x_i} \out^{(L-3)}_{r} + \left(\sum_{n=1}^{d_{L-2}} \alpha^{(L-1)}_{0, j,n} \mu^{(L-2)}_{3, n,r} + \frac{1}{d_{L-3}} \alpha^{(L-1)}_{3, j,n}\right) \x + \\
    & + \left(\sum_{n=1}^{d_{L-1}} \alpha^{(L-1)}_{0, j,n} \mu^{(L-2)}_{4, n,r} + \frac{1}{d_{L-3}} \alpha^{(L-1)}_{4, j,n}\right)\\
    = & \sum_{r=1}^{d_{L-3}} \rho^{(L-2)}_{0,j,r} \partial_{\x_i} \out^{(L-3)}_{r} + \rho^{(L-2)}_{1,j,r} \x + \rho^{(L-2)}_{2,j,r} \label{eq:last_ub},
\end{align}
where:
\begin{align*}
    \rho^{(L-2)}_{0, j,r} & = \sum_{n=1}^{d_{L-2}} \alpha^{(L-1)}_{0, j,n} \mu^{(L-2)}_{0, n,r}\\
    \rho^{(L-2)}_{1, j,r} & = \sum_{n=1}^{d_{L-2}} \alpha^{(L-1)}_{0, j,n} \mu^{(L-2)}_{3, n,r} + \frac{1}{d_{L-2}} \alpha^{(L-1)}_{3, j,n}\\
    \rho^{(L-2)}_{2, j,r} & = \sum_{n=1}^{d_{L-2}} \alpha^{(L-1)}_{0, j,n} \mu^{(L-2)}_{4, n,r} + \frac{1}{d_{L-2}} \alpha^{(L-1)}_{4, j,n},
\end{align*}
and:
\begin{align*}
    \mu^{(L-2)}_{p, n,:} = 
    \begin{cases}
    \alpha^{(L-2)}_{p, n,:}\quad\text{if } \alpha^{(L-1)}_{0, j,n} \geq 0\\
    \beta^{(L-2)}_{p, n,:}\quad\text{if } \alpha^{(L-1)}_{0, j,n} < 0\\
    \end{cases}\text{, } p\in \{0, 3, 4\}
\end{align*}
As in CROWN \citep{zhang2018efficient}, given we have put Equation \ref{eq:last_ub} in the same form as Equation \ref{eq:first_ub}, we can now apply this argument recursively using the $\rho^{(k)}$ and $\mu^{(k)}$ coefficients to obtain:
\begin{equation*}
    \partial_{\x_i} \out^{(L-1),U}_{j} = \rho^{(1)}_{0,j,i} + \sum_{r=1}^{d_0} \rho^{(1)}_{1, j,r} \x + \rho^{(1)}_{2,j,r},
\end{equation*}
where:
\begin{align*}
    \rho^{(k-1)}_{0, j,r} & = 
    \begin{cases}
    \alpha^{(k)}_{0, j,r}\quad\text{if } k = L\\
    \sum_{n=1}^{d_{k-1}} \rho^{(k)}_{0, j,n} \mu^{(k-1)}_{0, n,r}\quad\text{if } k\in \{2,\dots,L-1\}\\
    \end{cases}\\
    \rho^{(k-1)}_{1, j,r} & = \begin{cases}
    \alpha^{(k)}_{3, j,r}\quad\text{if } k = L\\
    \sum_{n=1}^{d_{k-1}} \rho^{(k)}_{0, j,n} \mu^{(k-1)}_{3, n,r} + \frac{1}{d_{k-2}} \rho^{(k)}_{1, j,n}\quad\text{if } k\in \{2,\dots,L-1\}\\
    \end{cases}\\
    \rho^{(k-1)}_{2, j,r} & = \begin{cases}
    \alpha^{(k)}_{4, j,r}\quad\text{if } k = L\\
    \sum_{n=1}^{d_{k-1}} \rho^{(k)}_{0, j,n} \mu^{(k-1)}_{4, n,r} + \frac{1}{d_{k-2}} \rho^{(k)}_{2, j,n}\quad\text{if } k\in \{2,\dots,L-1\}\\
    \end{cases},
\end{align*}
and:
\begin{align*}
    \mu^{(k-1)}_{p, n,:} = 
    \begin{cases}
    \alpha^{(k-1)}_{p, n,:}\quad\text{if } \rho^{(k)}_{0, j,n} \geq 0\\
    \beta^{(k-1)}_{p, n,:}\quad\text{if } \rho^{(k)}_{0, j,n} < 0\\
    \end{cases}\text{, } p\in \{0, 3, 4\}
\end{align*}
And following the same recursive argument:
\begin{equation*}
    \partial_{\x_i} \out^{(L-1),L}_{j} = \tau^{(1)}_{0,j,i} + \sum_{r=1}^{d_0} \tau^{(1)}_{1, j,r} \x + \tau^{(1)}_{2,j,r},
\end{equation*}
where:
\begin{align*}
    \tau^{(k-1)}_{0, j,r} & = 
    \begin{cases}
    \beta^{(k)}_{0, j,r}\quad\text{if } k = L\\
    \sum_{n=1}^{d_{k-1}} \tau^{(k)}_{0, j,n} \omega^{(k-1)}_{0, n,r}\quad\text{if } k\in \{2,\dots,L-1\}\\
    \end{cases}\\
    \tau^{(k-1)}_{1, j,r} & = \begin{cases}
    \beta^{(k)}_{3, j,r}\quad\text{if } k = L\\
    \sum_{n=1}^{d_{k-1}} \tau^{(k)}_{0, j,n} \omega^{(k-1)}_{3, n,r} + \frac{1}{d_{k-2}} \tau^{(k)}_{1, j,n}\quad\text{if } k\in \{2,\dots,L-1\}\\
    \end{cases}\\
    \tau^{(k-1)}_{2, j,r} & = \begin{cases}
    \beta^{(k)}_{4, j,r}\quad\text{if } k = L\\
    \sum_{n=1}^{d_{k-1}} \tau^{(k)}_{0, j,n} \omega^{(k-1)}_{4, n,r} + \frac{1}{d_{k-2}} \tau^{(k)}_{2, j,n}\quad\text{if } k\in \{2,\dots,L-1\}\\
    \end{cases},
\end{align*}
and:
\begin{align*}
    \omega^{(k-1)}_{p, n,:} = 
    \begin{cases}
    \beta^{(k-1)}_{p, n,:}\quad\text{if } \tau^{(k)}_{0, j,n} \geq 0\\
    \alpha^{(k-1)}_{p, n,:}\quad\text{if } \tau^{(k)}_{0, j,n} < 0\\
    \end{cases}\text{, } p\in \{0, 3, 4\}
\end{align*}
With these expressions, we can compute the required $\partial_{\x_i} \out^{(k-1), L}_{n}$ and $\partial_{\x_i} \out^{(k-1), U}_{n}$ which we assumed to be known to derive Equation \ref{eq:partial_xi_hat_z_non_substituted}.

Finally, by back-propagating the bounds starting from Equation \ref{eq:partial_xi_solnet_not_substituted}, we get:
\begin{align*}
    \partial_{\x_i} u^{U}_{\theta, j} = & \sum_{n=1}^{d_{L-1}} \W^{(L),+}_{j,n} \left(\sum_{r=1}^{d_{L-2}} \alpha^{(L-1)}_{0, n,r} \partial_{\x_i} \out^{(L-2)}_{[r]} + \alpha^{(L-1)}_{3, n,r} \x + \alpha^{(L-1)}_{4,n,r}\right) + \\
    & + \W^{(L),-}_{j,n} \left(\sum_{r=1}^{d_{L-2}} \beta^{(L-1)}_{0, n,r} \partial_{\x_i} \out^{(L-2)}_{[r]} + \beta^{(L-1)}_{3, n,r} \x + \beta^{(L-1)}_{4,n,r}\right)\\
    = & \sum_{r=1}^{d_{L-2}} \left(\sum_{n=1}^{d_{L-1}} \W^{(L),+}_{j,n} \alpha^{(L-1)}_{0, n,r} + \W^{(L),-}_{j,n}\beta^{(L-1)}_{0, n,r} \right) \partial_{\x_i} \out^{(L-2)}_{[r]} + \\
    & + \left(\sum_{n=1}^{d_{L-1}} \W^{(L),+}_{j,n} \alpha^{(L-1)}_{3, n,r} + \W^{(L),-}_{j,n} \beta^{(L-1)}_{3,n,r}\right) \x + \left(\sum_{n=1}^{d_{L-1}} \W^{(L),+}_{j,n} \alpha^{(L-1)}_{4, n,r} + \W^{(L),-}_{j,n} \beta^{(L-1)}_{4,n,r}\right) \\
    = & \sum_{r=1}^{d_{L-2}} \phi^{(L-1), U}_{0,j,r} \partial_{\x_i} \out^{(L-2), U}_{n} + \phi^{(L-1), U}_{1,j,r} \x + \phi^{(L-1), U}_{2,j,r},    
\end{align*}
where:
\begin{align*}
    \phi^{(L-1), U}_{0, j,r} & = \sum_{n=1}^{d_{L-1}} \W^{(L),+}_{j,n} \alpha^{(L-1)}_{0, n,r} + \W^{(L),-}_{j,n}\beta^{(L-1)}_{0, n,r}\\
    \phi^{(L-1), U}_{1, j,r} & = \sum_{n=1}^{d_{L-1}} \W^{(L),+}_{j,n} \alpha^{(L-1)}_{3, n,r} + \W^{(L),-}_{j,n} \beta^{(L-1)}_{3,n,r}\\
    \phi^{(L-1), U}_{2, j,r} & = \sum_{n=1}^{d_{L-1}} \W^{(L),+}_{j,n} \alpha^{(L-1)}_{4, n,r} + \W^{(L),-}_{j,n} \beta^{(L-1)}_{4,n,r}.
\end{align*}
From this, using the same back-propagation logic as in the derivations of $\partial_{\x_i} \out^{(k-1), L}_{n}$ and $\partial_{\x_i} \out^{(k-1), U}_{n}$, we can obtain:
\begin{equation}
    \partial_{\x_i} u^{U}_{\theta, j} = \phi^{(1), U}_{0,j,i} + \sum_{r=1}^{d_0} \phi^{(1), U}_{1, j,r} \x + \phi^{(1), U}_{2,j,r},
\end{equation}
where:
\begin{align*}
    \phi^{(k-1), U}_{0, j,r} & = 
    \begin{cases}
        \sum_{n=1}^{d_{k-1}} \W^{(k),+}_{j,n} \alpha^{(k-1)}_{0, n,r} + \W^{(k),-}_{j,n}\beta^{(k-1)}_{0, n,r}\quad\text{if } k = L\\[5pt]
        \sum_{n=1}^{d_{k-1}} \phi^{(k), U}_{0, j,n} \upsilon^{(k-1)}_{0, n,r}\quad\text{if } k\in \{2,\dots,L-1\}\\
    \end{cases}\\
    \phi^{(k-1), U}_{1, j,r} & = \begin{cases}
        \sum_{n=1}^{d_{k-1}} \W^{(k),+}_{j,n} \alpha^{(k-1)}_{3, n,r} + \W^{(k),-}_{j,n}\beta^{(k-1)}_{3, n,r}\quad\text{if } k = L\\[5pt]
        \sum_{n=1}^{d_{k-1}} \phi^{(k), U}_{0, j,n} \upsilon^{(k-1)}_{3, n,r} + \frac{1}{d_{k-2}} \phi^{(k), U}_{1, j,n}\quad\text{if } k\in \{2,\dots,L-1\}\\
    \end{cases}\\
    \phi^{(k-1), U}_{2, j,r} & = \begin{cases}
        \sum_{n=1}^{d_{k-1}} \W^{(k),+}_{j,n} \alpha^{(k-1)}_{4, n,r} + \W^{(k),-}_{j,n}\beta^{(k-1)}_{4, n,r}\quad\text{if } k = L\\[5pt]
        \sum_{n=1}^{d_{k-1}} \phi^{(k), U}_{0, j,n} \upsilon^{(k-1)}_{4, n,r} + \frac{1}{d_{k-2}} \phi^{(k), U}_{2, j,n}\quad\text{if } k\in \{2,\dots,L-1\}\\
    \end{cases},
\end{align*}
and:
\begin{align*}
    \upsilon^{(k-1)}_{p, n,:} = 
    \begin{cases}
    \alpha^{(k-1)}_{p, n,:}\quad\text{if } \phi^{(k), U}_{0, j,n} \geq 0\\
    \beta^{(k-1)}_{p, n,:}\quad\text{if } \phi^{(k), U}_{0, j,n} < 0\\
    \end{cases}\text{, } p\in \{0, 3, 4\}
\end{align*}
And similarly for the lower bound:
\begin{equation}
    \partial_{\x_i} u^{L}_{\theta, j} = \phi^{(1), L}_{0,j,i} + \sum_{r=1}^{d_0} \phi^{(1), L}_{1, j,r} \x + \phi^{(1), L}_{2,j,r},
\end{equation}
where:
\begin{align*}
    \phi^{(k-1), L}_{0, j,r} & = 
    \begin{cases}
    \sum_{n=1}^{d_{k-1}} \W^{(k),+}_{j,n} \beta^{(k-1)}_{0, n,r} + \W^{(k),-}_{j,n}\alpha^{(k-1)}_{0, n,r}\quad\text{if } k = L\\[5pt]
    \sum_{n=1}^{d_{k-1}} \phi^{(k), L}_{0, j,n} \chi^{(k-1)}_{0, n,r}\quad\text{if } k\in \{2,\dots,L-1\}\\
    \end{cases}\\
    \phi^{(k-1), L}_{1, j,r} & = \begin{cases}
    \sum_{n=1}^{d_{k-1}} \W^{(k),+}_{j,n} \beta^{(k-1)}_{3, n,r} + \W^{(k),-}_{j,n}\alpha^{(k-1)}_{3, n,r}\quad\text{if } k = L\\[5pt]
    \sum_{n=1}^{d_{k-1}} \phi^{(k), L}_{0, j,n} \chi^{(k-1)}_{3, n,r} + \frac{1}{d_{k-2}} \phi^{(k), L}_{1, j,n}\quad\text{if } k\in \{2,\dots,L-1\}\\
    \end{cases}\\
    \phi^{(k-1), L}_{2, j,r} & = \begin{cases}
    \sum_{n=1}^{d_{k-1}} \W^{(k),+}_{j,n} \beta^{(k-1)}_{4, n,r} + \W^{(k),-}_{j,n}\alpha^{(k-1)}_{4, n,r}\quad\text{if } k = L\\[5pt]
    \sum_{n=1}^{d_{k-1}} \phi^{(k), L}_{0, j,n} \chi^{(k-1)}_{4, n,r} + \frac{1}{d_{k-2}} \phi^{(k), L}_{2, j,n}\quad\text{if } k\in \{2,\dots,L-1\}\\
    \end{cases},
\end{align*}
and:
\begin{align*}
    \chi^{(k-1)}_{p, n,:} = 
    \begin{cases}
    \beta^{(k-1)}_{p, n,:}\quad\text{if } \phi^{(k), L}_{0, j,n} \geq 0\\
    \alpha^{(k-1)}_{p, n,:}\quad\text{if } \phi^{(k), L}_{0, j,n} < 0\\
    \end{cases}\text{, } p\in \{0, 3, 4\}.
\end{align*}

\subsection{Theorem \ref{thm:second_derivative_bounds} Formal Statement and Proof}
\label{app:proof_bounding_second_derivative}

\begin{reptheorem}{thm:second_derivative_bounds}[\partialcrown: linear lower and upper bounding $\partial_{\x_i^2} \solnet$]
Assume that through a previous computation of bounds on $\partial_{\x_i} \solnet$, the components of that network required for $\partial_{\x_i^2} \solnet$, \ie, $\partial_{\x_i} \out^{(k-1)}$ and $\partial_{\out^{(k-1)}} \out^{(k)}$, are lower and upper bounded by linear functions. In particular, $\vec{C}^{(k), L} \x + \vec{c}^{(k), L} \leq \partial_{\x_i} \out^{(k-1)} \leq \vec{C}^{(k), U} \x + \vec{c}^{(k), U}$ and $\vec{D}^{(k), L} \x + \vec{d}^{(k), L} \leq \partial_{\out^{(k-1)}} \out^{(k)} \leq \vec{D}^{(k), U} \x + \vec{d}^{(k), U}$.

For every $j\in\{1,\dots,d_L\}$ there exist two functions $\partial_{\x_i^2} u^{U}_{\theta, j}$ and $\partial_{\x_i^2} u^{L}_{\theta, j}$ such that, $\forall \x\in \mathcal{C}$ it holds that $\partial_{\x_i^2} u^{L}_{\theta, j} \leq \partial_{\x_i^2} u_{\theta, j} \leq \partial_{\x_i^2} u^{U}_{\theta, j}$. These functions can be written as:
\begin{align*}
    \partial_{\x_i^2} u^{U}_{\theta, j} = \psi^{(1), U}_{0,j,i} + \sum_{r=1}^{d_0} \psi^{(1), U}_{1, j,r} \x + \psi^{(1), U}_{2,j,r}\\
    \partial_{\x_i^2} u^{L}_{\theta, j} = \psi^{(1), L}_{0,j,i} + \sum_{r=1}^{d_0} \psi^{(1), L}_{1, j,r} \x + \psi^{(1), L}_{2,j,r}
\end{align*}
where for $p\in\{0,1,2\}$, $\psi^{(1), U}_{p, j, r}$ and $\psi^{(1), L}_{p, j, r}$ are functions of $\W^{(k)}$, $\preout^{(k), L}$, $\preout^{(k), U}$, $\vec{A}^{(k), L}$, $\vec{A}^{(k), U}$ $\vec{a}^{(k), L}$, $\vec{a}^{(k), U}$, $\vec{C}^{(k), L}$, $\vec{C}^{(k), U}$ $\vec{c}^{(k), L}$, $\vec{c}^{(k), U}$, $\vec{D}^{(k), L}$, $\vec{D}^{(k), U}$ $\vec{d}^{(k), L}$, and $\vec{d}^{(k), U}$, and can be computed using a recursive closed-form expression in $\mathcal{O}(L)$ time.
\end{reptheorem}

\textit{Proof:} Assume that through the computation of the previous bounds on $\solnet$, the pre-activation layer outputs of $\solnet$, $\preout^{(k)}$, are lower and upper bounded by linear functions defined as $\vec{A}^{(k), L} \x + \vec{a}^{(k), L} \leq \preout^{(k)} \leq \vec{A}^{(k), U} \x + \vec{a}^{(k), U}$ and $\preout^{(k), L} \leq \preout^{(k)} \leq \preout^{(k), U}$ for $\x\in \mathcal{C}$. Additionally, we consider also that through a previous computation of bounds on $\partial_{\x_i} \solnet$, the components of that network required for $\partial_{\x_i^2} \solnet$, \ie, $\partial_{\x_i} \out^{(k-1)}$ and $\partial_{\out^{(k-1)}} \out^{(k)}$ are lower and upper bounded by linear functions. In particular, $\vec{C}^{(k), L} \x + \vec{c}^{(k), L} \leq \partial_{\x_i} \out^{(k-1)} \leq \vec{C}^{(k), U} \x + \vec{c}^{(k), U}$ and $\vec{D}^{(k), L} \x + \vec{d}^{(k), L} \leq \partial_{\out^{(k-1)}} \out^{(k)} \leq \vec{D}^{(k), U} \x + \vec{d}^{(k), U}$.

Take the upper and lower bound functions for $\partial_{\x_i^2} \solnet$ as $\partial_{\x_i^2} \solnet^U$ and $\partial_{\x_i^2} \solnet^L$, respectively, and the upper and lower bound functions for $\partial_{\x_i^2} \out^{(k)}$ as $\partial_{\x_i^2} \out^{(k), U}$ and $\partial_{\x_i^2} \out^{(k), L}$, respectively. For the sake of simplicity of notation, we define $\vec{B}^{(k), +} = \mathbb{I}\left(\vec{B}^{(k)} \geq 0\right) \odot \vec{B}^{(k)}$ and $\vec{B}^{(k), -} = \mathbb{I}\left(\vec{B}^{(k)} < 0\right) \odot \vec{B}^{(k)}$.

\textbf{Note that, unless explicitly mentioned otherwise, the non-network variables (denoted by Greek letters, as well as bold, capital and lowercase letters) used here have no relation to the ones from Appendix \ref{app:proof_bounding_first_derivative}}.

Starting backwards from $\partial_{\x_i^2} \out^{(k)}$, we have that:
\begin{align*}
\partial_{\x_i^2} \out^{(k)}_j = \sum_{n=1}^{d_{k-1}} \partial_{\x_i \out^{(k-1)}} \out^{(k)}_{j, n} \partial_{\x_i} \out^{(k-1)}_{n} + \partial_{\out^{(k-1)}} \out^{(k)}_{j, n} \partial_{\x_i^2} \out^{(k-1)}_n.
\end{align*}
Given the transitive property of the sum operator, we can bound $\partial_{\x_i^2} \out^{(k)}_j$ by using a McCormick envelope around each of the multiplications. Assuming that for all $j\in\{1\dots,d_k\}, n\in\{1\dots,d_{k-1}\}$:  $\partial_{\x_i \out^{(k-1)}} \out^{(k), L}_{j, n} \leq \partial_{\x_i \out^{(k-1)}} \out^{(k)}_{j, n} \leq \partial_{\x_i \out^{(k-1)}} \out^{(k), U}_{j, n}$, $\partial_{\x_i} \out^{(k-1), L}_{n} \leq \partial_{\x_i} \out^{(k-1)}_{n} \leq \partial_{\x_i} \out^{(k-1), U}_{n}$, $\partial_{\out^{(k-1)}} \out^{(k)}_{j, n} \leq \partial_{\out^{(k-1)}} \out^{(k)}_{j, n} \leq \partial_{\out^{(k-1)}} \out^{(k)}_{j, n}$, and $\partial_{\x_i^2} \out^{(k-1), L}_n \leq \partial_{\x_i^2} \out^{(k-1)}_n \leq \partial_{\x_i^2} \out^{(k-1), U}_n$, we obtain:
{\small
\begin{align}
\label{eq:partial_xixi_hat_z_non_substituted}
\begin{split}
    \partial_{\x_i^2} \out^{(k)}_j \leq \partial_{\x_i^2} \out^{(k), U}_j = & \sum_{n=1}^{d_{k-1}} \alpha^{(k)}_{0, j,n} \partial_{\x_i} \out^{(k-1)}_{n} + \alpha^{(k)}_{1, j,n} \partial_{\x_i \out^{(k-1)}} \out^{(k)}_{j, n} + \alpha^{(k)}_{2,j,n} \partial_{\x_i^2} \out^{(k-1)}_n + \alpha^{(k)}_{3,j,n} \partial_{\out^{(k-1)}} \out^{(k)}_{j, n} + \alpha^{(k)}_{4,j,n} \\
    \partial_{\x_i^2} \out^{(k)}_j \geq \partial_{\x_i^2} \out^{(k), L}_j = & \sum_{n=1}^{d_{k-1}} \beta^{(k)}_{0, j,n} \partial_{\x_i} \out^{(k-1)}_{n} + \beta^{(k)}_{1, j,n} \partial_{\x_i \out^{(k-1)}} \out^{(k)}_{j, n} + \beta^{(k)}_{2,j,n} \partial_{\x_i^2} \out^{(k-1)}_n + \beta^{(k)}_{3,j,n} \partial_{\out^{(k-1)}} \out^{(k)}_{j, n} + \beta^{(k)}_{4,j,n}
\end{split}
\end{align}
}
for:
{\small
\begin{align*}
    \alpha^{(k)}_{0, j,n} = & \eta^{(k)}_{j,n} \partial_{\x_i\out^{(k-1)}} \out^{(k), U}_{j,n} + \left(1-\eta^{(k)}_{j,n}\right) \partial_{\x_i\out^{(k-1)}} \out^{(k), L}_{j,n}\qquad
    \alpha^{(k)}_{1, j,n} = \eta^{(k)}_{j,n} \partial_{\x_i} \out^{(k-1), L}_{n} + \left(1-\eta^{(k)}_{j,n}\right) \partial_{\x_i} \out^{(k-1), U}_{n}\\
    \alpha^{(k)}_{2, j,n} = & \gamma^{(k)}_{j,n} \partial_{\out^{(k-1)}} \out^{(k), U}_{j,n} + \left(1-\gamma^{(k)}_{j,n}\right) \partial_{\out^{(k-1)}} \out^{(k), L}_{j,n}\qquad
    \alpha^{(k)}_{3, j,n} = \gamma^{(k)}_{j,n} \partial_{\x_i^2} \out^{(k-1), L}_{n} + \left(1-\gamma^{(k)}_{j,n}\right) \partial_{\x_i^2} \out^{(k-1), U}_{n}\\
    \alpha^{(k)}_{4, j,n} = & -\eta^{(k)}_{j,n} \partial_{\x_i \out^{(k-1)}} \out^{(k), U}_{j,n} \partial_{\x_i} \out^{(k-1), L}_{n} - \left(1-\eta^{(k)}_{j,n}\right) \partial_{\x_i \out^{(k-1)}} \out^{(k), L}_{j,n} \partial_{\x_i} \out^{(k-1), U}_{n} +\\
    & -\gamma^{(k)}_{j,n} \partial_{\out^{(k-1)}} \out^{(k), U}_{j,n} \partial_{\x_i x_i} \out^{(k-1), L}_{n} - \left(1-\gamma^{(k)}_{j,n}\right) \partial_{\out^{(k-1)}} \out^{(k), L}_{j,n} \partial_{\x_i x_i} \out^{(k-1), U}_{n}\\[8pt]
    \beta^{(k)}_{0, j,n} = & \zeta^{(k)}_{j,n} \partial_{\x_i \out^{(k-1)}} \out^{(k), L}_{j,n} + \left(1-\zeta^{(k)}_{j,n}\right) \partial_{\x_i \out^{(k-1)}} \out^{(k), U}_{j,n}\qquad
    \beta^{(k)}_{1, j,n} = \zeta^{(k)}_{j,n} \partial_{\x_i} \out^{(k-1), L}_{n} + \left(1-\zeta^{(k)}_{j,n}\right) \partial_{\x_i} \out^{(k-1), U}_{n}\\
    \beta^{(k)}_{2, j,n} = & \delta^{(k)}_{j,n} \partial_{\out^{(k-1)}} \out^{(k), L}_{j,n} + \left(1-\delta^{(k)}_{j,n}\right) \partial_{\out^{(k-1)}} \out^{(k), U}_{j,n}\qquad
    \beta^{(k)}_{3, j,n} = \delta^{(k)}_{j,n} \partial_{\x_i^2} \out^{(k-1), L}_{n} + \left(1-\delta^{(k)}_{j,n}\right) \partial_{\x_i^2} \out^{(k-1), U}_{n}\\
    \beta^{(k)}_{4, j,n} = & -\zeta^{(k)}_{j,n} \partial_{\x_i \out^{(k-1)}} \out^{(k), L}_{j,n}  \partial_{\x_i} \out^{(k-1), L}_{n}  - \left(1-\zeta^{(k)}_{j,n}\right) \partial_{\x_i \out^{(k-1)}} \out^{(k), U}_{j,n} \partial_{\x_i} \out^{(k-1), U}_{n} +\\
    & -\delta^{(k)}_{j,n} \partial_{\out^{(k-1)}} \out^{(k), L}_{j,n} \partial_{\x_i^2} \out^{(k-1), L}_{n} - \left(1-\delta^{(k)}_{j,n}\right) \partial_{\out^{(k-1)}} \out^{(k), U}_{j,n} \partial_{\x_i^2} \out^{(k-1), U}_{n},
\end{align*}
}
where $\eta^{(k)}_{j,n}$, $\gamma^{(k)}_{j,n}$, $\zeta^{(k)}_{j,n}$ and $\delta^{(k)}_{j,n}$ are convex coefficients that can be set as hyperparameters, or optimized for as in $\alpha$-CROWN \citep{xu2020fast}.

For the next step of the back-propagation process, we now need to bound $\partial_{\x_i} \out^{(k-1)}_{n}$, $\partial_{\x_i \out^{(k-1)}} \out^{(k)}_{j, n}$, and $\partial_{\out^{(k-1)}} \out^{(k)}_{j, n}$, so as to eventually be able to write $\partial_{\x_i^2} \out^{(k)}_j$ as a function of simply $\partial_{\x_i^2} \out^{(k-1)}_n$ and $\x$. As per our assumptions at the beginning of this section, for the sake of computational efficiency we take $\partial_{\x_i} \out^{(k-1)}_{n}$ and $\partial_{\out^{(k-1)}} \out^{(k)}_{j, n}$ from the computation of the bounds of $\partial_{\x_i} u_{\theta, j}$, and thus assume we have a linear upper and lower bound function of $\x$. This leaves us with $\partial_{\x_i \out^{(k-1)}} \out^{(k)}_{j, n}$ to bound as a linear function of $\x$.

Note that, as per Lemma \ref{lem:second_derivative}, $\partial_{\x_i \out^{(k-1)}} \out^{(k)}_{j, n} = \sigma''\left(\vec{y}^{(k)}_j\right) \left(\W^{(k)}_{j,:} \partial_{\x_i} \out^{(k-1)}\right) \W^{(k)}_{j,n}$. Since $\left(\W^{(k)}_{j,:} \partial_{\x_i} \out^{(k-1)}\right) = \sum_{n=1}^{d_{k-1}} \W^{(k)}_{j,n} \partial_{\x_i} \out^{(k-1)}_n$, and $\vec{C}^{(k), U}_{n,:} \x + \vec{c}^{(k), U}_n \leq \partial_{\x_i} \out^{(k-1)}_n \leq \vec{C}^{(k), L}_{n,:} \x + \vec{c}^{(k), L}_n$ (from the assumptions above), we can write:
\begin{align*}
\W^{(k)}_{j,:} \partial_{\x_i} \out^{(k-1)} \leq & \underbrace{\left(\sum_{n=1}^{d_{k-1}} \W^{(k),+}_{j,n} \vec{C}^{(k), U}_{n,:} + \W^{(k),-}_{j,n} \vec{C}^{(k), L}_{n,:}\right)}_{\vec{E}_j^{(k), U}}\x + \underbrace{\left(\sum_{n=1}^{d_{k-1}} \W^{(k),+}_{j,n}\vec{c}^{(k), U} + \W^{(k),-}_{j,n} \vec{c}^{(k), L}\right)}_{\vec{e}_j^{(k), U}}\\
\W^{(k)}_{j,n} \partial_{\x_i} \out^{(k-1)} \geq & \underbrace{\left(\sum_{n=1}^{d_{k-1}} \W^{(k),+}_{j,n} \vec{C}^{(k), L}_{n,:} + \W^{(k),-}_{j,n} \vec{C}^{(k), U}_{n,:}\right)}_{\vec{E}_j^{(k), L}}\x + \underbrace{\left(\sum_{n=1}^{d_{k-1}} \W^{(k),+}_{j,n}\vec{c}^{(k), L}_n + \W^{(k),-}_{j,n} \vec{c}^{(k), U}_n\right)}_{\vec{e}_j^{(k), L}}.
\end{align*}
We define $\theta^{(k), U}_j =  \max_{\x \in \mathcal{C}} \vec{E}_j^{(k), U}\x + \vec{e}_j^{(k), U}$ and $\theta^{(k), L}_j =  \min_{\x \in \mathcal{C}} \vec{E}_j^{(k), L}\x + \vec{e}_j^{(k), L}$.
As with the first derivative case, since $\preout^{(k), L}_{j} \leq \preout^{(k)}_{j} \leq \preout^{(k), U}_{j}$, we can obtain a linear upper and lower bound relaxation for $\sigma''\left(\preout^{(k)}_{j}\right)$, such that $\lambda^{(k),L}_{j}\left(\preout^{(k)}_{j} + \mu^{(k),L}_{j}\right) \leq \sigma''\left(\preout^{(k)}_{j}\right)\leq \lambda^{(k),U}_{j}\left(\preout^{(k)}_{j} + \mu^{(k),U}_{j}\right)$, as well as the values $\iota^{(k), L}_j \leq \sigma''\left(\preout^{(k)}_{j}\right) \leq \iota^{(k), U}_j$. By considering the assumption that $\vec{A}^{(k), U}_{j,:} \x + \vec{a}^{(k), U}_{j} \leq \preout^{(k)} \leq \vec{A}^{(k), L}_{j,:} \x + \vec{a}^{(k), L}_{j}$, we can obtain:
\begin{align*}
    \sigma''\left(\preout^{(k)}_{j}\right) \leq \underbrace{\left(\lambda^{(k),U,+}_{j} \vec{A}^{(k), U}_{j,:} + \lambda^{(k),U,-}_{j} \vec{A}^{(k), L}_{j,:}\right)}_{\vec{H}^{(k), U}_{j}} \x + \underbrace{\left(\lambda^{(k),U,+}_{j} \vec{a}^{(k), U}_{j} + \lambda^{(k),U,-}_{j} \vec{a}^{(k), L}_{j} + \lambda^{(k),U}_{j}\mu^{(k),U}_{j}\right)}_{\vec{h}^{(k), U}_{j}}\\
    \sigma''\left(\preout^{(k)}_{j}\right) \geq \underbrace{\left(\lambda^{(k),L,+}_{j} \vec{A}^{(k), L}_{j,:} + \lambda^{(k),L,-}_{j} \vec{A}^{(k), U}_{j,:}\right)}_{\vec{H}^{(k), L}_{j}} \x + \underbrace{\left(\lambda^{(k),L,+}_{j} \vec{a}^{(k), L}_{j} + \lambda^{(k),L,-}_{j} \vec{a}^{(k), U}_{j} + \lambda^{(k),L}_{j}\mu^{(k),L}_{j}\right)}_{\vec{h}^{(k), L}_{j}}.
\end{align*}

This allows us to relax $\sigma''\left(\preout^{(k)}_{j}\right) \left(\W^{(k)}_{j,:} \partial_{\x_i} \out^{(k-1)}\right)$ using McCormick envelopes:
\begin{align*}
\sigma''\left(\preout^{(k)}_{j}\right) \left(\W^{(k)}_{j,:} \partial_{\x_i} \out^{(k-1)}\right) \leq \nu_{0, j}^{(k), U} \left(\W^{(k)}_{j,:} \partial_{\x_i} \out^{(k-1)}\right) + \nu_{1, j}^{(k), U} \sigma''\left(\preout^{(k)}_{j}\right) + \nu_{2, j}^{(k), U}\\
\sigma''\left(\preout^{(k)}_{j}\right) \left(\W^{(k)}_{j,:} \partial_{\x_i} \out^{(k-1)}\right) \geq \nu_{0, j}^{(k), L} \left(\W^{(k)}_{j,:} \partial_{\x_i} \out^{(k-1)}\right) + \nu_{1, j}^{(k), L} \sigma''\left(\preout^{(k)}_{j}\right) + \nu_{2, j}^{(k), L},
\end{align*}
for:
\begin{align*}
    \nu^{(k), U}_{0, j} & = \rho^{(k)}_j \iota^{(k), U}_j + \left(1-\rho^{(k)}_j\right) \iota^{(k), L}_j\qquad
    \nu^{(k), U}_{1, j,n} = \rho^{(k)}_j \theta^{(k), L}_j + \left(1-\rho^{(k)}_j\right) \iota^{(k), U}_j\\
    \nu^{(k), U}_{2, j} & = -\rho^{(k)}_j \iota^{(k), U}_j \theta^{(k), L}_j - \left(1-\rho^{(k)}_j\right) \iota^{(k), L}_j \theta^{(k), U}_j\\[10pt]
    \nu^{(k), L}_{0, j} & = \tau^{(k)}_j \iota^{(k), L}_j + \left(1-\tau^{(k)}_j\right) \iota^{(k), U}_j\qquad
    \nu^{(k), L}_{1, j} = \tau^{(k)}_j \theta^{(k), L}_j + \left(1-\tau^{(k)}_j\right) \theta^{(k), U}_j\\
    \nu^{(k), L}_{2, j} & = -\tau^{(k)}_j \iota^{(k), L}_j \theta^{(k), L}_j  - \left(1-\tau^{(k)}_j\right)\iota^{(k), U}_j \theta^{(k), U}_j,
\end{align*}
where $\rho^{(k)}_{j}$ and $\tau^{(k)}_{j}$ are convex coefficients that can be set as hyperparameters, or optimized for as in $\alpha$-CROWN \citep{xu2020fast}. By replacing this multiplication in the expression from Lemma \ref{lem:second_derivative}, we bound $\partial_{\x_i \out^{(k-1)}} \out^{(k)}_{j, n}$ as:
\begin{align*}
    \partial_{\x_i \out^{(k-1)}} \out^{(k)}_{j, n} \leq \upsilon^{(k), U}_{0, j, n} \left(\W^{(k)}_{j,:} \partial_{\x_i} \out^{(k-1)}\right) + \upsilon^{(k), U}_{1, j, n} \sigma''\left(\preout^{(k)}_{j}\right) + \upsilon^{(k), U}_{2, j}\\
    \partial_{\x_i \out^{(k-1)}} \out^{(k)}_{j, n} \geq \upsilon^{(k), L}_{0, j, n} \left(\W^{(k)}_{j,:} \partial_{\x_i} \out^{(k-1)}\right) + \upsilon^{(k), L}_{1, j, n} \sigma''\left(\preout^{(k)}_{j}\right) + \upsilon^{(k), L}_{2, j},
\end{align*}
for:
\begin{align*}
    \upsilon^{(k), U}_{i, j, n} = \nu_{i, j}^{(k), U}\W^{(k), +}_{j,n} + \nu_{i, j}^{(k), L}\W^{(k), -}_{j,n}, \quad
    \upsilon^{(k), L}_{i, j, n} = \nu_{i, j}^{(k), L}\W^{(k), +}_{j,n} + \nu_{i, j}^{(k), U}\W^{(k), -}_{j,n}\qquad
    i\in\{0, 1, 2\}.
\end{align*}
By replacing the lower and upper bounds for $\sigma''(y_j^{(k)})$ and $\left(\W^{(k)}_{j,:} \partial_{\x_i} \out^{(k-1)}\right)$ in the previous inequality, we obtain the expression:
\begin{align*}
    \partial_{\x_i \out^{(k-1)}} \out^{(k)}_{j, n} \leq \vec{M}^{(k), U}_{j,n} \x + \vec{m}^{(k), U}_{j,n}\\
    \partial_{\x_i \out^{(k-1)}} \out^{(k)}_{j, n} \geq \vec{M}^{(k), L}_{j,n} \x + \vec{m}^{(k), L}_{j,n},
\end{align*}
for:
\begin{align*}
    \vec{M}^{(k), U}_{j,n} = & \upsilon^{(k), U, +}_{0, j, n} \vec{E}_j^{(k), U} + \upsilon^{(k), U, -}_{0, j, n} \vec{E}_j^{(k), L} + \upsilon^{(k), U, +}_{1, j, n} \vec{H}_j^{(k), U} + \upsilon^{(k), U, -}_{1, j, n} \vec{H}_j^{(k), L}\\
    \vec{m}^{(k), U}_{j,n} = & \upsilon^{(k), U, +}_{0, j, n} \vec{e}_j^{(k), U} + \upsilon^{(k), U, -}_{0, j, n} \vec{e}_j^{(k), L} + \upsilon^{(k), U, +}_{1, j, n} \vec{h}_j^{(k), U} + \upsilon^{(k), U, -}_{1, j, n} \vec{h}_j^{(k), L} + \upsilon^{(k), U}_{2, j, n}\\
    \vec{M}^{(k), L}_{j,n} = & \upsilon^{(k), L, +}_{0, j, n} \vec{E}_j^{(k), L} + \upsilon^{(k), L, -}_{0, j, n} \vec{E}_j^{(k), U} + \upsilon^{(k), L, +}_{1, j, n} \vec{H}_j^{(k), L} + \upsilon^{(k), L, -}_{1, j, n} \vec{H}_j^{(k), U}\\
    \vec{m}^{(k), L}_{j,n} = & \upsilon^{(k), L, +}_{0, j, n} \vec{e}_j^{(k), L} + \upsilon^{(k), L, -}_{0, j, n} \vec{e}_j^{(k), U} + \upsilon^{(k), L, +}_{1, j, n} \vec{h}_j^{(k), L} + \upsilon^{(k), L, -}_{1, j, n} \vec{h}_j^{(k), U} + \upsilon^{(k), L}_{2, j, n}.\\
\end{align*}
Finally in the derivation of $\partial_{\x_i^2} \out^{(k)}_j$ as a function of $\x$ and $\partial_{\x_i^2} \out^{(k-1)}$, we just have to replace all the quantities in Equation \ref{eq:partial_xixi_hat_z_non_substituted} (recalling from the assumptions that $\vec{C}^{(k), U} \x + \vec{c}^{(k), U} \leq \partial_{\x_i} \out^{(k-1)} \leq \vec{C}^{(k), L} \x + \vec{c}^{(k), L}$ and $\vec{D}^{(k), U} \x + \vec{d}^{(k), U} \leq \partial_{\out^{(k-1)}} \out^{(k)} \leq \vec{D}^{(k), L} \x + \vec{d}^{(k), L}$) to obtain:
\begin{align}
\label{eq:partial_xixi_hat_z_substituted}
\begin{split}
    \partial_{\x_i^2} \out^{(k)}_j \leq \partial_{\x_i^2} \out^{(k), U}_j = & \sum_{n=1}^{d_{k-1}} \alpha^{(k)}_{2,j,n} \partial_{\x_i^2} \out^{(k-1)}_n + \alpha^{(k)}_{5, j,n} \x + \alpha^{(k)}_{6,j,n} \\
    \partial_{\x_i^2} \out^{(k)}_j \geq \partial_{\x_i^2} \out^{(k), L}_j = & \sum_{n=1}^{d_{k-1}} \beta^{(k)}_{2,j,n} \partial_{\x_i^2} \out^{(k-1)}_n + \beta^{(k)}_{5, j,n} \x + \beta^{(k)}_{6,j,n},
\end{split}
\end{align}
where:
\begin{align*}
    \alpha^{(k)}_{5, j,n} & = \alpha^{(k),+}_{0, j,n} \vec{C}^{(k), U}_n + \alpha^{(k),-}_{0, j,n} \vec{C}^{(k), L}_n + \alpha^{(k),+}_{1, j,n} \vec{M}^{(k), U}_{j,n} + \alpha^{(k),-}_{1, j,n} \vec{M}^{(k), L}_{j,n} + \alpha^{(k), +}_{3,j,n} \vec{D}^{(k), U}_{j,n} + \alpha^{(k), -}_{3,j,n} \vec{D}^{(k), L}_{j,n}\\
    \alpha^{(k)}_{6, j,n} & = \alpha^{(k),+}_{0, j,n} \vec{c}^{(k), U}_n + \alpha^{(k),-}_{0, j,n} \vec{c}^{(k), L}_n + \alpha^{(k),+}_{1, j,n} \vec{m}^{(k), U}_{j,n} + \alpha^{(k),-}_{1, j,n} \vec{m}^{(k), L}_{j,n} + \alpha^{(k), +}_{3,j,n} \vec{d}^{(k), U}_{j,n} + \alpha^{(k), -}_{3,j,n} \vec{d}^{(k), L}_{j,n} + \alpha^{(k)}_{4,j,n}\\
    \beta^{(k)}_{5, j,n} & = \beta^{(k),+}_{0, j,n} \vec{C}^{(k), L}_n + \beta^{(k),-}_{0, j,n} \vec{C}^{(k), U}_n + \beta^{(k),+}_{1, j,n} \vec{M}^{(k), L}_{j,n} + \beta^{(k),-}_{1, j,n} \vec{M}^{(k), U}_{j,n} + \beta^{(k), +}_{3,j,n} \vec{D}^{(k), L}_{j,n} + \beta^{(k), -}_{3,j,n} \vec{D}^{(k), U}_{j,n}\\
    \beta^{(k)}_{6, j,n} & = \beta^{(k),+}_{0, j,n} \vec{c}^{(k), L}_n + \beta^{(k),-}_{0, j,n} \vec{c}^{(k), U}_n + \beta^{(k),+}_{1, j,n} \vec{m}^{(k), L}_{j,n} + \beta^{(k),-}_{1, j,n} \vec{m}^{(k), U}_{j,n} + \beta^{(k), +}_{3,j,n} \vec{d}^{(k), L}_{j,n} + \beta^{(k), -}_{3,j,n} \vec{d}^{(k), U}_{j,n} + \beta^{(k)}_{4,j,n}
\end{align*}
This forms a recursion of exactly the same form as Equation \ref{eq:partial_xi_hat_z_substituted} from Appendix \ref{app:proof_bounding_first_derivative}, where only the coefficients of $\partial_{\x_i^2} \out^{(k-1)}_n$ and $\x$ are different ($\alpha^{(k)}_{0,j,n}$ in this case is referred by $\alpha^{(k)}_{2,j,n}$, $\alpha^{(k)}_{3,j,n}$ by $\alpha^{(k)}_{5,j,n}$, and $\alpha^{(k)}_{4,j,n}$ by $\alpha^{(k)}_{6,j,n}$, and similarly for the $\beta$ values). This yields:
\begin{equation*}
    \partial_{\x_i x_i} \out^{(L-1),U}_{j} = \rho^{(1), U}_{0,j,i} + \sum_{r=1}^{d_0} \rho^{(1), U}_{1, j,r} \x + \rho^{(1), U}_{2,j,r},
\end{equation*}
where:
\begin{align*}
    \rho^{(k-1), U}_{0, j,r} & = 
    \begin{cases}
    \alpha^{(k)}_{2, n,r}\quad\text{if } k = L\\
    \sum_{n=1}^{d_{k-1}} \rho^{(k), U}_{0, j,n} \mu^{(k-1), U}_{2, n,r}\quad\text{if } k\in \{2,\dots,L-1\}\\
    \end{cases}\\
    \rho^{(k-1), U}_{1, j,r} & = \begin{cases}
    \alpha^{(k)}_{5, n,r}\quad\text{if } k = L\\
    \sum_{n=1}^{d_{k-1}} \rho^{(k), U}_{0, j,n} \mu^{(k-1), U}_{5, n,r} + \frac{1}{d_{k-2}} \rho^{(k), U}_{1, j,n}\quad\text{if } k\in \{2,\dots,L-1\}\\
    \end{cases}\\
    \rho^{(k-1), U}_{2, j,r} & = \begin{cases}
    \alpha^{(k)}_{6, n,r}\quad\text{if } k = L\\
    \sum_{n=1}^{d_{k-1}} \rho^{(k), U}_{0, j,n} \mu^{(k-1), U}_{6, n,r} + \frac{1}{d_{k-2}} \rho^{(k), U}_{2, j,n}\quad\text{if } k\in \{2,\dots,L-1\}\\
    \end{cases},
\end{align*}
and:
\begin{align*}
    \mu^{(k-1), U}_{p, n,:} = 
    \begin{cases}
    \alpha^{(k-1)}_{p, n,:}\quad\text{if } \rho^{(k), U}_{0, j,n} \geq 0\\
    \beta^{(k-1)}_{p, n,:}\quad\text{if } \rho^{(k), U}_{0, j,n} < 0\\
    \end{cases}\text{, } p\in \{2, 5, 6\}.
\end{align*}
And following the same argument:
\begin{equation*}
    \partial_{\x_i} \out^{(L-1),L}_{j} = \rho^{(1),L}_{0,j,i} + \sum_{r=1}^{d_0} \rho^{(1),L}_{1, j,r} \x + \rho^{(1),L}_{2,j,r},
\end{equation*}
where:
\begin{align*}
    \rho^{(k-1), L}_{0, j,r} & = 
    \begin{cases}
    \beta^{(k)}_{2, n,r}\quad\text{if } k = L\\
    \sum_{n=1}^{d_{k-1}} \rho^{(k),L}_{0, j,n} \mu^{(k-1), L}_{2, n,r}\quad\text{if } k\in \{2,\dots,L-1\}\\
    \end{cases}\\
    \rho^{(k-1), L}_{1, j,r} & = \begin{cases}
    \beta^{(k)}_{5, n,r}\quad\text{if } k = L\\
    \sum_{n=1}^{d_{k-1}} \rho^{(k),L}_{0, j,n} \mu^{(k-1), L}_{5, n,r} + \frac{1}{d_{k-2}} \rho^{(k),L}_{1, j,n}\quad\text{if } k\in \{2,\dots,L-1\}\\
    \end{cases}\\
    \rho^{(k-1), L}_{2, j,r} & = \begin{cases}
    \beta^{(k)}_{6, n,r}\quad\text{if } k = L\\
    \sum_{n=1}^{d_{k-1}} \rho^{(k),L}_{0, j,n} \mu^{(k-1), L}_{6, n,r} + \frac{1}{d_{k-2}} \rho^{(k),L}_{2, j,n}\quad\text{if } k\in \{2,\dots,L-1\}\\
    \end{cases},
\end{align*}
and:
\begin{align*}
    \mu^{(k-1),L}_{p, n,:} = 
    \begin{cases}
    \beta^{(k-1)}_{p, n,:}\quad\text{if } \rho^{(k), L}_{0, j,n} \geq 0\\
    \alpha^{(k-1)}_{p, n,:}\quad\text{if } \rho^{(k), L}_{0, j,n} < 0\\
    \end{cases}\text{, } p\in \{2, 5, 6\}
\end{align*}
With these expressions, we can compute the required $\partial_{\x_i^2} \out^{(k-1), L}_{n}$ and $\partial_{\x_i^2} \out^{(k-1), U}_{n}$ which we assumed to be known to derive Equation \ref{eq:partial_xixi_hat_z_non_substituted}.

Finally, with the exact same argument as in Appendix~\ref{app:proof_bounding_first_derivative}, we obtain:
\begin{equation*}
    \partial_{\x_i} u^{U}_{\theta, j} = \psi^{(1), U}_{0,j,i} + \sum_{r=1}^{d_0} \psi^{(1), U}_{1, j,r} \x + \psi^{(1), U}_{2,j,r},
\end{equation*}
where:
\begin{align*}
    \psi^{(k-1), U}_{0, j,r} & = 
    \begin{cases}
        \sum_{n=1}^{d_{k-1}} \W^{(k),+}_{j,n} \alpha^{(k-1)}_{2, n,r} + \W^{(k),-}_{j,n}\beta^{(k-1)}_{2, n,r}\quad\text{if } k = L\\[5pt]
        \sum_{n=1}^{d_{k-1}} \psi^{(k), U}_{0, j,n} \psi^{(k-1), U}_{2, n,r}\quad\text{if } k\in \{2,\dots,L-1\}\\
    \end{cases}\\
    \psi^{(k-1), U}_{1, j,r} & = \begin{cases}
        \sum_{n=1}^{d_{k-1}} \W^{(k),+}_{j,n} \alpha^{(k-1)}_{5, n,r} + \W^{(k),-}_{j,n}\beta^{(k-1)}_{5, n,r}\quad\text{if } k = L\\[5pt]
        \sum_{n=1}^{d_{k-1}} \psi^{(k), U}_{0, j,n} \psi^{(k-1), U}_{5, n,r} + \frac{1}{d_{k-2}} \psi^{(k), U}_{1, j,n}\quad\text{if } k\in \{2,\dots,L-1\}\\
    \end{cases}\\
    \psi^{(k-1), U}_{2, j,r} & = \begin{cases}
        \sum_{n=1}^{d_{k-1}} \W^{(k),+}_{j,n} \alpha^{(k-1)}_{6, n,r} + \W^{(k),-}_{j,n}\beta^{(k-1)}_{6, n,r}\quad\text{if } k = L\\[5pt]
        \sum_{n=1}^{d_{k-1}} \psi^{(k), U}_{0, j,n} \psi^{(k-1), U}_{6, n,r} + \frac{1}{d_{k-2}} \psi^{(k), U}_{2, j,n}\quad\text{if } k\in \{2,\dots,L-1\}\\
    \end{cases},
\end{align*}
and:
\begin{align*}
    \psi^{(k-1)}_{p, n,:} = 
    \begin{cases}
    \alpha^{(k-1)}_{p, n,:}\quad\text{if } \psi^{(k), U}_{0, j,n} \geq 0\\
    \beta^{(k-1)}_{p, n,:}\quad\text{if } \psi^{(k), U}_{0, j,n} < 0\\
    \end{cases}\text{, } p\in \{2, 5, 6\}.
\end{align*}
And similarly for the lower bound:
\begin{equation*}
    \partial_{\x_i} u^{L}_{\theta, j} = \psi^{(1), L}_{0,j,i} + \sum_{r=1}^{d_0} \psi^{(1), L}_{1, j,r} \x + \psi^{(1), L}_{2,j,r},
\end{equation*}
where:
\begin{align*}
    \psi^{(k-1), L}_{0, j,r} & = 
    \begin{cases}
    \sum_{n=1}^{d_{k-1}} \W^{(k),+}_{j,n} \beta^{(k-1)}_{2, n,r} + \W^{(k),-}_{j,n}\alpha^{(k-1)}_{2, n,r}\quad\text{if } k = L\\[5pt]
    \sum_{n=1}^{d_{k-1}} \psi^{(k), L}_{0, j,n} \psi^{(k-1), L}_{2, n,r}\quad\text{if } k\in \{2,\dots,L-1\}\\
    \end{cases}\\
    \psi^{(k-1), L}_{1, j,r} & = \begin{cases}
    \sum_{n=1}^{d_{k-1}} \W^{(k),+}_{j,n} \beta^{(k-1)}_{5, n,r} + \W^{(k),-}_{j,n}\alpha^{(k-1)}_{5, n,r}\quad\text{if } k = L\\[5pt]
    \sum_{n=1}^{d_{k-1}} \psi^{(k), L}_{0, j,n} \psi^{(k-1), L}_{5, n,r} + \frac{1}{d_{k-2}} \psi^{(k), L}_{1, j,n}\quad\text{if } k\in \{2,\dots,L-1\}\\
    \end{cases}\\
    \psi^{(k-1), L}_{2, j,r} & = \begin{cases}
    \sum_{n=1}^{d_{k-1}} \W^{(k),+}_{j,n} \beta^{(k-1)}_{6, n,r} + \W^{(k),-}_{j,n}\alpha^{(k-1)}_{6, n,r}\quad\text{if } k = L\\[5pt]
    \sum_{n=1}^{d_{k-1}} \psi^{(k), L}_{0, j,n} \psi^{(k-1), L}_{6, n,r} + \frac{1}{d_{k-2}} \psi^{(k), L}_{2, j,n}\quad\text{if } k\in \{2,\dots,L-1\}\\
    \end{cases},
\end{align*}
and:
\begin{align*}
    \psi^{(k-1), L}_{p, n,:} = 
    \begin{cases}
    \beta^{(k-1)}_{p, n,:}\quad\text{if } \psi^{(k), L}_{0, j,n} \geq 0\\
    \alpha^{(k-1)}_{p, n,:}\quad\text{if } \psi^{(k), L}_{0, j,n} < 0\\
    \end{cases}\text{, } p\in \{2, 5, 6\}.
\end{align*}

\subsection{Formulation and proof of closed-form global bounds on \texorpdfstring{$\partial_{\x_i} \solnet$}{∂\_\{x\_i\} u\_θ}}
\label{app:proof_lemma_closed_form_bounds}

\begin{lemma}[Closed-form global bounds on $\partial_{\x_i} \solnet$]
\label{lem:closed_form_bounds}
For every $j\in \{1,\dots,d_L\}$ there exist two values $\kappa^U_j \in \mathbb{R}$ and $\kappa^L_j \in \mathbb{R}$, such that $\forall \x \in \mathcal{C} = \{\x \in \mathbb{R}^{d_0}: \x^L \leq \x \leq \x^U\}$ it holds that $\kappa^L_j \leq \partial_{\x_i} u_{\theta,j} \leq \kappa^U_j$, with:
\begin{align*}
    \kappa^U_j = \vec{B}^{U, +} \x^U + \vec{B}^{U, -} \x^L + \phi^{(1)}_{0,j,i} + \sum_{r=1}^{d_0} \phi^{(1)}_{2,j,r}\\
    \kappa^L_j = \vec{B}^{L, +} \x^L + \vec{B}^{L, -} \x^U + \psi^{(1)}_{0,j,i} + \sum_{r=1}^{d_0} \psi^{(1)}_{2,j,r},
\end{align*}
where $\vec{B}^{U} = \sum_{r=1}^{d_0} \phi^{(1)}_{1, j,r}$, $\vec{B}^{L} = \sum_{r=1}^{d_0} \psi^{(1)}_{1, j,r}$, and $\vec{B}^{\cdot, +} = \mathbb{I}\left(\vec{B}^{\cdot} \geq 0\right) \odot \vec{B}^{\cdot}$ and $\vec{B}^{\cdot, -} = \mathbb{I}\left(\vec{B}^{\cdot} < 0\right) \odot \vec{B}^{\cdot}$.
\end{lemma}

\begin{proof}
Take a function $f: \mathbb{R}^{d_0} \to \mathbb{R}$ defined as $f(\x) = \vec{v}^\top \x + c$ for $\vec{v} \in \mathbb{R}^{d_0}$ and $c \in \mathbb{R}$, as well as a domain $\mathcal{C} = \{\x \in \mathbb{R}^{d_0}: \x^L \leq \x \leq \x^U\}$. Given the perpendicularity of the constraints in $\mathcal{C}$, by separating each component of $f$ we obtain:
\begin{align*}
    \max_{\x \in \mathcal{C}} f(\x) = (\vec{v}^{+})^\top \x^{U} + (\vec{v}^{-})^\top \x^{L} + c,\qquad \min_{\x \in \mathcal{C}} f(\x) = (\vec{v}^{+})^\top \x^{L} + (\vec{v}^{-})^\top \x^{U} + c,
\end{align*}
where $\vec{v}^{+} = \mathbb{I}\left(\vec{v} \geq 0\right) \odot \vec{v}$ and $\vec{v}^{-} = \mathbb{I}\left(\vec{v} < 0\right) \odot \vec{v}$. 
\end{proof}

\section{On the Complexity of Bounding using \partialcrown}
\label{app:complexity}

The complexity $\mathcal{M}$ of bounding $f_\theta$ (or any function of the partial derivatives of $u_\theta$) is contingent on the type of PDE we are bounding. 

For simplicity, assume the solution network, $u_\theta$, has $L$ fully connected hidden layers each with $d$ output neurons, and that the relaxation of the activation functions and their derivatives, i.e., $\sigma$, $\sigma’$, …, can be computed in $\mathcal{O}(1)$. Using CROWN we can bound the output of layer $l \in \{1,\dots,L\}$ in $\mathcal{O}(l d^2)$, yielding the complexity of bounding the output of $u_\theta$ as $\mathcal{O}(L^2 d^2)$. With our hybrid scheme of backward propagation within the bounding component ($\partial_{\mathbf{x}\_i} u_\theta$ or $\partial_{\mathbf{x}^2\_i} u_\theta$) and forward substitution for elements from other components (e.g., $y^{(k)}$ in the bounding of $\partial_{\mathbf{x}\_i} u_\theta$, see Equation \ref{eq:partial_hat_z_k_1_hat_z_substituted} in Appendix \ref{app:proof_bounding_first_derivative}), the complexity of bounding the output of each of these components remains $\mathcal{O}(L^2 d^2)$. Following the McCormick envelope bounding described in Section \ref{sec:multiplicative_terms}, to estimate the final complexity we must now assume that the particular structure of $f_\theta$ will be linear lower and upper bounded as a function of $R$ partial derivative components (of first or second order). For example, in the case of Burgers’ equation, $R=3$. The final complexity of bounding $f_\theta$ can then be written as $\mathcal{O}(R L^2 d^2)$.

\section{Correctness Certification for PINNs with \texorpdfstring{$\tanh$}{tanh} activations}
\label{app:tanh_relaxations}

\partialcrown allows one to compute lower and upper bounds on the outputs of $\partial_{\x_i} \solnet$, $\partial_{\x_i^2} \solnet$ and $\resnet$ as long as we can obtain linear bounds for $\solnet$'s activations, $\sigma$, $\partial_{\x_i} \solnet$'s activations, $\sigma'$, and $\partial_{\x_i^2} \solnet$'s activations, $\sigma''$, assuming previously computed bounds on the input of those activations. In this section we explore how to compute those bounds when $\solnet$ has $\tanh$ activations.

Throughout, we assume the activation's input ($y$) is lower bounded by $l_b$ and upper bounded by $u_b$ (\ie,  $l_b \leq y \leq u_b$), and define the upper bound line as $h^U(y) = \alpha^U(y + \beta^U)$, and the lower bound line as $h^L(y) = \alpha^L(y + \beta^L)$. For the sake of brevity, we define for a function $h: \mathbb{R} \to \mathbb{R}$, and points $p, d \in \mathbb{R}$ the function $\tau(h, p, d) = \nicefrac{(h(p) - h(d))}{(p - d)} - h'(d)$. This is useful as for a given $h$ and $p$, if there exists a $d\in[d_l, d_u]$, such that $\tau_{d_l, d_u}(h, p, d) = 0$, then $h'(d)$ is the slope of a tangent line to $h$ that passes through $p$ and $d$.

\paragraph{Bounding $\sigma(y) = \tanh(y)$} We follow the bounds provided in CROWN \citep{zhang2018efficient}, by observing that $\tanh$ is a convex function for $y < 0$ and concave for $y > 0$. For $l_b \leq u_b \leq 0$ we let $h^U$ be the line that connects $l_b$ and $u_b$, and for an arbitrary $d\in[l_b, u_b]$ we let $h^L$ be the tangent line at that point. Similarly, for $0 \leq l_b \leq u_b$ we let $h^L$ be the line that connects $l_b$ and $u_b$, and for an arbitrary $d\in[l_b, u_b]$ we let $h^U$ be the tangent line at that point. For the last case where $l_b \leq 0 \leq u_b$, we let $h^U$ be the tangent line at $d_1 \geq 0$ that passes through $(l_b, \sigma(l_b))$, and $h^L$ be the tangent line at $d_2 \leq 0$ that passes through $(u_b, \sigma(u_b))$. Given these bounds were given in \citet{zhang2018efficient}, we omit visual representations of them.

\begin{table*}[t]
    \caption{\textbf{Relaxing $\sigma'(y) = 1-\tanh^2(y)$}: linear upper and lower bounds for a given $l_b$ and $u_b$.}
    \label{tab:sigma_prime_relaxations}
    \center
    {
    \footnotesize
    \begin{tabular}{>{\centering\arraybackslash}m{0.025\textwidth}>{\centering\arraybackslash}m{0.025\textwidth}>{\centering\arraybackslash}>{\centering\arraybackslash}m{0.20\textwidth}>{\centering\arraybackslash}m{0.20\textwidth}>{\centering\arraybackslash}m{0.20\textwidth}>{\centering\arraybackslash}m{0.20\textwidth}}
        \hline
        $l_b$ & $u_b$ & $\alpha^U$ & $\beta^U$ & $\alpha^L$ & $\beta^L$\\\hline\hline
        $\mathcal{R}_1$ & $\mathcal{R}_1$ & \multirow{2}{*}{$\nicefrac{(\sigma(u_b) - \sigma(l_b))}{(u_b - l_b)}$} & \multirow{2}{*}{$\nicefrac{\sigma(l_b)}{\alpha^U} - l_b$} & \multirow{2}{*}{$\sigma'(d)$, $d\in[l_b, u_b]$} & \multirow{2}{*}{$\nicefrac{\sigma(d)}{\alpha^L} - d$} \\
        $\mathcal{R}_3$ & $\mathcal{R}_3$ & & & & \\\hline
        $\mathcal{R}_2$ & $\mathcal{R}_2$ & $\sigma'(d)$, $d\in[l_b, u_b]$ & $\nicefrac{\sigma(d)}{\alpha^U} - d$ & $\nicefrac{(\sigma(u_b) - \sigma(l_b))}{(u_b - l_b)}$ & $\nicefrac{\sigma(l_b)}{\alpha^L} - l_b$\\\hline
        $\mathcal{R}_1$ & $\mathcal{R}_2$ & $\sigma'(d_1)$, $\tau_{y_1, u_b}(\sigma', l_b, d_1) = 0$ & $\nicefrac{\sigma(l_b)}{\alpha^U} - l_b$ & $\sigma'(d_2)$, $\tau_{l_b, y_1}(\sigma', u_b, d_2) = 0$ & $\nicefrac{\sigma(u_b)}{\alpha^L} - u_b$\\\hline
        $\mathcal{R}_2$ & $\mathcal{R}_3$ & $\sigma'(d_1)$, $\tau_{l_b, y_2}(\sigma', u_b, d_1) = 0$ & $\nicefrac{\sigma(u_b)}{\alpha^U} - u_b$ & $\sigma'(d_2)$, $\tau_{y_2, u_b}(\sigma', l_b, d_2) = 0$ & $\nicefrac{\sigma(l_b)}{\alpha^L} - l_b$\\\hline
        $\mathcal{R}_1$ & $\mathcal{R}_3$ & $\alpha\sigma'(d_1) + (1-\alpha)\sigma'(d_2)$, $\tau_{l_b, 0}(\sigma', l_b, d_1) = 0$, $\tau_{0, u_b}(\sigma', u_b, d_2) = 0$ & $\alpha \beta^U_1 + (1-\alpha) \beta^U_2$,\newline $\beta^U_1 = \nicefrac{\sigma(l_b)}{\sigma'(d_1)} - l_b$, \newline $\beta^U_2 = \nicefrac{\sigma(u_b)}{\sigma'(d_2)} - u_b$ & $\begin{cases}
        \sigma'(d_3)\text{, } -l_b \geq u_b\\
        \sigma'(d_4)\text{, } -l_b < u_b
        \end{cases}$,\newline $\tau_{l_b, y_1}(\sigma', u_b, d_3) = 0$, $\tau_{y_2, u_b}(\sigma', l_b, d_4) = 0$ & $\begin{cases}
        \frac{\sigma'(u_b)}{\sigma'(d_3)} - u_b\text{, } -l_b \geq u_b\\
        \frac{\sigma'(l_b)}{\sigma'(d_4)} - l_b\text{, } -l_b < u_b
        \end{cases}$\\
        \hline
    \end{tabular}
    }
\end{table*}

\begin{figure*}
    \begin{subfigure}[b]{0.32\textwidth}
        \centering
        \includegraphics[width=\textwidth]{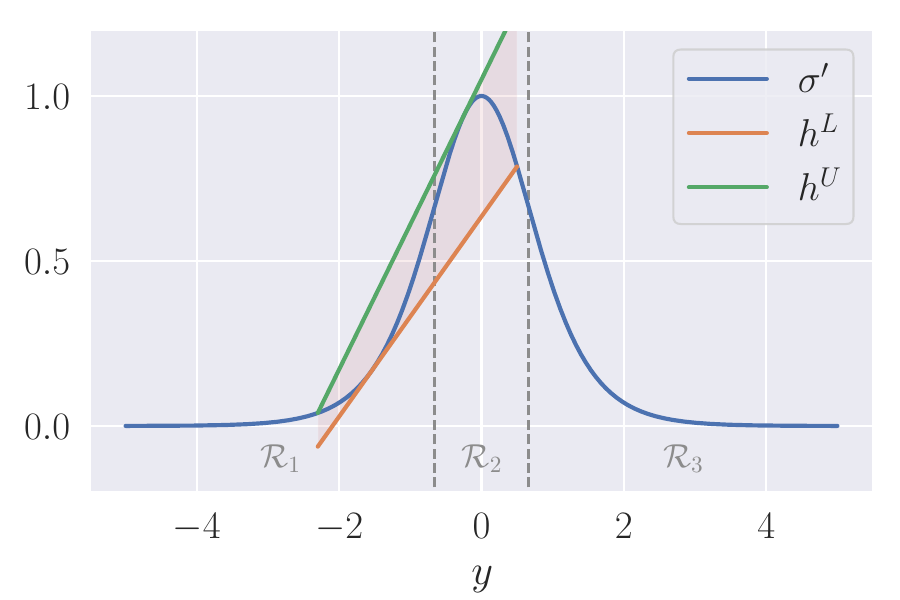}
        \caption{$l_b \in \mathcal{R}_1$ and $u_b \in \mathcal{R}_2$}
        \label{fig:tanh_first_derivative_1}
    \end{subfigure}
    \hfill
    \begin{subfigure}[b]{0.32\textwidth}
        \centering
        \includegraphics[width=\textwidth]{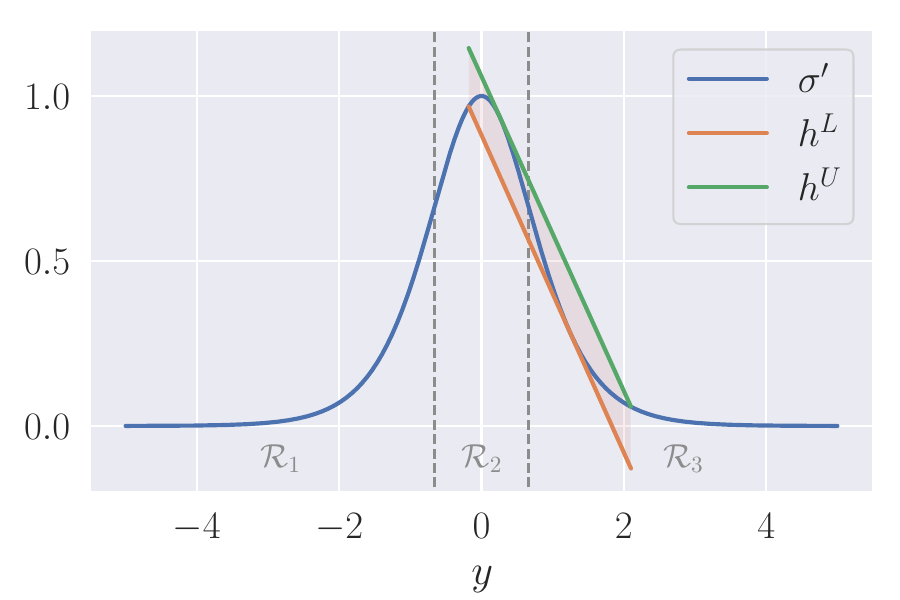}
        \caption{$l_b \in \mathcal{R}_2$ and $u_b \in \mathcal{R}_3$}
        \label{fig:tanh_first_derivative_2}
    \end{subfigure}
    \hfill
    \begin{subfigure}[b]{0.32\textwidth}
        \centering
        \includegraphics[width=\textwidth]{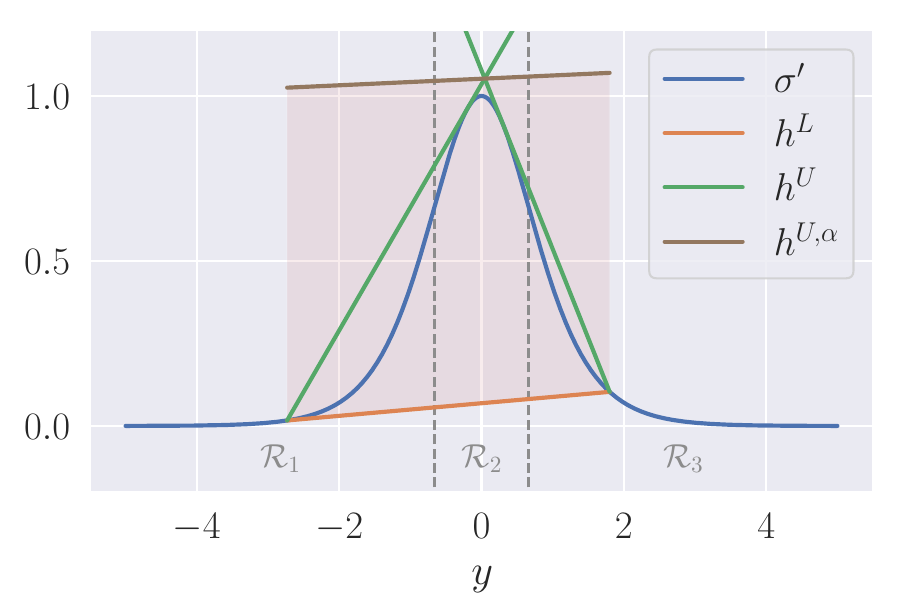}
        \caption{$l_b \in \mathcal{R}_1$ and $u_b \in \mathcal{R}_3$}
        \label{fig:tanh_first_derivative_3}
    \end{subfigure}
    \caption{\textbf{Relaxing $\sigma'(y) = 1-\tanh^2(y)$}: examples of the linear relaxations of $\sigma'$ for different sets of $l_b$ and $u_b$.}
    \label{fig:tanh_first_derivative_relaxations}
\end{figure*}

\paragraph{Bounding $\sigma'(y) = 1 - \tanh^2(y)$} The derivative of $\tanh(y)$, $1 - \tanh^2(y)$, is a more complicated function. By inspecting it's derivative, $\sigma''(y) = -2\tanh(y) (1 - \tanh^2(y))$, we conclude that there are two inflection points at $y_1 = \max \sigma''(y)$ and $y_2 = \min \sigma''(y)$, leading to three different regions: $y \in ]-\infty, y_1]$ ($\mathcal{R}_1$, the first convex region), $y\in ]y_1, y_2]$ ($\mathcal{R}_2$, the concave region), and $y\in ]y_2, +\infty[$ ($\mathcal{R}_3$, the second convex region). As a result, there are 6 combinations for the location of $l_b$ and $u_b$ which must be resolved.

The first two cases are the straightforward: if $l_b \in \mathcal{R}_1$ and $u_b \in \mathcal{R}_1$ or $l_b \in \mathcal{R}_3$ and $u_b \in \mathcal{R}_3$, \ie, if both ends are in the same convex region, then we use the same relaxation as in the bounding of $\tanh$ in the convex region - $h^U$ is the line that connects $l_b$ and $u_b$, while $h^L$ is a tangent line at a point $d \in [l_b, u_b]$. Similarly for the case where $l_b \in \mathcal{R}_2$ and $u_b \in \mathcal{R}_2$, we take the solution from the $\tanh$ concave side and use $h^L$ to be the line that connects $l_b$ and $u_b$, and $h^U$ to be the tangent line at a point $d \in [l_b, u_b]$.
The next case is $l_b \in \mathcal{R}_1$ and $u_b \in \mathcal{R}_2$, \ie, $l_b$ in the first convex region and $u_b$ in the concave one. In this case we use the same bounding as in the $\tanh$ case when $l_b \leq 0 \leq u_b$: $h^U$ is the tangent line at $d_1 \geq y_1$ that passes through $(l_b, \sigma'(l_b))$, and $h^L$ is the tangent line at $d_2 \leq y_1$ that passes through $(u_b, \sigma'(u_b))$. In a similar fashion, for the case in which $l_b \in \mathcal{R}_2$ and $u_b \in \mathcal{R}_3$, \ie, $l_b$ in the concave region and $u_b$ in the second convex region, we take the opposite approach: $h^U$ is the tangent line at $d_1 \leq y_2$ that passes through $(u_b, \sigma'(u_b))$, and $h^L$ is the tangent line at $d_2 \geq y_2$ that passes through $(l_b, \sigma'(l_b))$. These two cases are plotted in Figures \ref{fig:tanh_first_derivative_1} and \ref{fig:tanh_first_derivative_2}.

Finally, we tackle the case where $l_b \in \mathcal{R}_1$ and $u_b \in \mathcal{R}_3$, \ie, where $l_b$ is in the first convex region and $u_b$ is in the second convex region. Given there is a concave region in between them, two valid upper bounds would be the ones considered previously for $l_b \in \mathcal{R}_1$ and $u_b \in \mathcal{R}_2$, and $l_b \in \mathcal{R}_2$ and $u_b \in \mathcal{R}_3$. To obtain these bounds, we shift the upper bound in the first case to $0$, and the lower bound in the second case to $0$ (see $h^U$ in Figure \ref{fig:tanh_first_derivative_3}). As our bounding requires a single $h^U$, we take a convex combination of the two bounds obtained, $h^{U,\alpha}$. For the lower bound, we use a line that passes by either $(u_b, \sigma'(u_b))$, if $-l_b \geq u_b$, or by $(l_b, \sigma'(l_b))$, otherwise, as well as by a tangent point $d_3 \in \mathcal{R}_1$, if $-l_b \geq u_b$, or by $d_4 \in \mathcal{R}_3$, otherwise. See the line $h^{U, \alpha}$ in Figure \ref{fig:tanh_first_derivative_3} for a visual representation. 

\begin{table*}[t]
    \caption{\textbf{Relaxing $\sigma''(y) = -2\tanh(y)\left(1-\tanh^2(y)\right)$}: linear upper and lower bounds for a given $l_b$ and $u_b$.}
    \label{tab:sigma_prime_prime_relaxations}
    \center
    {
    \footnotesize
    \begin{tabular}{>{\centering\arraybackslash}m{0.025\textwidth}>{\centering\arraybackslash}m{0.025\textwidth}>{\centering\arraybackslash}>{\centering\arraybackslash}m{0.20\textwidth}>{\centering\arraybackslash}m{0.20\textwidth}>{\centering\arraybackslash}m{0.20\textwidth}>{\centering\arraybackslash}m{0.20\textwidth}}
        \hline
        $l_b$ & $u_b$ & $\alpha^U$ & $\beta^U$ & $\alpha^L$ & $\beta^L$\\\hline\hline
        $\mathcal{R}_1$ & $\mathcal{R}_1$ & \multirow{2}{*}{$\nicefrac{(\sigma''(u_b) - \sigma''(l_b))}{(u_b - l_b)}$} & \multirow{2}{*}{$\nicefrac{\sigma''(l_b)}{\alpha^U} - l_b$} & \multirow{2}{*}{$\sigma'''(d)$, $d\in[l_b, u_b]$} & \multirow{2}{*}{$\nicefrac{\sigma''(d)}{\alpha^L} - d$} \\
        $\mathcal{R}_3$ & $\mathcal{R}_3$ & & & & \\\hline
        $\mathcal{R}_2$ & $\mathcal{R}_2$ & \multirow{2}{*}{$\sigma'''(d)$, $d\in[l_b, u_b]$} & \multirow{2}{*}{$\nicefrac{\sigma''(d)}{\alpha^U} - d$} & \multirow{2}{*}{$\nicefrac{(\sigma''(u_b) - \sigma''(l_b))}{(u_b - l_b)}$} & \multirow{2}{*}{$\nicefrac{\sigma''(l_b)}{\alpha^L} - l_b$}\\
        $\mathcal{R}_4$ & $\mathcal{R}_4$ & & & & \\\hline
        %
        %
        $\mathcal{R}_1$ & $\mathcal{R}_2$ & $\sigma'''(d_1)$, $\tau_{y_1, u_b}(\sigma'', l_b, d_1) = 0$ & $\nicefrac{\sigma(l_b)}{\alpha^U} - l_b$ & $\sigma'''(d_2)$, $\tau_{l_b, y_1}(\sigma'', u_b, d_2) = 0$ & $\nicefrac{\sigma''(u_b)}{\alpha^L} - u_b$\\\hline
        $\mathcal{R}_3$ & $\mathcal{R}_4$ & $\sigma'''(d_1)$, $\tau_{y_3, u_b}(\sigma'', l_b, d_1) = 0$ & $\nicefrac{\sigma''(l_b)}{\alpha^U} - l_b$ & $\sigma'''(d_2)$, $\tau_{l_b, y_3}(\sigma'', u_b, d_2) = 0$ & $\nicefrac{\sigma''(u_b)}{\alpha^L} - u_b$\\\hline
        %
        %
        $\mathcal{R}_2$ & $\mathcal{R}_3$ & $\sigma'''(d_1)$, $\tau_{l_b, y_2}(\sigma'', u_b, d_1) = 0$ & $\nicefrac{\sigma''(u_b)}{\alpha^U} - u_b$ & $\sigma'''(d_2)$, $\tau_{y_2, u_b}(\sigma'', l_b, d_2) = 0$ & $\nicefrac{\sigma''(l_b)}{\alpha^L} - l_b$\\\hline
        %
        %
        $\mathcal{R}_1$ & $\mathcal{R}_3$ & $\alpha\sigma'''(d_1) + (1-\alpha)\sigma'''(d_2)$, $\tau_{l_b, y_{\max}}(\sigma'', l_b, d_1) = 0$, $\tau_{y_{\max}, u_b}(\sigma'', u_b, d_2) = 0$ & $\alpha \beta^U_1 + (1-\alpha) \beta^U_2$,\newline $\beta^U_1 = \nicefrac{\sigma''(l_b)}{\sigma'''(d_1)} - l_b$, \newline $\beta^U_2 = \nicefrac{\sigma''(u_b)}{\sigma'''(d_2)} - u_b$ & $\sigma'''(d_3)$,\newline $\tau_{y_1, u_b}(\sigma', l_b, d_3) = 0$ & $\nicefrac{\sigma''(l_b)}{\alpha^L} - l_b$\\\hline
        %
        %
        $\mathcal{R}_2$ & $\mathcal{R}_4$ & $\sigma'''(d_1)$,\newline $\tau_{l_b, y_2}(\sigma', u_b, d_1) = 0$ & $\nicefrac{\sigma''(u_b)}{\alpha^U} - u_b$ & $\alpha\sigma'''(d_2) + (1-\alpha)\sigma'''(d_3)$, $\tau_{l_b, y_{\min}}(\sigma'', l_b, d_2) = 0$, $\tau_{y_{\min}, u_b}(\sigma'', u_b, d_3) = 0$ & $\alpha \beta^L_1 + (1-\alpha) \beta^L_2$,\newline $\beta^L_1 = \nicefrac{\sigma''(l_b)}{\sigma'''(d_2)} - l_b$, \newline $\beta^L_2 = \nicefrac{\sigma''(u_b)}{\sigma'''(d_3)} - u_b$\\\hline
        %
        %
        $\mathcal{R}_1$ & $\mathcal{R}_4$ & $\alpha\sigma'''(d_1) + (1-\alpha)\sigma'''(d_2)$, $\tau_{l_b, y_{\max}}(\sigma'', l_b, d_1) = 0$, $\tau_{y_{\max}, u_b}(\sigma'', u_b, d_2) = 0$ & $\alpha \beta^U_1 + (1-\alpha) \beta^U_2$,\newline $\beta^U_1 = \nicefrac{\sigma''(l_b)}{\sigma'''(d_1)} - l_b$, \newline $\beta^U_2 = \nicefrac{\sigma''(u_b)}{\sigma'''(d_2)} - u_b$ & $\alpha\sigma'''(d_3) + (1-\alpha)\sigma'''(d_4)$, $\tau_{l_b, y_{\min}}(\sigma'', l_b, d_3) = 0$, $\tau_{y_{\min}, u_b}(\sigma'', u_b, d_4) = 0$ & $\alpha \beta^L_1 + (1-\alpha) \beta^L_2$,\newline $\beta^L_1 = \nicefrac{\sigma''(l_b)}{\sigma'''(d_3)} - l_b$, \newline $\beta^L_2 = \nicefrac{\sigma''(u_b)}{\sigma'''(d_4)} - u_b$\\
        \hline
    \end{tabular}
    }
\end{table*}

\begin{figure*}
    \begin{subfigure}[b]{0.32\textwidth}
        \centering
        \includegraphics[width=\textwidth]{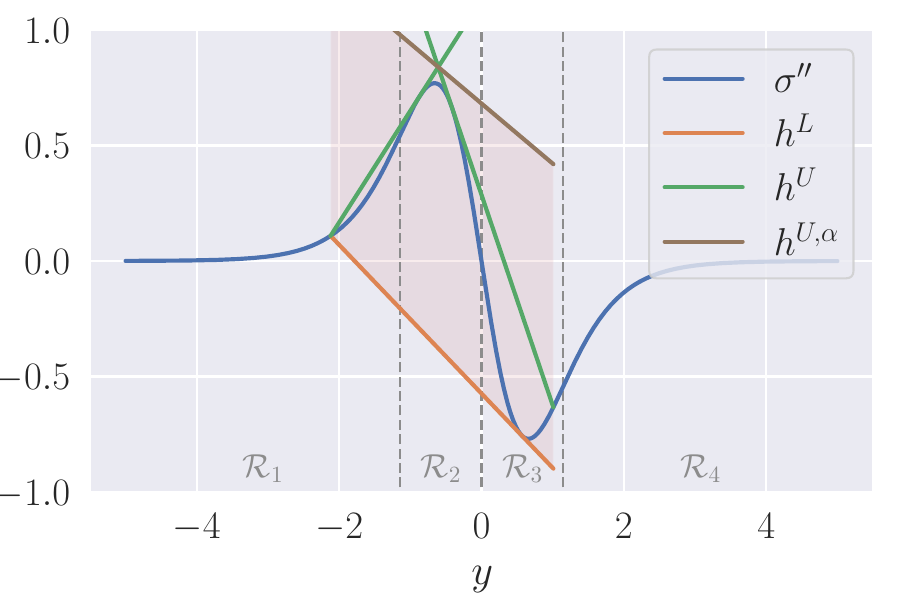}
        \caption{$l_b \in \mathcal{R}_1$ and $u_b \in \mathcal{R}_3$}
        \label{fig:tanh_second_derivative_1}
    \end{subfigure}
    \hfill
    \begin{subfigure}[b]{0.32\textwidth}
        \centering
        \includegraphics[width=\textwidth]{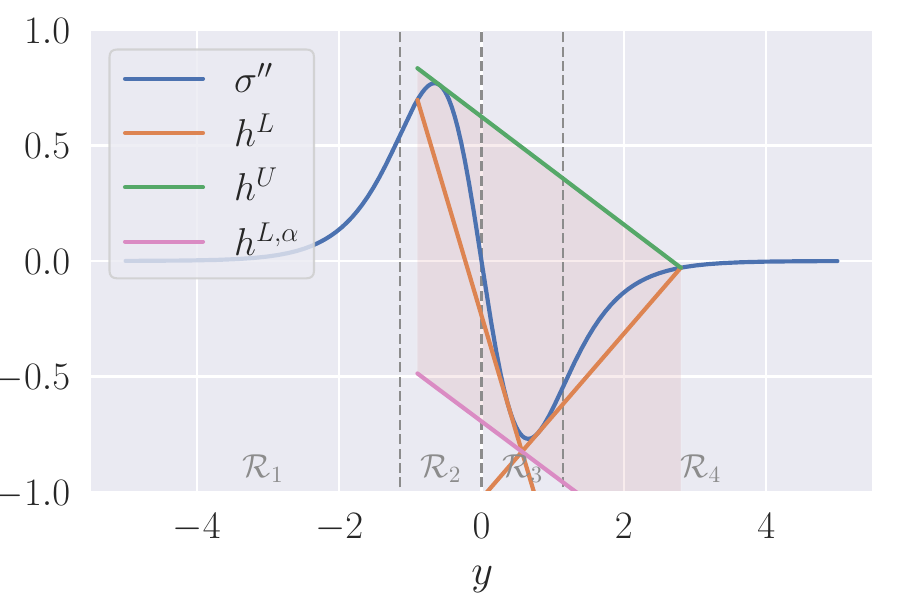}
        \caption{$l_b \in \mathcal{R}_2$ and $u_b \in \mathcal{R}_4$}
        \label{fig:tanh_second_derivative_2}
    \end{subfigure}
    \hfill
    \begin{subfigure}[b]{0.32\textwidth}
        \centering
        \includegraphics[width=\textwidth]{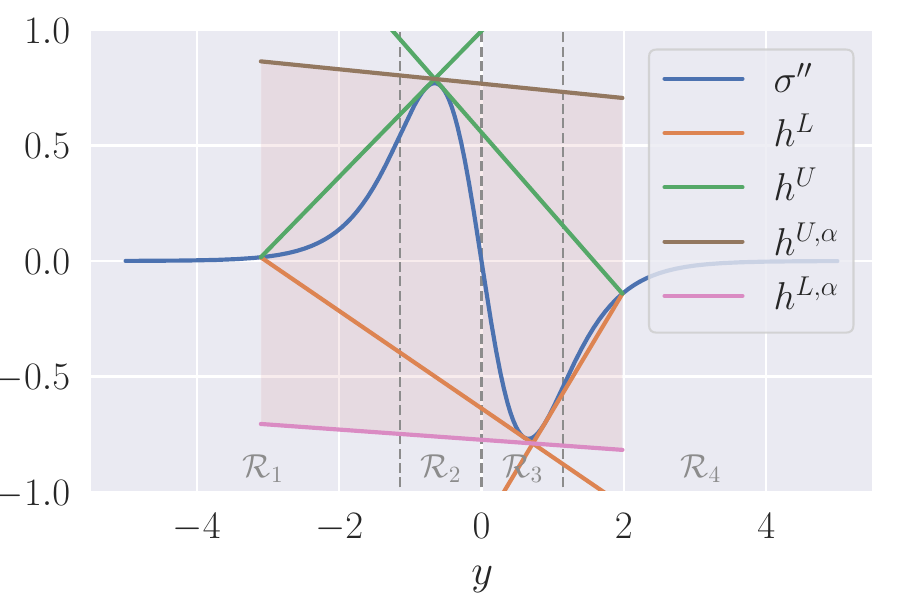}
        \caption{$l_b \in \mathcal{R}_1$ and $u_b \in \mathcal{R}_4$}
        \label{fig:tanh_second_derivative_3}
    \end{subfigure}
    \caption{\textbf{Relaxing $\sigma''(y) = -2\tanh(y)\left(1-\tanh^2(y)\right)$}: examples of the linear relaxations of $\sigma''$ for different sets of $l_b$ and $u_b$.}
    \label{fig:tanh_second_derivative_relaxations}
\end{figure*}

\paragraph{Bounding $\sigma''(y) = -2 \tanh(y) \left(1-\tanh(y)^2\right)$} By inspecting the derivative of $\sigma''$, $\sigma''(y) = -2 + 8 \tanh^2(y) - 6 \tanh^4(y)$, we conclude there are three inflection points for this function, one at $y_1 = \argmax_{y \leq 0} \sigma'''(y)$, another at $y_2 = 0$, and finally at $y_3 = -y_1$. Take also, for the sake of bounding, $y_{\max} = \argmax_{y \leq 0} \sigma''(y)$ and $y_{\min} = \argmin_{y \leq 0} \sigma''(y)$. This leads to four different regions of $\sigma''$: $y \in ]-\infty, y_1]$ ($\mathcal{R}_1$, the first convex region), $y\in ]y_1, y_2]$ ($\mathcal{R}_2$, the first concave region), $y\in ]y_2, y_3]$ ($\mathcal{R}_3$, the second convex region), and $y\in ]y_3, +\infty[$ ($\mathcal{R}_4$, the second concave region). This leads to 10 combinations for the location of $l_b$ and $u_b$.

The first four are straightforward: if $l_b \in \mathcal{R}_i$ and $u_b \in \mathcal{R}_i$ for $i\in\{1,\dots,4\}$, then we use exactly the same approximations as for $\sigma$ and $\sigma''$, varying only based on the convexity of $\mathcal{R}_i$. Similarly, if $l_b \in \mathcal{R}_i$ and $u_b \in \mathcal{R}_{i+1}$ for $i\in\{1,2,3\}$, then we are also in the same situation as the adjacent regions of different convexity from $\sigma'$, so we use exactly the same relaxation.

We are left with three cases where $l_b$ and $u_b$ are in non-adjacent regions. For $l_b \in \mathcal{R}_1$ and $u_b \in \mathcal{R}_3$, we are in the same scenario as in the bounding of $\sigma'$, since $\mathcal{R}_1$ and $\mathcal{R}_3$ are convex regions separated by a concave one. In that case we follow the bounding procedure outlined before for $\sigma'$ - see Figure \ref{fig:tanh_second_derivative_1} for an example of it applied in this setting. For the case where $l_b \in \mathcal{R}_2$ and $u_b \in \mathcal{R}_4$, we are in an analogous case where $\mathcal{R}_2$ and $\mathcal{R}_4$ are concave regions separated by a convex one. As such, we consider the two valid lower bounds computed previously for $l_b \in \mathcal{R}_2$ and $u_b \in \mathcal{R}_3$, and $l_b \in \mathcal{R}_3$ and $u_b \in \mathcal{R}_4$. 
\begin{wraptable}[12]{r}{0.5\textwidth}
    \caption{\textbf{Ablation on our relaxations for the derivatives of $\tanh$}: comparison of the residual upper bounds on Burgers' equation obtained using our relaxations in \partialcrown vs using a simple baseline which takes the minimum area in the convex/concave regions and a constant value elsewhere with a time limit of $10^4$s.}
    \label{tab:tanh_relaxation_efficiency}
    \vspace{-1.5em}
    \center
    {
        \footnotesize
        \begin{tabular}{>{\centering\arraybackslash}m{0.14\linewidth}>{\centering\arraybackslash}m{0.35\linewidth}>{\centering\arraybackslash}>{\centering\arraybackslash}m{0.35\linewidth}}
        & \textbf{Our relaxations $u_b$}    & \textbf{Simple baseline $u_b$}\\\midrule
        $|\resnet(\x)|^2$ & $1.30 \times 10^1$ & $4.34 \times 10^2$\\
        \bottomrule
        \end{tabular}
    }
\end{wraptable}
To obtain these bounds, we shift the upper bound in the first case to $\argmin \sigma''(y)$, and the lower bound in the same case to the same value (see $h^L$ in Figure \ref{fig:tanh_second_derivative_2}). As our bounding requires a single $h^L$, we take a convex combination of the two bounds, $h^{L,\alpha}$. For the upper bound, we simply assume $l_b$ is in a concave region while $u_b$ is in a convex region, and take the tangent at $d$ for $\argmax \sigma''(y) \geq d\leq 0$ (see $h^U$ in Figure \ref{fig:tanh_second_derivative_2}). Finally, we are left with the case where $l_b \in \mathcal{R}_1$ and $u_b \in \mathcal{R}_4$. In that case, we take the upper bound lines from the case where $l_b \in \mathcal{R}_1$ and $u_b \in \mathcal{R}_3$, and the lower bound ones from where $l_b \in \mathcal{R}_2$ and $u_b \in \mathcal{R}_4$. As before, given the requirement of one lower and upper bound functions, we take a convex combination of both in $h^{L,\alpha}$ and $h^{U,\alpha}$, respectively. See Figure \ref{fig:tanh_second_derivative_3} for a visual representation.

\subsection{Ablation on $\sigma'$ and $\sigma''$ relaxations for $\tanh$}
\label{app:tanh_relaxations_ablation}

To understand the effectiveness of proposed the proposed relaxations for $\sigma''$ and $\sigma''$ for the case of $\tanh$, we compare the performance of \partialcrown in bounding the residual of Burgers’ equation (with the fixed time limit of $10^4$s from Section \ref{sec:efficiency-experiments}), using \textbf{Our relaxations} for $\sigma'$ and $\sigma''$, as well as a \textbf{Simple baseline} which takes the minimum area in the convex/concave sections and a constant elsewhere. For $\tanh$, both use the same relaxation from \citet{zhang2018efficient}. The comparison results are presented in the Table \ref{tab:tanh_relaxation_efficiency}. The tightness difference showcases the efficacy of our proposed nonlinearity relaxations for $\sigma'$ and $\sigma''$ for $\tanh$ activations.

\section{Linear lower and upper bounding nonlinear functions}

Throughout, we assume the function's input ($x$) is lower bounded by $l_b$ and upper bounded by $u_b$ (\ie,  $l_b \leq x \leq u_b$), and define the upper bound line as $h^U(x) = \alpha^U(x + \beta^U)$, and the lower bound line as $h^L(x) = \alpha^L(x + \beta^L)$. For the sake of brevity, we define for a function $h: \mathbb{R} \to \mathbb{R}$, and points $p, d \in \mathbb{R}$ the function $\tau(h, p, d) = \nicefrac{(h(p) - h(d))}{(p - d)} - h'(d)$. This is useful as for a given $h$ and $p$, if there exists a $d\in[d_l, d_u]$, such that $\tau_{d_l, d_u}(h, p, d) = 0$, then $h'(d)$ is the slope of a tangent line to $h$ that passes through $p$ and $d$.

\subsection{Case study: \texorpdfstring{$-\text{sin}(\pi x)$}{-sin(pi x)} for \texorpdfstring{$x\in [-1, 1]$}{x in [-1, 1]}}
\label{app:sin_relaxation}

\begin{figure*}
    \begin{subfigure}[b]{0.32\textwidth}
        \centering
        \includegraphics[width=\textwidth]{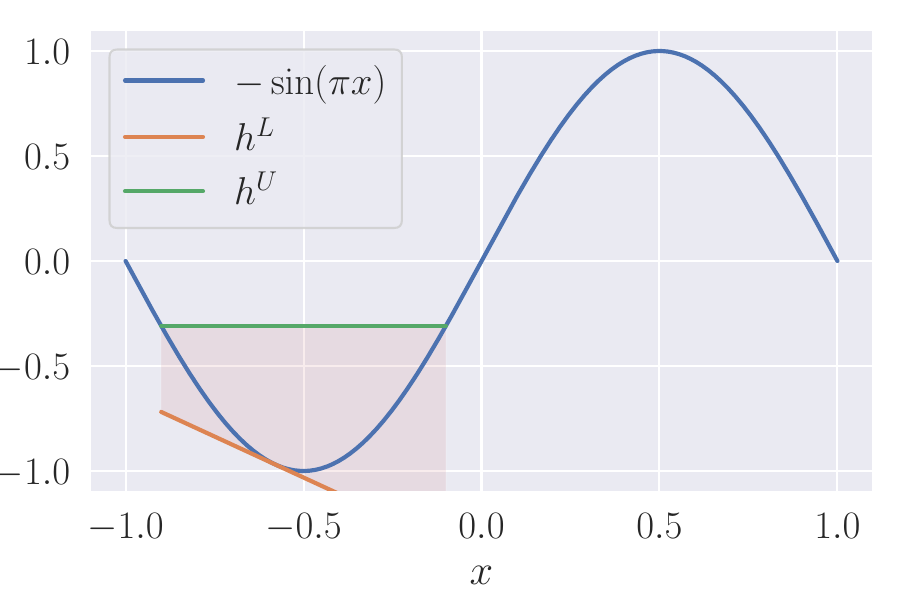}
        \caption{$l_b \leq 0$ and $u_b \leq 0$}
        \label{fig:sin_1}
    \end{subfigure}
    \hfill
    \begin{subfigure}[b]{0.32\textwidth}
        \centering
        \includegraphics[width=\textwidth]{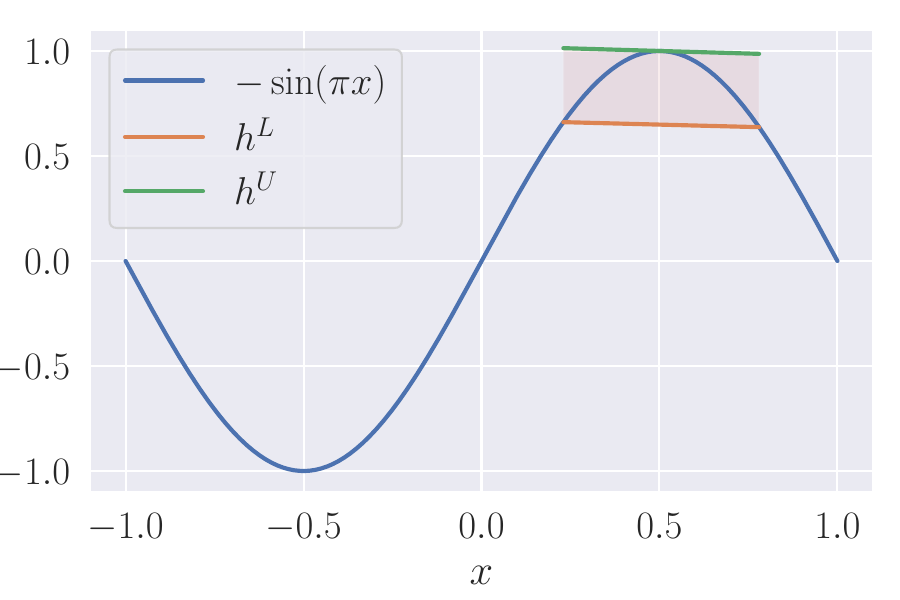}
        \caption{$l_b \geq 0$ and $u_b \geq 0$}
        \label{fig:sin_2}
    \end{subfigure}
    \hfill
    \begin{subfigure}[b]{0.32\textwidth}
        \centering
        \includegraphics[width=\textwidth]{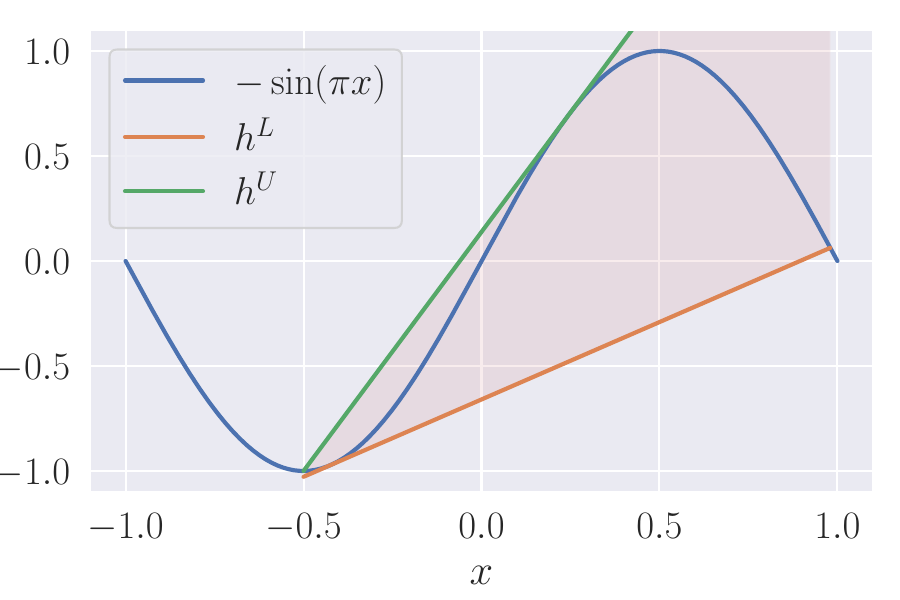}
        \caption{$l_b \leq 0$ and $u_b \geq 0$}
        \label{fig:sin_3}
    \end{subfigure}
    \caption{\textbf{Relaxing $-\sin(\pi x)$}: examples of the linear relaxations for different sets of $l_b$ and $u_b$.}
    \label{fig:sin_relaxation}
\end{figure*}

\begin{figure*}
    \begin{subfigure}[b]{0.32\textwidth}
        \centering
        \includegraphics[width=\textwidth]{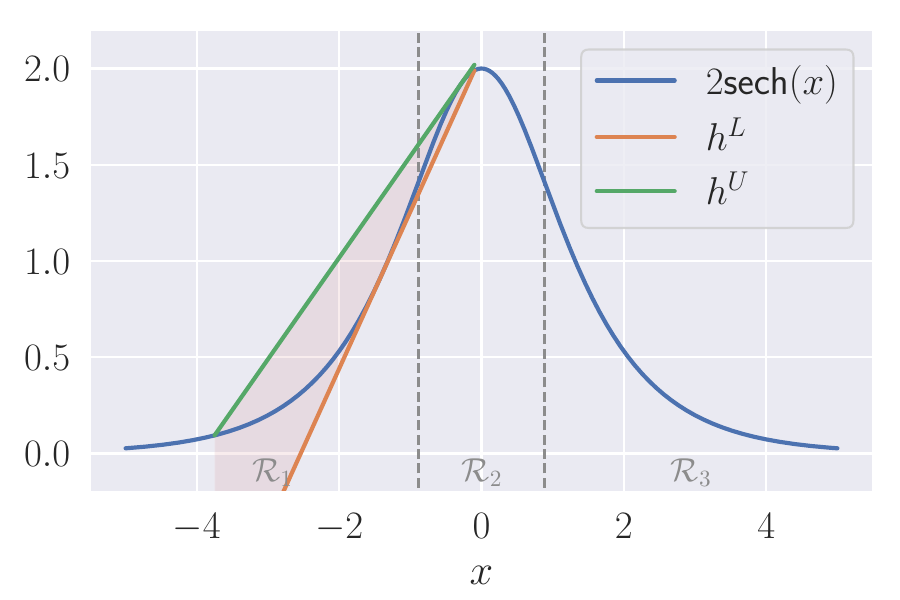}
        \caption{$l_b \leq 0$ and $u_b \leq 0$}
        \label{fig:sech_1}
    \end{subfigure}
    \hfill
    \begin{subfigure}[b]{0.32\textwidth}
        \centering
        \includegraphics[width=\textwidth]{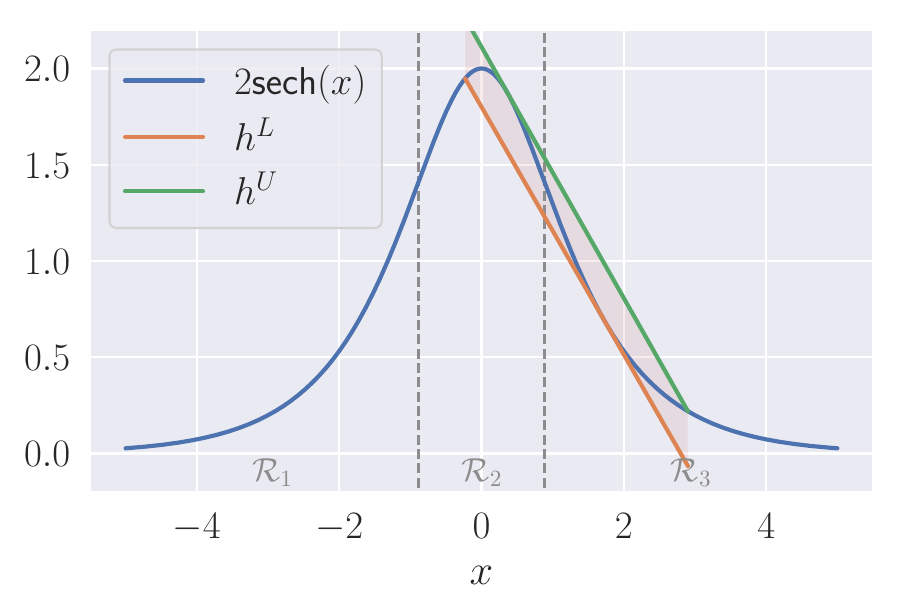}
        \caption{$l_b \geq 0$ and $u_b \geq 0$}
        \label{fig:sech_2}
    \end{subfigure}
    \hfill
    \begin{subfigure}[b]{0.32\textwidth}
        \centering
        \includegraphics[width=\textwidth]{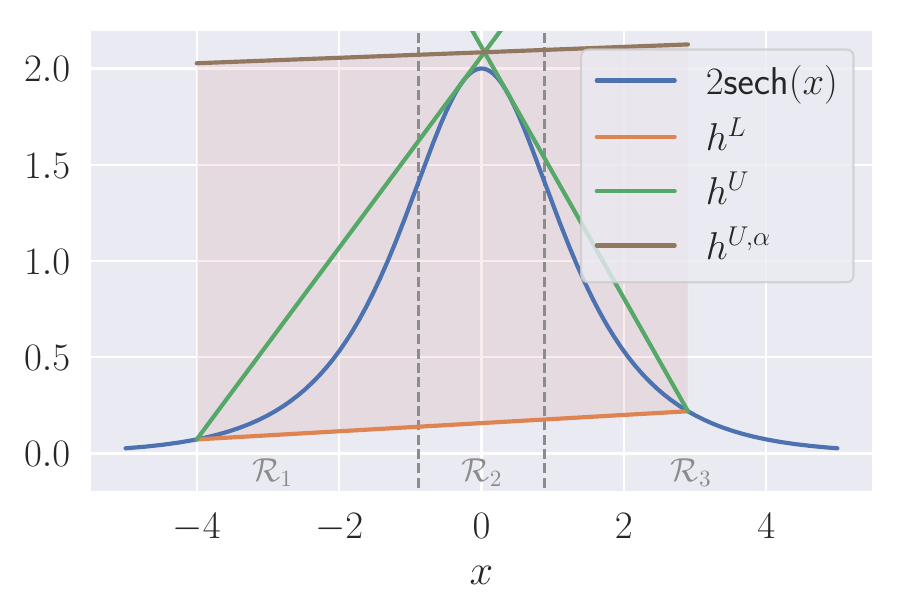}
        \caption{$l_b \leq 0$ and $u_b \geq 0$}
        \label{fig:sech_3}
    \end{subfigure}
    \caption{\textbf{Relaxing $2\sech(x)$}: examples of the linear relaxations for different sets of $l_b$ and $u_b$.}
    \label{fig:sech_relaxation}
\end{figure*}

As in Appendix~\ref{app:tanh_relaxations}, we observe the convexity of the function $-\sin(\pi x)$ for $x\in[-1, 1]$, noticing that the function is convex for $x \leq 0$ and concave for $x \geq 0$. For $l_b \leq u_b \leq 0$ we let $h^U$ be the line that connects $l_b$ and $u_b$, and for an arbitrary $d\in[l_b, u_b]$ we let $h^L$ be the tangent line at that point. Similarly, for $0 \leq l_b \leq u_b$ we let $h^L$ be the line that connects $l_b$ and $u_b$, and for an arbitrary $d\in[l_b, u_b]$ we let $h^U$ be the tangent line at that point. For the last case where $l_b \leq 0 \leq u_b$, we let $h^U$ be the tangent line at $d_1 \geq 0$ that passes through $(l_b, \sigma(l_b))$, and $h^L$ be the tangent line at $d_2 \leq 0$ that passes through $(u_b, \sigma(u_b))$. Given the similarity of to the $\tanh$ bounds from~\citet{zhang2018efficient}, we omit a summary table, but present 3 examples of the possible cases in Figure~\ref{fig:sin_relaxation}.

\subsection{Case study: \texorpdfstring{$2\text{sech}(x)$}{2 sech(x)} for \texorpdfstring{$x\in [-5, 5]$}{x in [-5, 5]}}
\label{app:sech_relaxation}

We start by observing that the function $2\sech(x)$ is similar to the derivative of $\tanh$, whose relaxation we presented in Appendix~\ref{app:tanh_relaxations}. By inspecting it's derivative, $f'(x) = 2\sech(x)\tanh(x)$, we conclude that there are two inflection points at $x_1 = \max f'(x)$ and $x_2 = \min f'(x)$, leading to three different regions: $x \in ]-\infty, x_1]$ ($\mathcal{R}_1$, the first convex region), $x\in ]x_1, x_2]$ ($\mathcal{R}_2$, the concave region), and $x\in ]x_2, +\infty[$ ($\mathcal{R}_3$, the second convex region). As a result, there are 6 combinations for the location of $l_b$ and $u_b$ which must be resolved. This is exactly the same case as the first derivative of $\tanh$, simply with $x_1$ and $x_2$ instead of $y_1$ and $y_2$. Due to the similarities, we can use exactly the same relaxations as presented in Table~\ref{tab:sigma_prime_relaxations}. We present visual examples of 3 cases of this relaxation in Figure~\ref{fig:sech_relaxation}.

\section{Further details on Greedy Input Branching}
\label{app:greedy-branching-details}

In Section \ref{sec:input_branching} we motivated and described at a high-level greedy input branching. In the following we provide a step-by-step analysis of Algorithm \ref{alg:greedy_input_branching}. 

We start by initializing a lower and upper bound list of pairs $\mathcal{B}$ (line 3) as well as a list for storing the maximum error between the empirical and certified bounds $\mathcal{B}_\Delta$ (line 4). To initialize them (line 7 and 8), we first compute the empirical lower and upper bounds across the domain by sampling $N_s$ points within the full domain $\mathcal{C}$ using \textsc{Sample}$(\mathcal{C}, N_s)$ and evaluating the function $h$ there (line 5) yielding $\hat{h}_{lb}$ and $\hat{h}_{ub}$, as well as the first version of the certified lower and upper bounds using \partialcrown on $h$ (line 6) yielding $h_{lb}$, $h_{ub}$. Next, we pop from $\mathcal{B}$ and $\mathcal{B}_\Delta$ as $C_i$ the interval which has the maximum error between the empirical and certified bounds (line 10), which we then proceed to split into $N_d$ parts following a policy defined by \textsc{DomainSplit} (line 11). Importantly, \textsc{DomainSplit} must be complete, i.e., it must be that $C_i = \cup \,C’$. For each of those split subdomains $\mathcal{C}’$ we compute new bounds using \partialcrown (line 12) and add this subdomain along with its bounds and error to the empirical estimates to $\mathcal{B}$ and $\mathcal{B}_\Delta$, respectively (line 13 and 14). This process is repeated using the updated lists until the branching budget is spent, at which point the global lower bound is the minimum of all of lower bounds in $\mathcal{B}$ (defined as the list $\mathcal{B}_0$), and the global upper bound is the maximum of all upper bounds in $\mathcal{B}$ (defined as the list $\mathcal{B}_1$). These are computed in line 17. This algorithm is greedy as increasing the branching budget is expected to improve the bounds, since \partialcrown’s bounds are guaranteed to monotonically decrease with smaller input domains.

\section{On Extending \partialcrown to higher-order PDEs}
\label{app:higher-order}

In this section we explore the potential of applying \partialcrown to higher-order PDEs. We divide the analysis into the theory and experimental challenges, and how these could be mitigated.

\paragraph{Theory.} For the purposes of this paper, we only derive first and second partial derivative bounds, yet there is nothing that theoretically limits our method to second-order PDEs. The extension of the theory to third-order PDEs is relatively straightforward, consisting of applying the chain rule to Lemma \ref{lem:second_derivative}, and following the same backward-forward mechanism in the proof of Theorem \ref{thm:second_derivative_bounds} (Appendix \ref{app:proof_bounding_second_derivative}). We acknowledge that extending it to higher order PDEs leads to a growing computational graph, which can be more difficult to derive.

\paragraph{Experiments.} It is likely that the obtained bounds with extensions of \partialcrown to higher-order PDEs will be looser due to the growth of the computational graph. However, it is possible to mitigate these issues by designing (i) tighter nonlinearity relaxations and (ii) more efficient branching methods than our greedy branching one. 

We perform a qualitative analysis of the greedy branching strategy in Section \ref{sec:importance-branching}, and present in Table \ref{tab:branching_comparison} the \partialcrown $u_b$ difference between using a simple uniform strategy and our greedy input branching in Burgers’ equation given a fixed number number of branchings that leads to approximately the same runtime for both methods ($10^4$s). This highlights the importance of input branching in achieving tight bounds in second-order PINNs. With more efficient branching strategies, such as asymmetrical branching using sampling or through learning following a similar idea to \citet{balcan2018learning}, these could be significantly improved and applied to higher-order PINNs.

\begin{table}[t]
\centering
\caption{\textbf{Efficiency of Greedy Input Branching}: comparing greedy input branching to uniform branching in Burgers' equation given an approximate runtime limit of $10^4$s in both cases.}
\label{tab:branching_comparison}
\begin{tabular}{lcc}
   & \multicolumn{1}{c}{Uniform branching ($N_b=1.6\times 10^5$)} & \multicolumn{1}{c}{Greedy input branching ($N_b=1.3 \times 10^5$)} \\ \midrule
    $|f_\theta(x,t)|^2$ & $1.51 \times 10^2$  & $1.30 \times 10^1$ \\ \bottomrule
\end{tabular}
\end{table}

\end{document}